\RequirePackage{amsmath, amsthm}
\documentclass[smallcondensed]{svjour3}     
\smartqed

\usepackage{microtype}
\usepackage{subfigure}
\usepackage{booktabs} 

\usepackage{amsmath, amssymb, amsfonts, amsthm, bbm}
\usepackage{algorithm, algpseudocode, algorithmicx}
\usepackage{mathtools,commath}
\usepackage{graphicx, multirow}
\usepackage[round]{natbib}
\usepackage{xcolor}
\usepackage{authblk}
\usepackage{hyperref}

\begin{document}

\title{Principled analytic classifier for positive-unlabeled learning via weighted integral probability metric}
\titlerunning{Principled analytic classifier for PU learning via WIPM} 

\author{Yongchan Kwon \and
        Wonyoung Kim \and
        Masashi Sugiyama \and
        Myunghee Cho Paik\(^{\dagger}\)
}

\authorrunning{Yongchan Kwon et al.} 

\institute{
Yongchan Kwon, Wonyoung Kim, Myunghee Cho Paik \at
Department of Statistics, Seoul National University, Seoul, South Korea \\
\email{ykwon0407@snu.ac.kr, eraser347@snu.ac.kr, myungheechopaik@snu.ac.kr}
\and
Masashi Sugiyama \at
Center for Advanced Intelligence Project, RIKEN, Japan \& Graduate School of Frontier Sciences, The University of Tokyo, Japan\\
\email{sugi@k.u-tokyo.ac.jp}\\
\(\dagger\) indicates the corresponding author.
}

\date{Received: date / Accepted: date}

\maketitle

\begin{abstract}
We consider the problem of learning a binary classifier from only positive and unlabeled observations (called PU learning).
Recent studies in PU learning have shown superior performance theoretically and empirically.
However, most existing algorithms may not be suitable for large-scale datasets because they face repeated computations of a large Gram matrix or require massive hyperparameter optimization.
In this paper, we propose a computationally efficient and theoretically grounded PU learning algorithm.
The proposed PU learning algorithm produces a closed-form classifier when the hypothesis space is a closed ball in reproducing kernel Hilbert space.
In addition, we establish upper bounds of the estimation error and the excess risk.
The obtained estimation error bound is sharper than existing results and the derived excess risk bound has an explicit form, which vanishes as sample sizes increase.
Finally, we conduct extensive numerical experiments using both synthetic and real datasets, demonstrating improved accuracy, scalability, and robustness of the proposed algorithm.

\end{abstract}

\section{Introduction}
\label{s:intro}
Supervised binary classification assumes that all the training data are labeled as either being positive or negative.
However, in many practical scenarios, collecting a large number of labeled samples from the two categories is often costly, difficult, or not even possible.
In contrast, unlabeled data are relatively cheap and abundant.
As a consequence, semi-supervised learning is used for partially labeled data \citep{chapelle2006}.
In this paper, as a special case of semi-supervised learning, we consider Positive-Unlabeld (PU) learning, the problem of building a binary classifier from only positive and unlabeled samples \citep{denis2005, li2005}.
PU learning provides a powerful framework when negative labels are impossible or very expensive to obtain, and thus has frequently appeared in many real-world applications.
Examples include document classification \citep{elkan2008, xiao2011}, image classification \citep{zuluaga2011, gong2018}, gene identification \citep{yang2012, yang2014}, and novelty detection \citep{blanchard2010, zhang2017}. 

Several PU learning algorithms have been developed over the last two decades.
\citet{liu2002} and \citet{li2003} considered a two-step learning scheme: in Step 1, assigning negative labels to some unlabeled observations believed to be negative, and in Step 2, learning a binary classifier with existing positive samples and the negatively labeled samples from Step 1.
\citet{liu2003} pointed out that the two-step learning scheme is based on heuristics, and suggested fitting a biased support vector machine by regarding all the unlabeled observations as being negative.

\citet{scott2009} and \citet{blanchard2010} suggested a modification of supervised Neyman-Pearson classification, whose goal is to find a classifier minimizing the false positive rate keeping the false negative rate low.
To circumvent the problem of lack of negative samples, they tried to build a classifier minimizing the marginal probability of being classified as positive while keeping the false negative rate low.
Solving the empirical version of this constrained optimization problem is challenging, but the authors did not present an explicit algorithm.

Recently, many PU learning algorithms based on the empirical risk minimization principle have been studied.
\citet{du2014} proposed the use of the ramp loss and provided an algorithm that requires solving a non-convex optimization problem. 
\citet{du2015} formulated a convex optimization problem by using the logistic loss or double hinge loss. 
However, all the aforementioned approaches involve solving a non-linear programming problem.
This causes massive computational burdens for calculating the large Gram matrix when the sample size is large.
\citet{kiryo2017} suggested a stochastic algorithm for large-scale datasets with a non-negative risk estimator.
However, to execute the algorithm, several hyperparameters are required, and choosing the optimal hyperparameter may demand substantial trials of running the algorithm \citep{oh2018}, causing heavy computation costs.

In supervised binary classification, \citet{sriperumbudur2012} proposed a computationally efficient algorithm building a closed-form binary discriminant function.
The authors showed that their function estimator obtained by evaluating the negative of the empirical integral probability metric (IPM) is the minimizer of the empirical risk using the specific loss defined in Section \ref{s:IPM_binary}.
They further showed that a closed form can be derived as the result of restricting a hypothesis space to a closed unit ball in reproducing kernel Hilbert space (RKHS).

In this paper, capitalizing on the properties shown in the supervised learning method by \citet{sriperumbudur2012}, we extend it to PU learning settings. 
In addition, we derive new theoretical results on excess risk bounds.
We first define a weighted version of IPM between two probability measures and call it the weighted integral probability metric (WIPM).
We show that computing the negative of WIPM between the unlabeled data distribution and the positive data distribution is equivalent to minimizing the hinge risk. 
Based on this finding, we propose a binary discriminant function estimator that computes the negative of the empirical WIPM, and then derive associated upper bounds of the estimation error and the excess risk.
Under a mild condition, our obtained upper bounds are shown to be sharper than the existing ones because of using Talagrand\rq{}s inequality over McDiarmid\rq{}s inequality \citep{kiryo2017}.
Moreover, we pay special attention to the case where the hypothesis space is a closed ball in RKHS and propose a closed-form classifier.
We show that the associated excess risk bound has an explicit form that converges to zero as the sample sizes increase.
To the best of our knowledge, this is the first result to explicitly show the excess risk bound in PU learning.

As a summary, our main contributions are:
\begin{itemize}
\item We formally define WIPM and establish a link with the infimum of the hinge risk (Theorem \ref{thm:relationship}).
We derive an estimation error bound and show that it is sharper than existing results (Theorem \ref{thm:estimation_error} and Proposition \ref{prop:sharper}).
\item The proposed algorithm produces a closed-form classifier when the underlying hypothesis space is a closed ball in RKHS (Proposition \ref{prop:MMD_emp}).
Furthermore, we obtain a novel excess risk bound that converges to zero as sample sizes increase (Theorem \ref{thm:excess_super_fast}).
\item Numerical experiments using both synthetic and real datasets show that our method is comparable to or better than existing PU learning algorithms in terms of accuracy, scalability, and robustness in the case of unknown class-priors.
\end{itemize}

\section{Preliminaries}
In this section, we describe the $L$-risk for binary classification and present its PU representation. 
We briefly review several PU learning algorithms based on the $L$-risk minimization principle.
We first introduce problem settings and notations.

\subsection{Problem settings of PU learning}
\label{s:definition}
Let $X$ and $Y$ be random variables for input data and class labels, respectively, whose range is the product space $\mathcal{X} \times \{ \pm 1 \} \subseteq \mathbb{R}^{d} \times \{\pm 1 \}$.
The $d$ is a positive integer.
We denote the joint distribution of $(X,Y)$ by $P_{X,Y}$ and the marginal distribution of $X$ by $P_{X}$. 
The distributions of positive and negative data are defined by conditional distributions, ${P}_{X \mid Y=1}$ and ${P}_{X\mid Y=-1}$, respectively.
Let $\pi_{+} := P_{X,Y}(Y=1)$ be the marginal probability of being positive and set $\pi_{-}= 1-\pi_{+}$.
We follow the {\it two samples of data} scheme \citep{ward2009, niu2016}. 
That is, let $\mathcal{X}_{\rm p} =\{x_i ^{\rm p} \}_{i=1} ^{n_{\rm p}}$ and $\mathcal{X}_{\rm u} =\{x_i ^{\rm u} \}_{i=1} ^{n_{\rm u}}$ be observed sets of independently identically distributed samples from the positive data distribution ${P}_{X \mid Y=1}$ and the marginal distribution ${P}_X$, respectively.
Here, the $n_{\rm p}$ and $n_{\rm u}$ are the number of positive and unlabeled data points, respectively.
Note that the unlabeled data distribution is the marginal distribution.

Let $\mathcal{U}$ be a class of real-valued measurable functions defined on $\mathcal{X}$.
A function $f \in \mathcal{U}$, often called a hypothesis, can be understood as a binary discriminant function and we classify an input $x$ with the sign of a discriminant function, ${\rm sign}(f(x))$. 
Define $\mathcal{M} = \{f: \mathcal{X} \to \mathbb{R} \mid \norm{f}_{\infty} \leq 1 \} \subseteq \mathcal{U}$, where $\norm{f}_{\infty} = \sup_{x \in \mathcal{X}} | f(x) |$ is the supremum norm.
We restrict our attention to a class $\mathcal{F} \subseteq \mathcal{M}$ and call $\mathcal{F}$ a hypothesis space. 
Throughout this paper, we assume that the hypothesis space is symmetric, {\it i.e.}, $f \in \mathcal{F}$ implies $-f \in \mathcal{F}$.
In PU learning, the main goal is to construct a classifier ${\rm sign}(f(x))$ only from the positive dataset $\mathcal{X}_{\rm p}$ and the unlabeled dataset $\mathcal{X}_{\rm u}$ with $f \in \mathcal{F}$. 

In this paper, the quantity $\pi_{+}$, often called the class-prior, is assumed to be known as in the literature \citep{kiryo2017, kato2019} to focus on theoretical and practical benefits of our proposed algorithm.
We examine the performance when $\pi_{+}$ is unknown in Experiment 3 of Section \ref{s:simul_synthetic} and in Section \ref{s:real_data}.

\subsection{$L$-risk minimization in PU learning}
\label{s:risk}
In supervised binary classification, the $L$-risk is defined by
\allowdisplaybreaks
\begin{align}
R_L(f) :=& \int_{\mathcal{X} \times \{\pm 1\}}  L(y, f(x)) dP_{X,Y}(x,y)  \notag \\
=& \pi_{+} \int_{\mathcal{X}}  L(1, f(x)) dP_{X \mid Y=1}(x) + \pi_{-} \int_{\mathcal{X}}  L(-1, f(x)) dP_{X \mid Y=-1}(x),
\label{eq:risk}
\end{align} 
for a loss function $L : \{\pm 1\} \times \mathbb{R} \to \mathbb{R}$ \citep[Section 2.1]{steinwart2008}. 
We denote the margin-based loss function by $\ell(yt) := L(y, t)$ if a loss function $L(y,t)$ can be represented as a function of margin $yt$, the product of a label $y$ and a score $t$ for all possible $y \in \{ \pm 1 \}$ and $t \in \mathbb{R}$.


Under the PU learning framework, however, the right-hand side of Equation \eqref{eq:risk} cannot be directly estimated due to lack of negatively labeled observations.
To circumvent this problem, many studies in the field of PU learning exploited the relationship $P_X = \pi_{+} P_{X \mid Y =1}+\pi_{-} P_{X \mid Y =-1}$ and replaced $P_{X \mid Y =-1}$ in Equation \eqref{eq:risk} with $(P_X - \pi_{+} P_{X \mid Y =1})/\pi_{-}$ \citep{du2014, sakai2017}.
That is, the $L$-risk can be alternatively expressed as: 
\begin{align}
R_L(f) = \int_{\mathcal{X}}  L(-1, f(x)) dP_X(x) + \pi_{+} \int_{\mathcal{X}}  L(1, f(x))- L(-1, f(x)) dP_{X \mid Y=1}(x). \label{eq:pu_risk}
\end{align}
Now the right-hand side of Equation \eqref{eq:pu_risk} can be empirically estimated by the positive dataset $\mathcal{X}_{\rm p}$ and the unlabeled dataset $\mathcal{X}_{\rm u}$.
However, the $L$-risk $R_L (f)$ is not convex with respect to $f$ in general, and minimizing an empirical estimator for $R_L (f)$ is often formulated as a complicated non-convex optimization problem.

There have been several approaches to resolving the computational difficulty by modifying loss functions.
\citet{du2014} proposed to use non-convex loss functions satisfying the symmetric condition, $L(1, f(x)) + L(-1, f(x))=1$.
They proposed to optimize the empirical risk based on the ramp loss $\ell_{\mathrm{ramp}}(yt)=0.5 \times \max( 0,$  $\min(2, 1-yt))$ via the concave-convex procedure \citep{collobert2006}.
\citet{du2015} converted the problem to convex optimization through the linear-odd condition, $L(1, f(x)) - L(-1, f(x)) = -f(x)$.
They showed that the logistic loss $\ell_{\mathrm{log}}(yt) = \log(1+\exp(-yt))$ and the double hinge loss $\ell_{\mathrm{dh}}(yt) = \max(0, \max(-yt, (1-yt)/2))$ satisfy the linear-odd condition.
However, all the aforementioned methods utilized a weighted sum of $n_{\rm p} + n_{\rm u}$ predefined basis functions as a binary discriminant function, which triggered calculating the $(n_{\rm p}+n_{\rm u})$ $\times$ $(n_{\rm p}+n_{\rm u})$ Gram matrix.
Hence, executing algorithms is not scalable and can be intractable when $n_{\rm p}$ and $n_{\rm u}$ are large \citep{sansone2018}.
Our first goal in this paper is to overcome this computational problem by providing a computationally efficient method.


\section{Weighted integral probability metric and $L$-risk}
In this section, we formally define WIPM, a key tool for constructing the proposed algorithm, and build a link with the $L$-risk in Theorem \ref{thm:relationship} below. 
Based on the link, we propose a new binary discriminant function estimator and present its theoretical properties in Theorem \ref{thm:estimation_error}. 
We first introduce the earlier work by \citet{sriperumbudur2012} that provided a closed-form classifier in supervised binary classification.

\subsection{Relation between IPM and $L$-risk in supervised binary classification}
\label{s:IPM_binary}

\citet{muller1997} introduced an IPM for any two probability measures ${P}$ and ${Q}$ defined on $\mathcal{X}$ and a class $\mathcal{F}$ of bounded measurable functions, given by
\begin{align*}
{\rm IPM}({P}, {Q}; \mathcal{F}) := \sup_{f \in \mathcal{F}} \left| \int_{\mathcal{X}} f(x) d{P} (x) - \int_{\mathcal{X}} f(x) d{Q}(x) \right|.
\end{align*}
IPM has been studied as either a metric between two probability measures \citep{sriperumbudur2010b, arjovsky2017, tolstikhin2017} or a hypothesis testing tool \citep{gretton2012}. 

Under the supervised binary classification setting, \citet{sriperumbudur2012}  showed that calculating IPM between $P_{X \mid Y=1}$ and $P_{X \mid Y=-1}$ is negatively related to minimizing the risk with a loss function, {\it i.e.}, ${\rm IPM}(P_{X\mid Y=1}, P_{X\mid Y=-1}; \mathcal{F})=-\inf_{f \in \mathcal{F}} R_{L_{\mathrm{c}}}(f)$, where $L_{\mathrm{c}}(1, t)= -t/\pi_{+}$ and $L_{\mathrm{c}}(-1, t)= t/\pi_{-}$ for all $t \in \mathbb{R}$.
They further showed that a discriminant function minimizing the $L_{\mathrm{c}}$-risk can be obtained analytically when $\mathcal{F}$ is a closed unit ball in RKHS.
This result cannot be directly extended to PU learning due to absence of negatively labeled observations.
In the next subsection, we define a generalized version of IPM and extend the previous results for supervised binary classification to PU learning.

\subsection{Extension to WIPM and $L$-risk in PU learning}
\label{s:wipm}

Let $\mathcal{F}$ be a given class of bounded measurable functions and let $\tilde{w}: \mathcal{X} \to \mathbb{R}$ be a weight function such that $\norm{\tilde{w}}_{\infty} < \infty$.
We define WIPM\footnote{Although WIPM is not a metric in general, we keep saying the name WIPM to emphasize that it is a weighted version of IPM.} between two probability measures ${P}$ and ${Q}$ with a function class $\mathcal{F}$ and a weight function $\tilde{w}$ by
\begin{align}
{\rm WIPM}({P}, {Q}; \tilde{w}, \mathcal{F}) := \sup_{f \in \mathcal{F}} \left| \int_{\mathcal{X}} f(x) d{P} (x) - \int_{\mathcal{X}} \tilde{w}(x) f(x) d{Q}(x) \right|.
\label{eq:WIPM_fct}
\end{align}
Note that WIPM reduces to IPM if $\tilde{w}(x)=1$ for all $x \in \mathcal{X}$. 
Other special cases of Equation \eqref{eq:WIPM_fct} have been discussed in many applications.
In the covariate shift problem, \citet{huang2007} and \citet{gretton2009} proposed to minimize WIPM with respect to $\tilde{w}$ when $\mathcal{F}$ is the unit ball in RKHS and $P,Q$ are empirical distributions of test and training data, respectively.
In unsupervised domain adaptation, \citet{yan2017} regarded $P,Q$ as empirical distributions of target and source data, respectively, where in this case, $\tilde{w}$ is a ratio of two class-prior distributions. 
 
We pay special attention to the case where $\tilde{w}(x)$ is constant, $w \in \mathbb{R}$, for every input value and denote WIPM by  ${\rm WIPM}({P}, {Q}; w, \mathcal{F})$, 
\begin{align*}
{\rm WIPM}({P}, {Q}; w, \mathcal{F}) := \sup_{f \in \mathcal{F}} \left| \int_{\mathcal{X}} f(x) d{P} (x) - w \int_{\mathcal{X}} f(x) d{Q}(x) \right|.
\end{align*}
In the following theorem, we establish a link between ${\rm WIPM}(P_{X}, P_{X \mid Y = 1} ; 2\pi_{+}, \mathcal{F})$ and the infimum of the $\ell_{\mathrm{h}}$-risk over $\mathcal{F}$ for the hinge loss $\ell_{\mathrm{h}}(yt) = \max(0, 1-yt)$.
A proof is in Appendix \ref{app:relationship}.

\begin{theorem}[Relationship between $\ell_{\mathrm{h}}$-risk and WIPM]
Let $\mathcal{F}$ be a symmetric hypothesis space in $\mathcal{M}$ and $\ell_{\mathrm{h}}(yt) = \max(0, 1-yt)$ be the hinge loss. 
Then, we have
$$\inf_{f \in \mathcal{F}} R_{\ell_{\mathrm{h}}} (f) = 1 -  {\rm WIPM} (P_{X}, P_{X \mid Y = 1} ; 2\pi_{+}, \mathcal{F}).$$
Moreover, if $g_{\mathcal{F}}$ satisfies 
$$ {\rm WIPM} (P_{X}, P_{X \mid Y = 1} ; 2\pi_{+}, \mathcal{F}) = \int_{\mathcal{X}} g_{\mathcal{F}}(x) dP_X(x) - 2\pi_{+} \int_{\mathcal{X}} g_{\mathcal{F}}(x) dP_{X \mid Y=1}(x),$$then $\inf_{f \in \mathcal{F}} R_{\ell_{\mathrm{h}}}(f)$ $=R_{\ell_{\mathrm{h}}} (-g_{\mathcal{F}})$.
\label{thm:relationship}
\end{theorem}

Theorem \ref{thm:relationship} shows that the infimum of the $\ell_{\mathrm{h}}$-risk over a hypothesis space $\mathcal{F}$ equals the negative WIPM between the unlabeled data distribution $P_X$ and the positive data distribution $P_{X \mid Y=1}$ with the same hypothesis space $\mathcal{F}$ and the weight $2\pi_{+}$ up to addition by constant.
Furthermore, by negating the WIPM optimizer $g_{\mathcal{F}}$, we obtain the minimizer of the $\ell_{\mathrm{h}}$-risk over the hypothesis space $\mathcal{F}$.
Here, we define a WIPM optimizer $g_{\mathcal{F}}$ as a function that attains the supremum, {\it i.e.},
${\rm WIPM} (P_{X}, P_{X \mid Y = 1} ; 2\pi_{+}, \mathcal{F}) = \int_{\mathcal{X}} g_{\mathcal{F}}(x)$ $dP_X(x)$ $-2\pi_{+} \int_{\mathcal{X}} g_{\mathcal{F}}(x) dP_{X \mid Y=1}(x)$ and we set $f_{\mathcal{F}} = -g_{\mathcal{F}}$ for later notational convenience.
\citet{sriperumbudur2012} derived a similar result to Theorem \ref{thm:relationship} by showing ${\rm IPM}(P_{X\mid Y=1}, P_{X\mid Y=-1}; \mathcal{F})$ $=$ $-\inf_{f \in \mathcal{F}} R_{L_{\mathrm{c}}}(f)$ in supervised binary classification.
However, as we mentioned in Section \ref{s:IPM_binary}, their method is only applicable to supervised binary classification settings.

\subsection{Theoretical properties of empirical WIPM optimizer}
\label{s:theory}
We denote the empirical distributions of  ${P}_{X \mid Y=1}$ and ${P}_{X}$ by ${P}_{X \mid Y=1, n_{\rm p}}$ and  ${P}_{X, n_{\rm u}}$,  respectively. Let ${P}_{X \mid Y=1, n_{\rm p}} =  n_{\rm p} ^{-1} \sum_{i=1} ^{n_{\rm p}} \delta_{x_i ^{\rm p} }$ and ${P}_{X, n_{\rm u}} =  n_{\rm u} ^{-1} \sum_{i=1} ^{n_{\rm u}} \delta_{x_{i} ^{\rm u} }$, where $\delta(\cdot)$ defined on $\mathcal{X}$ is  the Dirac delta function and
 $\delta_x (\cdot) := \delta(\cdot -x)$ for $x \in \mathcal{X}$. 
The empirical Rademacher complexity of $\mathcal{F}$ given a set $S= \{z_1, \dots, z_{m}\} $ is defined by $\mathfrak{R}_{S}(\mathcal{F} ) := \mathbb{E}_{\sigma} \left( \frac{1}{m} \sup_{f \in \mathcal{F}} \left|  \sum_{i=1} ^{m} \sigma_i f(z_i) \right|   \right)$.
Here, $\{\sigma_i\}_{i=1} ^{m}$ is a set of independent Rademacher random variables taking $1$ or $-1$ with probability $0.5$ each and $\mathbb{E}_{\sigma}(\cdot)$ is the expectation operator over the Rademacher random variables \citep{bartlett2002}.
Denote a maximum by $a \vee b := \max(a,b)$, a minimum by $ a \wedge b := \min(a,b)$.  
For a probability measure ${Q}$ defined on $\mathcal{X}$, denote the expectation of a discriminant function $f$ by $\mathbb{E}_{Q}(f) := \int_{\mathcal{X}} f(x) d{Q}(x)$ and the variance by $\mathrm{Var}_{{Q}} (f) := \mathbb{E}_{Q}(f^2) - (\mathbb{E}_{Q}(f))^2$.

The empirical estimator for ${\rm WIPM}({P}_{X}, {P}_{X \mid Y=1}; w, \mathcal{F})$ is given by plugging in the empirical distributions,
\allowdisplaybreaks
\begin{align*}
&{\rm WIPM}({P}_{X, n_{\rm u}}, {P}_{X \mid Y=1, n_{\rm p}}; w, \mathcal{F}) = \sup_{f \in \mathcal{F}} \left| \frac{1}{n_{\rm u}} \sum_{i=1} ^{n_{\rm u}} f( x_i ^{\rm u}) - \frac{w}{n_{\rm p}} \sum_{i=1} ^{n_{\rm p}} f( x_i ^{\rm p}) \right|,
\end{align*}
and we define an empirical WIPM optimizer $\hat{g}_{\mathcal{F}} \in \mathcal{F}$ that satisfies the following equation,
\begin{align}
{\rm WIPM}({P}_{X, n_{\rm u}}, {P}_{X \mid Y=1, n_{\rm p}}; w, \mathcal{F}) = \frac{1}{n_{\rm u}} \sum_{i=1} ^{n_{\rm u}} \hat{g}_{\mathcal{F}} ( x_i ^{\rm u}) - \frac{w}{n_{\rm p}} \sum_{i=1} ^{n_{\rm p}} \hat{g}_{\mathcal{F}} ( x_i ^{\rm p}).
\label{eq:wipm_estimator}
\end{align}
We set $\hat{f}_{\mathcal{F}}= -\hat{g}_{\mathcal{F}}$ for notational convenience as in Section \ref{s:wipm}. 

We analyze the estimation error $R_{\ell_{\mathrm{h}}}(\hat{f}_{\mathcal{F}}) - \inf_{f \in \mathcal{F} } R_{\ell_{\mathrm{h}}} (f)$ in the following theorem.
A proof is provided in Appendix \ref{app:estimation_error}.
To begin, let $\chi_{n_{\rm p}, n_{\rm u}} ^{(1)} (w)=w/\sqrt{n_{\rm p}} + 1/\sqrt{n_{\rm u}}$ and $\chi_{n_{\rm p}, n_{\rm u}} ^{(2)} (w) = 2(w/n_{\rm p} + 1/n_{\rm u})$.

\begin{theorem}[Estimation error bound for general function space]
Let $\hat{g}_{\mathcal{F}}$ be an empirical WIPM optimizer defined in Equation \eqref{eq:wipm_estimator} and set $\hat{f}_{\mathcal{F}} = -\hat{g}_{\mathcal{F}}$.
Let $\mathcal{F}$ be a symmetric hypothesis space such that  $\norm{f}_{\infty} \leq \nu \leq 1$, $\mathrm{Var}_{{P}_{X \mid Y=1}} (f) \leq \sigma_{X \mid Y=1} ^2$, and $\mathrm{Var}_{{P}_{X}} (f) \leq \sigma_{X} ^2$.
Denote $\rho^2 = \sigma_{X \mid Y=1} ^2 \vee \sigma_{X} ^2$. Then, for all $\alpha, \tau >0$, the following holds with probability at least $1-e^{-\tau}$,
\begin{align}
R_{\ell_{\mathrm{h}}}(\hat{f}_{\mathcal{F}}) - \inf_{f \in \mathcal{F} } R_{\ell_{\mathrm{h}}} (f) \leq & C_{\alpha} (\mathbb{E}_{P_{X} ^{n_{\rm u}} } ( \mathfrak{R}_{\mathcal{X}_{\rm u}}(\mathcal{F}) )  + 2 \pi_{+} \mathbb{E}_{P_{X \mid Y =1 } ^{n_{\rm p}} }( \mathfrak{R}_{\mathcal{X}_{\rm p}}(\mathcal{F}) ) ) \label{eq:convergence_optim}
\\ &+  C_{\tau, \rho^2}^{(1)} \chi_{n_{\rm p}, n_{\rm u}} ^{(1)} (2\pi_{+})  +  C_{\tau, \nu, \alpha}^{(2)} \chi_{n_{\rm p}, n_{\rm u}} ^{(2)} (2\pi_{+}), \notag
\end{align}
where $C_{\alpha}=4(1+\alpha)$, $C_{\tau, \rho^2}^{(1)} =2\sqrt{ 2 \tau \rho^2 }$, $C_{\tau, \nu, \alpha}^{(2)} = 2\tau \nu \left( \frac{2}{3}   + \frac{1}{\alpha} \right)$.
\label{thm:estimation_error}
\end{theorem}

Due to Talagrand\rq{}s inequality, Theorem \ref{thm:estimation_error} provides a sharper bound than the existing result based on McDiarmid\rq{}s inequality. 
Specifically, \citet[Theorem 4]{kiryo2017} utilized McDiarmid\rq{}s inequality and showed that for $\tau >0$ and some $\Delta >0$ the following holds with probability at least $1-e^{-\tau}$,
\begin{align}
R_{\ell_{\mathrm{h}}}(\hat{f}) - \inf_{f \in \mathcal{F} } R_{\ell_{\mathrm{h}}} (f) \leq&  8 ( \mathbb{E}_{P_{X} ^{n_{\rm u}} } ( \mathfrak{R}_{\mathcal{X}_{\rm u}}(\mathcal{F}) ) + 2\pi_{+} \mathbb{E}_{P_{X \mid Y =1} ^{n_{\rm p}} } ( \mathfrak{R}_{\mathcal{X}_{\rm p}}(\mathcal{F}) ) ) \label{eq:kiryo_ineq} \\
&+ \chi_{n_{\rm p}, n_{\rm u}} ^{(1)} (2\pi_{+}) (1 + \nu) \sqrt{2\tau} + \Delta . \notag
\end{align}

The following proposition shows that the proposed upper bound \eqref{eq:convergence_optim} is sharper than the upper bound \eqref{eq:kiryo_ineq} under a certain condition.
A proof is provided in Appendix \ref{app:sharper}.

\begin{proposition}
With the notations defined in Theorem \ref{thm:estimation_error}, suppose that the following holds,
\begin{align}
& \frac{1+\nu}{2} - \frac{ 5 \sqrt{2 \tau } \chi_{n_{\rm p}, n_{\rm u}} ^{(2)}(2\pi_{+}) \nu}{ 6 \chi_{n_{\rm p}, n_{\rm u}} ^{(1)} (2\pi_{+})}  \geq   \rho.
\label{eq:condition}
\end{align}
Then, the proposed upper bound \eqref{eq:convergence_optim} is sharper than the previous result \eqref{eq:kiryo_ineq} proposed by \citet{kiryo2017}.
\label{prop:sharper}
\end{proposition}

It is noteworthy that the second term in the left-hand side of \eqref{eq:condition} converges to zero as $n_{\rm p}$ and $n_{\rm u}$ increase because $\chi_{n_{\rm p}, n_{\rm u}} ^{(1)} (2\pi_{+})$ $=$ $O_{ {P}_{X \mid Y=1}, {P}_{X}}  ( (n_{\rm p} \wedge n_{\rm u})^{-1/2} )$ and $\chi_{n_{\rm p}, n_{\rm u}} ^{(2)}(2\pi_{+})$ $=$ $O_{ {P}_{X \mid Y=1}, {P}_{X}} ( (n_{\rm p} \wedge n_{\rm u})^{-1} )$.
Due to $(1+\nu)/2 \geq \nu \geq \rho$, the condition \eqref{eq:condition} is quite reasonable if the upper bounds of the variances, $\sigma_{X} ^2$ and $\sigma_{X \mid Y=1} ^2$, are sufficiently small.

In binary classification, one ultimate goal is to find a classifier minimizing the misclassification error, or equivalently, minimizing the excess risk.
\citet{bartlett2006} showed that there is an invertible function $\psi: [-1,1] \to [0,\infty)$ such that the excess risk $R_{\ell_{01}}(\hat{f}_{\mathcal{F}}) - \inf_{f \in \mathcal{U} }  R_{\ell_{01}}(f)$ is bounded above by $\psi^{-1} (R_{\ell} (\hat{f}_{\mathcal{F}}) - \inf_{f \in \mathcal{U} }  R_{\ell}(f) )$ if the margin-based loss $\ell$ is classification-calibrated.
In particular, \citet{zhang2004} showed that the excess risk is bounded above by the excess $\ell_{\mathrm{h}}$-risk, {\it i.e.}, $R_{\ell_{01}}(\hat{f}_{\mathcal{F}}) - \inf_{f \in \mathcal{U} }  R_{\ell_{01}}(f)$ $\leq$ $R_{\ell_{\mathrm{h}}} (\hat{f}_{\mathcal{F}}) - \inf_{f \in \mathcal{U} } R_{\ell_{\mathrm{h}}}(f)$.
This implies that an excess risk bound can be obtained by analyzing the excess $\ell_{\mathrm{h}}$-risk bound with Theorem \ref{thm:estimation_error}.
The following corollary provides the excess risk bound.

\begin{corollary}[Excess risk bound for general function space]
With the notations defined in Theorem \ref{thm:estimation_error}, for all $\alpha, \tau >0$, the following holds with probability at least $1-e^{-\tau}$,
\begin{align*}
R_{\ell_{01}}(\hat{f}_{\mathcal{F}}) - \inf_{f \in \mathcal{U} }  R_{\ell_{01}}(f) \leq&  \inf_{f \in \mathcal{F} } R_{\ell_{\mathrm{h}}} (f) -\inf_{f \in \mathcal{U} } R_{\ell_{\mathrm{h}}}(f) \\
&+  C_{\alpha} (\mathbb{E}_{P_{X} ^{n_{\rm u}} } ( \mathfrak{R}_{\mathcal{X}_{\rm u}}(\mathcal{F}) )  + 2 \pi_{+} \mathbb{E}_{P_{X \mid Y =1 } ^{n_{\rm p}} }( \mathfrak{R}_{\mathcal{X}_{\rm p}}(\mathcal{F}) ) ) \\ &+  C_{\tau, \rho^2}^{(1)} \chi_{n_{\rm p}, n_{\rm u}} ^{(1)} (2\pi_{+})  +  C_{\tau, \nu, \alpha}^{(2)} \chi_{n_{\rm p}, n_{\rm u}} ^{(2)} (2\pi_{+}).
\end{align*}
\label{cor:excess_risk_bound}
\end{corollary}

\section{WIPM optimizer with reproducing kernel Hilbert space}
\label{s:RKHS}
In this section, we provide a computationally efficient PU learning algorithm which builds an analytic classifier when a hypothesis space is a closed ball in RKHS.
In addition, unlike the excess risk bound in Corollary \ref{cor:excess_risk_bound}, we explicitly derive the bound that converges to zero when the sample sizes $n_{\rm p}$ and $n_{\rm u}$ increase.

\subsection{An analytic classifier via WMMD optimizer}
To this end, we assume that $\mathcal{X} \subseteq [0,1]^{d}$ is compact. 
Let $k: \mathcal{X} \times \mathcal{X} \to \mathbb{R}$ be a reproducing kernel defined on $\mathcal{X}$ and $\mathcal{H}_k$ be the associated RKHS with the inner product $\langle \cdot, \cdot \rangle_{\mathcal{H}_k} : \mathcal{H}_k \times \mathcal{H}_k \to \mathbb{R}$. 
We denote the induced norm by $\norm{ \cdot }_{\mathcal{H}_k}$.
Denote a closed ball in RKHS $\mathcal{H}_k$ with a radius $r>0$, by $\mathcal{H}_{k,r} = \{f: \norm{f}_{\mathcal{H}_k} \leq r \}$.
We define the weighted maximum mean discrepancy (WMMD) between two probability measures ${P}$ and ${Q}$ with a weight $w$ and a closed ball $\mathcal{H}_{k,r}$ by ${\rm WMMD}_k ({P}, {Q}; w, r) $ $ := $ ${\rm WIPM} ({P}, {Q};w, \mathcal{H}_{k,r})$.
The name of WMMD comes from the maximum mean discrepancy (MMD), a popular example of the IPM whose function space is the unit ball ${\mathcal{H}}_{k,1}$, {\it i.e.}, ${\rm MMD}_k ({P}, {Q}) := {\rm IPM} ({P}, {Q};\mathcal{H}_{k,1})$ \citep{sriperumbudur2010b, sriperumbudur2010a}.
As defined in Equation \eqref{eq:wipm_estimator}, let $\hat{g}_{\mathcal{H}_{k,r}} \in \mathcal{H}_{k,r}$ be the empirical WMMD optimizer such that 
$${\rm WMMD}_k ({P}_{X, n_{\rm u}}, {P}_{X \mid Y=1, n_{\rm p}}; w, r) = \frac{1}{n_{\rm u}} \sum_{i=1} ^{n_{\rm u}} \hat{g}_{\mathcal{H}_{k,r}} ( x_i ^{\rm u}) - \frac{w}{n_{\rm p}} \sum_{i=1} ^{n_{\rm p}} \hat{g}_{\mathcal{H}_{k,r}} ( x_i ^{\rm p}). $$
In addition, we set $\hat{f}_{\mathcal{H}_{k,r}} = -\hat{g}_{\mathcal{H}_{k,r}}$, which leads the corresponding classification rule to ${\rm sign}(\hat{f}_{\mathcal{H}_{k,r}}(z))$. 
In the following proposition, we show that this classification rule has an analytic expression by exploiting the reproducing property $f(x) = \langle f, k(\cdot, x) \rangle_{\mathcal{H}_k}$ and the Cauchy-Schwarz inequality.
A proof is provided in Appendix \ref{app:MMD}.

\begin{proposition}
Let $k: \mathcal{X} \times \mathcal{X} \to \mathbb{R} $ be a bounded reproducing kernel.
Then, the classification rule has a closed-form expression given by
\begin{align}
{\rm sign} (\hat{f}_{\mathcal{H}_{k,r}}(z) ) = \begin{cases} 
      +1 & {\rm if} \quad (2\pi_{+})^{-1} < \hat{\lambda}_{n_{\rm p},n_{\rm u} }(z),   \\
      -1 & otherwise,
   \end{cases} \label{eq:WIPM_PU_classifier}    
\end{align}
where
$$\hat{\lambda}_{n_{\rm p},n_{\rm u} }(z) = \frac{ n_{\rm p} ^{-1}  \sum_{i=1 } ^{n_{\rm p}} k(z, x_i ^{\rm p})}{ n_{\rm u} ^{-1} \sum_{i=1 } ^{n_{\rm u}} k(z, x_i ^{\rm u})}.$$
\label{prop:MMD_emp}
\end{proposition}
We call the classifier defined in Equation \eqref{eq:WIPM_PU_classifier} the {\it WMMD classifier} and the score $\hat{\lambda}_{n_{\rm p},n_{\rm u} }(z)$ the {\it WMMD score} for $z$.
One strength of the WMMD classifier is that the classification rule has a closed-form expression, resulting in computational efficiency.
Furthermore, the WMMD score $\hat{\lambda}_{n_{\rm p},n_{\rm u} }$ is independent of the class-prior $\pi_{+}$, and thus we can obtain the score function without prior knowledge of the class-prior. 

\subsection{Explicit excess risk bound of WMMD classifier}
Since the empirical WMMD optimizer $\hat{g}_{\mathcal{H}_{k,r}}$ is a special case of the empirical WIPM optimizer, we have an excess risk bound from the result of Corollary \ref{cor:excess_risk_bound}.
However, without knowing convergence rates of the Rademacher complexities, $\mathbb{E}_{P_{X} ^{n_{\rm u}} } ( \mathfrak{R}_{\mathcal{X}_{\rm u}}(\mathcal{F}) )$ and $\mathbb{E}_{P_{X \mid Y=1} ^{n_{\rm p}} } ( \mathfrak{R}_{\mathcal{X}_{\rm p}}(\mathcal{F}) )$, and the approximation error, the consistency of the classifier remains unclear.
In this subsection, we establish an explicit excess risk bound that vanishes.
We first derive an explicit estimation error bound in the following lemma.

\begin{lemma}
[Explicit estimation error bound]
With the notations defined in Theorem \ref{thm:estimation_error}, assume that a reproducing kernel $k$ defined on a compact space $\mathcal{X}$ is bounded.
Let $r_1^{-1}=\sup_{x \in \mathcal{X}} \sqrt{k(x,x)}$.
Then, we have $\mathcal{H}_{k, r_1} \subseteq \mathcal{M}$.
Moreover, for all $\alpha, \tau >0$, the following holds with probability at least $1-e^{-\tau}$,
\begin{align*}
{R_{\ell_{\mathrm{h}}} (\hat{f}_{\mathcal{H}_{k,r_1}}) - \inf_{f \in \mathcal{H}_{k,r_1} } R_{\ell_{\mathrm{h}}} (f)} \leq ( C_{\alpha} + C_{\tau, \rho^2}^{(1)} )  \chi_{n_{\rm p}, n_{\rm u}} ^{(1)} (2\pi_{+})  +  C_{\tau, \nu, \alpha}^{(2)}  \chi_{n_{\rm p}, n_{\rm u}} ^{(2)} (2\pi_{+}).
\end{align*}
\label{lem:estimation_error_bound_RKHS}
\end{lemma}
While the bound in Theorem \ref{thm:estimation_error} is expressed in terms $\mathbb{E}_{P_{X} ^{n_{\rm u}} } ( \mathfrak{R}_{\mathcal{X}_{\rm u}}(\mathcal{F}) )$ and  $\mathbb{E}_{P_{X \mid Y=1} ^{n_{\rm p}} } ( \mathfrak{R}_{\mathcal{X}_{\rm p}}( \mathcal{F}) )$, these are evaluated in terms of $n_p$ and $n_u$ in the upper bound in Lemma \ref{lem:estimation_error_bound_RKHS}, giving an explicit estimation error bound with $O((n_{\rm p} \wedge n_{\rm u})^{-1/2})$ convergence rate.
The key idea is to use reproducing property $f(x) = \langle f, k(\cdot, x) \rangle_{\mathcal{H}_k}$ and the Cauchy-Schwarz inequality to obtain an upper bound for the Rademacher complexity. 
Detailed proofs are given in Appendix \ref{app:rade_rkhs}.

In the following lemma, we elaborate on the approximation error bound.
To begin, for any $0< \beta \leq 1$, let $\beta \mathcal{M} := \{ \beta f : f \in \mathcal{M} \}$.
Set $f_1^* (x) = {\rm sign} ( P(Y=1 \mid X=x) - \frac{1}{2} )$.

\begin{lemma}[Approximation error bound over uniformly bounded hypothesis space]
With the notations defined in Lemma \ref{lem:estimation_error_bound_RKHS},  we have
$$ \inf_{f \in \mathcal{H}_{k,r_1} } {R_{\ell_{\mathrm{h}}} ( f ) - \inf_{f \in \beta \mathcal{M} } R_{\ell_{\mathrm{h}}} (f)}  \leq \beta \inf_{g \in \mathcal{H}_{k,{r_1 /\beta}}} \norm{ g- f_1^*}_{L_2(P_X)},$$ for any $0< \beta \leq 1$.
\label{lem:approximation_error_bound_RKHS}
\end{lemma}
When $\beta=1$, Lemma \ref{lem:approximation_error_bound_RKHS} implies that the approximation error $ \inf_{f \in \mathcal{H}_{k,r_1} } R_{\ell_{\mathrm{h}}} ( f ) $ $ - \inf_{f \in \mathcal{U} } R_{\ell_{\mathrm{h}}} (f) $ is bounded above by $ \inf_{g \in \mathcal{H}_{k,{r_1}}} \norm{ g- f_1^*}_{L_2(P_X)}$ due to $\inf_{f \in \mathcal{U} } R_{\ell_{\mathrm{h}}}(f) $ $= \inf_{f \in \mathcal{M} } R_{\ell_{\mathrm{h}}}(f)$ \citep{lin2002}.
Hence, a naive substitution to Corollary \ref{cor:excess_risk_bound} will give a sub-optimal bound because $\inf_{g \in \mathcal{H}_{k,{r_1}}} \norm{ g- f_1^*}_{L_2(P_X)}$ is non-zero in general.

In the following theorem, we rigorously establish the explicit excess risk bound which vanishes as $n_{\rm p}$ and $ n_{\rm u}$ increase. 

\begin{theorem}
Assume that the Gaussian kernel $k(x, y) = \exp $ $(-\frac{\norm{x-y}^2}{2 h^2})$ is used.
Under the assumptions (A1)-(A4) in Appendix \ref{app:excess_super_fast}, for $h = (n_{\rm p} \wedge n_{\rm u} )^{-\frac{1}{2\alpha_{\text{H}}+d}}$, we have the following holds with probability at least $1- 1/n_{\rm p}  - 1/ n_{\rm u}$:
\begin{align*}
R_{\ell_{01}}(\hat{f}_{\mathcal{H}_{k, 1}}) - \inf_{f \in \mathcal{U} }  R_{\ell_{01}}(f) \leq C_{(\rm p, \rm u)} (n_{\rm p} \wedge n_{\rm u})^{-\frac{\alpha_{\text{H}} (1+q)}{2\alpha_{\text{H}}+d}}, 
\end{align*} 
where $\alpha_{\text{H}}$ and $q$ are defined in Appendix \ref{app:excess_super_fast} and $C_{(\rm p, \rm u)} >0$ is some constant.
\label{thm:excess_super_fast}
\end{theorem}
In supervised binary classification settings, a similar result is obtained by \citet[Theorem 3.3]{audibert2007} and the convergence rate is called \textit{a super-fast rate} when $\alpha_{\text{H}} q > d$.
However, $\alpha_{\text{H}}$ and $q$  cannot be simultaneously very large \citep{audibert2007}. 

\citet{niu2016} provided the excess risk bound expressed as a function of $n_{\rm p}, n_{\rm u}$.
However, their bound included combined terms of the approximation error and the Rademacher complexity, as in Corollary \ref{cor:excess_risk_bound}.
To the best of our knowledge, we are the first to explicitly derive the excess risk bound with convergence rate in terms of a function of $n_{\rm p}, n_{\rm u}$ in PU learning.

\section{Related work}
\label{s:related_work}

\textbf{Excess risk bound in noisy label literature:}
PU learning can be considered as a special case of classification with asymmetric label noise, and many studies in this literature have shown consistency results similar to Theorem \ref{thm:excess_super_fast} \citep{natarajan2013}.
\citet{patrini2016} derived an explicit estimation error when $\mathcal{F}$ is a set of linear hypotheses and \citet{blanchard2016} showed a consistency result of the excess risk bound when the hypothesis space is RKHS with universal kernels.
While the two studies assumed the {\it one sample of data} scheme, the proposed bound is based on the {\it two samples of data} scheme.
Therefore, our proposed excess risk bound is expressed in $n_{\rm p}$ and $n_{\rm u}$, giving a new consistency theory.

\textbf{Closed-form classifier:}
\citet{blanchard2010} suggested a score function similar to the WMMD score by using different bandwidth hyperparameters for the denominator and the numerator.
However, with these differences, our method gains theoretical justification while their score function does not. 
\citet{du2015} derived a closed-form classifier based on the squared loss. 
They estimated $P(Y=1 \mid X)- P(Y=-1 \mid X)$ and showed the consistency of the estimation error bound in the {\it two samples of data} scheme.
However, the classifier is not scalable because it requires to compute the inverse of a $(n_{\rm p}+n_{\rm u})$ $\times$ $(n_{\rm p}+n_{\rm u})$ matrix.

\section{Numerical experiments}
\label{s:simul}
In this section, we empirically analyze the proposed algorithm to demonstrate its practical efficacy using synthetic and real datasets.
Optimization procedures and the selection of hyperparameters are detailed in Appendix \ref{app:imple_details}.
Pytorch implementation for the experiments is available at \url{https://github.com/eraser347/WMMD_PU}.

\subsection{Synthetic data analysis}
\label{s:simul_synthetic}
We first visualize the effect of increasing the sample sizes $n_{\rm p}$ and $n_{\rm u}$ on the discriminant ability of the proposed algorithm (Experiment 1). 
Then we compare performance with (i) the logistic loss $\ell_{\mathrm{log}}$, denoted by LOG, (ii) the double hinge loss $\ell_{\mathrm{dh}}$, denoted by DH, both proposed by \citet{du2015}, (iii) the non-negative risk estimator method, denoted by NNPU, proposed by \citet{kiryo2017}, (iv) the threshold adjustment method, denoted by tADJ, proposed by \citet{elkan2008}, and (v) the proposed algorithm, denoted by WMMD (Experiments 2, 3, and 4). 

\begin{figure}[t]
\centering
\includegraphics[width=0.4\textwidth]{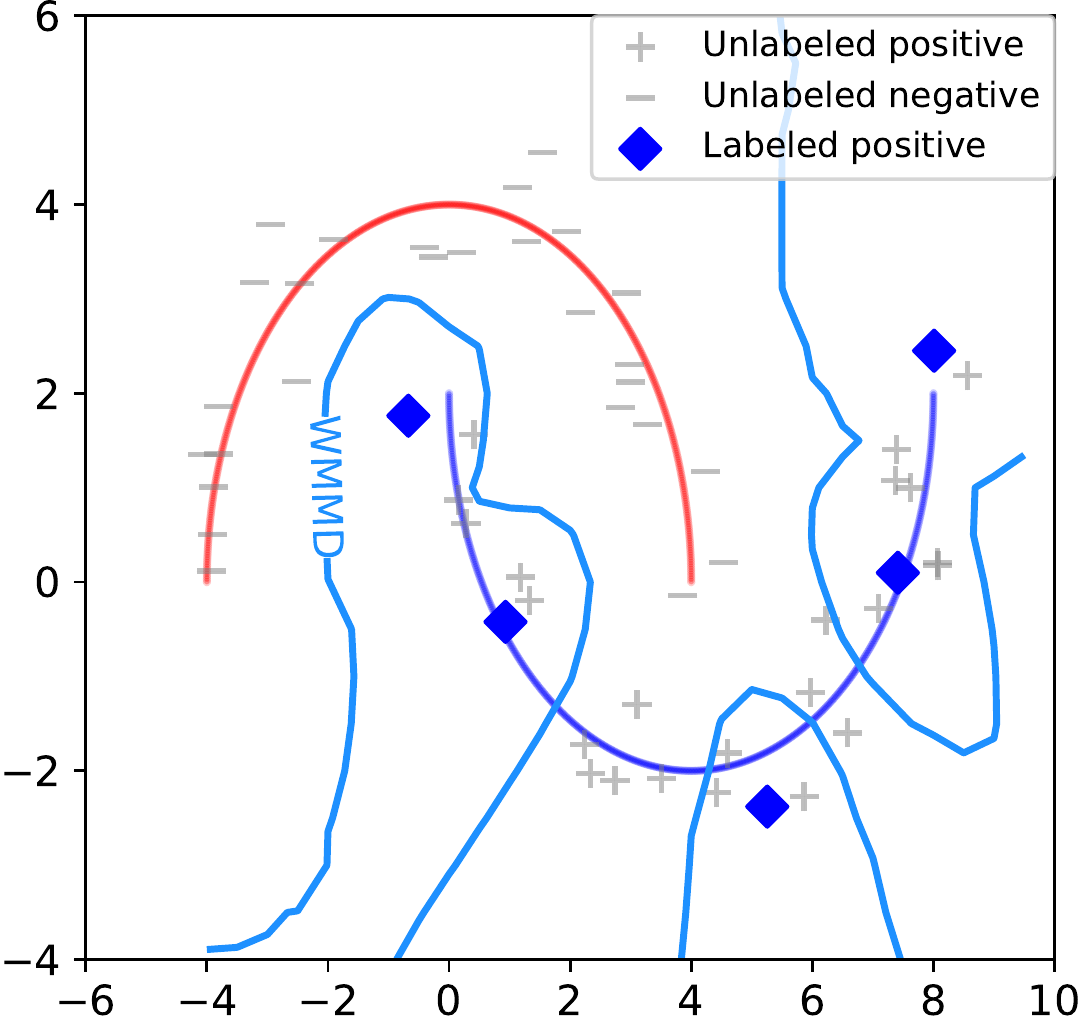}
\includegraphics[width=0.4\textwidth]{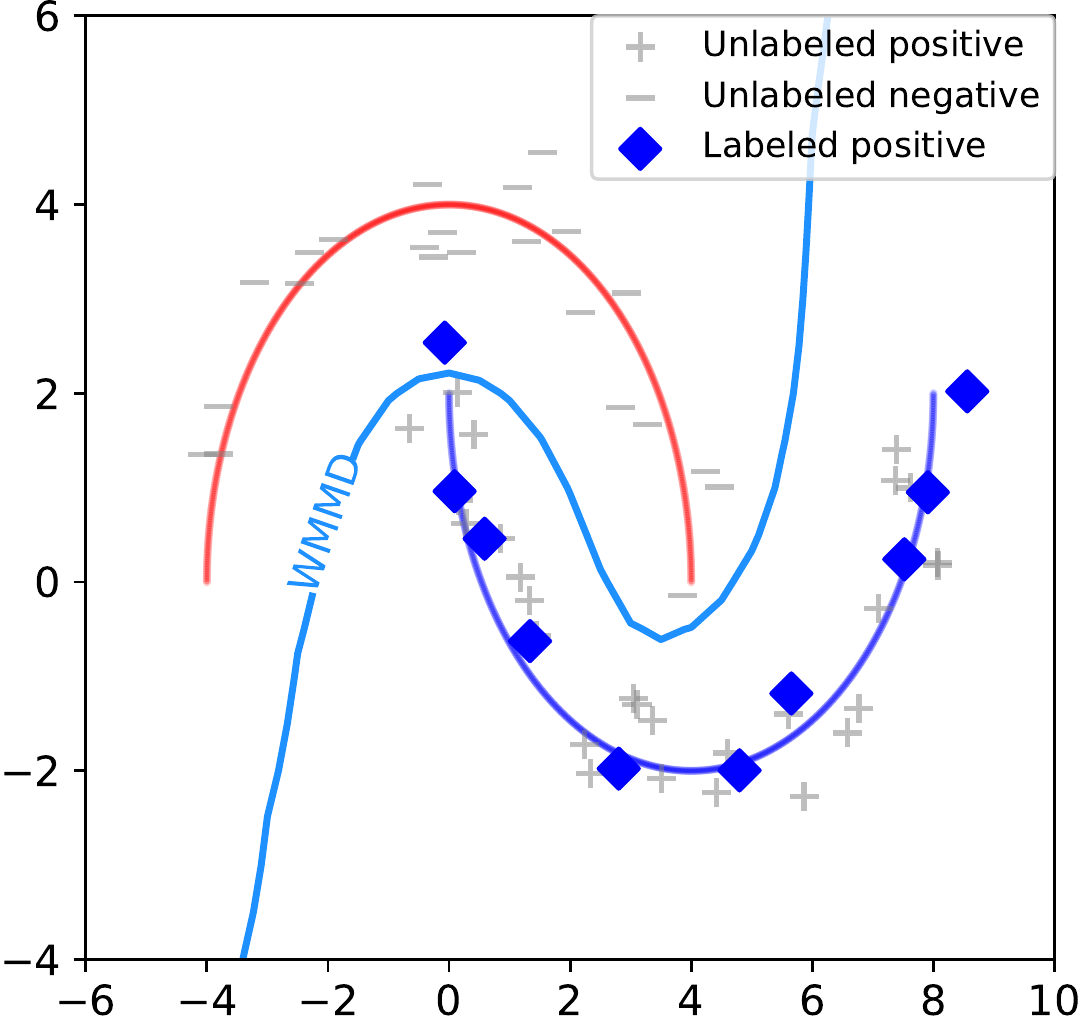}
\includegraphics[width=0.4\textwidth]{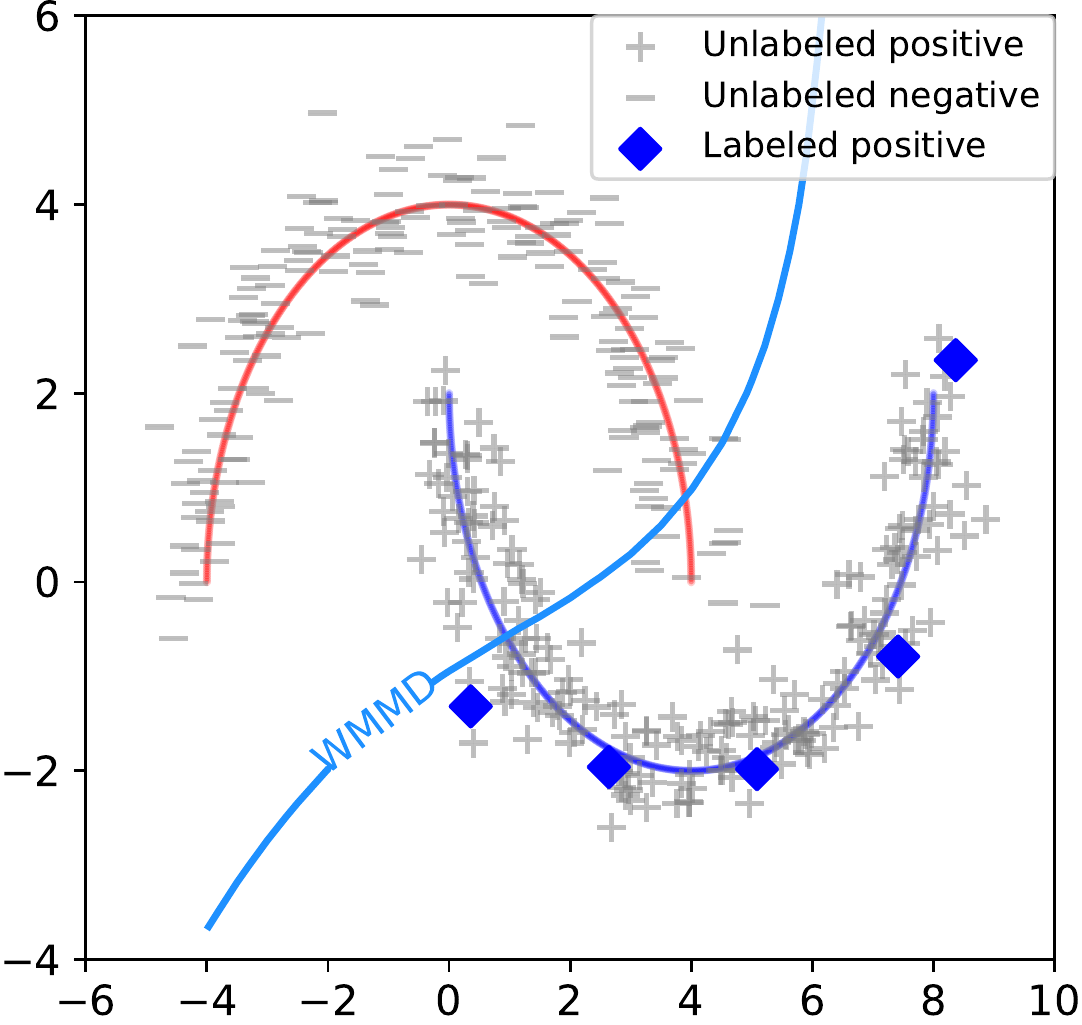}
\includegraphics[width=0.4\textwidth]{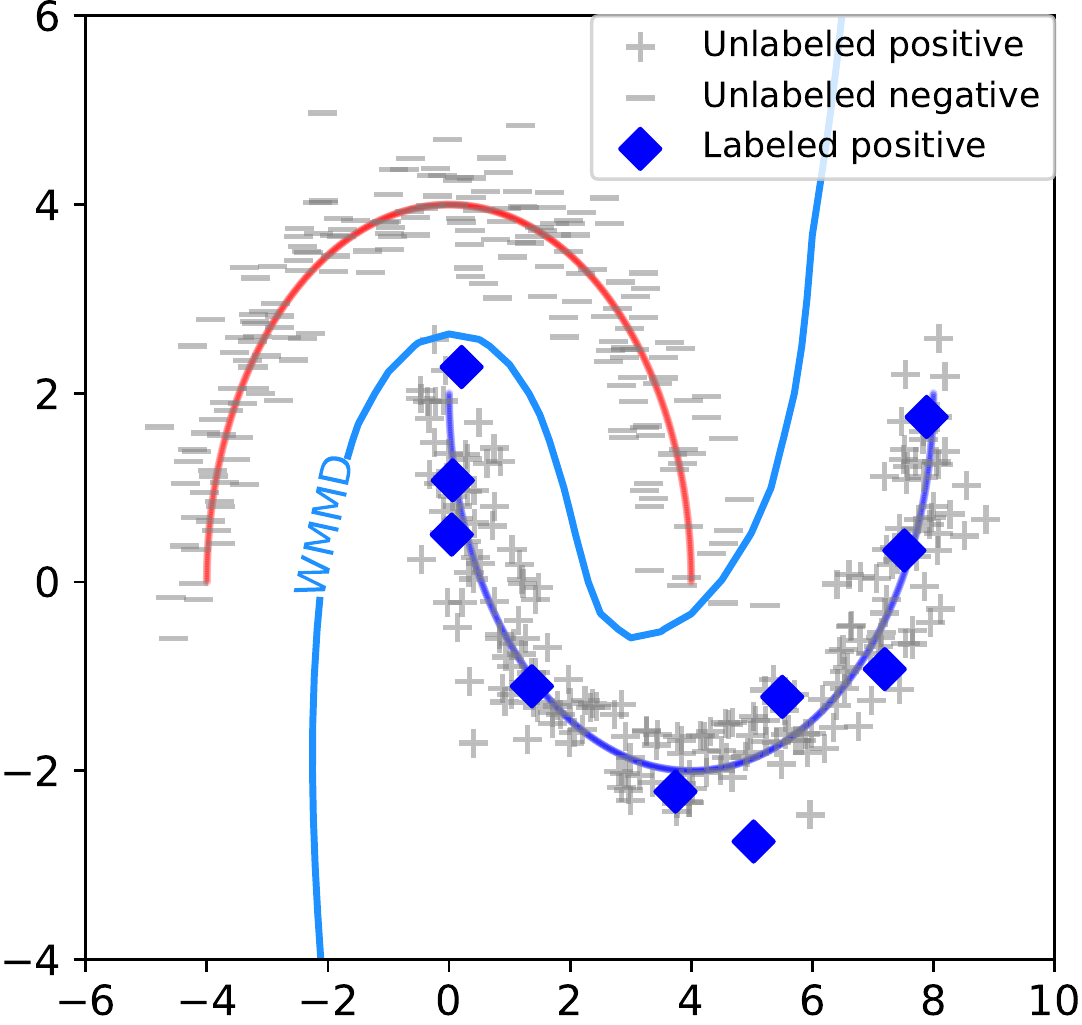}
\caption{The illustration of the decision boundaries of the WMMD classifier using the \texttt{two\_moons} dataset with the increases in the size of positive and unlabeled samples. The true means of the positive and negative data distributions are plotted by blue and red lines respectively. The gray `+' points and the gray `-' points refer to the unlabeled positive and unlabeled negative training data, respectively.}
\label{fig:synthetic_data_decision_boundary}

\end{figure}

\textbf{Experiment 1:} In this case, we used the \texttt{two\_moons} dataset whose underlying distributions are 
\begin{align*}
X|Y=y, U &\sim N\left(\begin{bmatrix}2(1+y) - 4y\cos (\pi U) \\
(1+y)-4y\sin(\pi U)
\end{bmatrix},\begin{bmatrix}0.4^{2} & 0\\
0 & 0.4^{2}
\end{bmatrix}\right),
\end{align*}
where $U$ refers to a uniform random variable ranges from 0 to 1 and $N(\mu,\Sigma)$ is the normal distribution with mean $\mu$ and covariance $\Sigma$.
We used the \textsf{\lq{}make\_moons\rq{}} function in the Python module \textsf{\lq{}sklearn.datasets\rq{}} \citep{scikit-learn} to generate the datasets.

Figure \ref{fig:synthetic_data_decision_boundary} illustrates the decision boundaries of WMMD using the \texttt{two\_moons} dataset. 
The first row displays the case where the unlabeled sample size is small, $n_{\rm u}=50$, and the second row displays the case where the unlabeled sample size is large, $n_{\rm u}=400$.
The first and second columns display the case where the positive sample sizes are $n_{\rm p}=5$ and $n_{\rm p}=10$, respectively.
The class-prior is fixed to $\pi_{+}=0.5$, and we assumed that the class-prior is known.
We visualize the true mean function of the positive and negative data distributions with blue and red lines, respectively.
The positive data are represented by blue diamond points, and the unlabeled data are represented by gray points.
The decision boundaries of the WMMD classifier tend to correctly separate the two clusters as $n_{\rm p}$ and $n_{\rm u}$ increase.

In Experiments 2, 3, and 4, we evaluate: 
(i) the accuracy and area under the receiver operating characteristic curve (AUC) as $n_{\rm u}$ and $\pi_{+}$ change when the class-prior is known (Experiment 2) and unknown (Experiment 3);
(ii) the elapsed training time (Experiment 4).
In these experiments, we set up the underlying joint distribution as follows:
\begin{align}
X\mid Y=y \sim N\left( y\frac{\boldsymbol{1}_{2}}{\sqrt{2}}, I_{2}\right), Y \sim 2 \times \mathrm{Bern}(\pi_{+}) -1,
\label{eq:comparison_data_distribution}
\end{align}
where $\mathrm{Bern}(p)$ is the Bernoulli distribution with mean $p$, $\boldsymbol{1}_{2} = (1,1)^T$ is the 2 dimensional vector of all ones and $I_{2}$ is the identity matrix of size 2. 

\textbf{Experiment 2:}
In this experiment, we compare the accuracy and AUC of the five PU learning algorithms when the true class-prior $\pi_{+}$ is known.
Figures \ref{fig:synthetic_data_accuracy_n_u} and \ref{fig:synthetic_data_AUC_n_u} show the accuracy and AUC on various $n_{\rm u}$.
The training sample size for the positive data is $n_{\rm p}=100$ and the class prior is $\pi_{+}=0.5$.
The unlabeled sample size changes from 40 to 500 by 20.
We repeat a random generation of training and test data 100 times.
For comparison purposes, we add the $1-$Bayes risk for each unlabeled sample size.
In terms of accuracy, the proposed WMMD tends to be closer to the $1-$Bayes risk as the $n_{\rm u}$ increases.
Compared with other PU learning algorithms, WMMD achieves higher accuracy in every $n_{\rm u}$ and achieves comparable to or better AUC.

Figures \ref{fig:synthetic_data_accuracy_pi_plus} and \ref{fig:synthetic_data_AUC_pi_plus} show a comparison of accuracy and AUC as $\pi_{+}$ changes.
The training sample size for the positive and unlabeled data are $n_{\rm p}=100$ and $n_{\rm u}=400$, respectively. 
The class-prior $\pi_{+}$ changes from $0.05$ to $0.95$ by $0.05$.
The test sample size is $10^3$.
Training and test data are repeatedly generated 100 times with different random seeds.
In terms of accuracy, the proposed WMMD performs comparably with LOG and NNPU, showing advantages over DH and tADJ.
When the true class-prior is less than equal to $0.8$, WMMD performs better in terms of AUC, except for tADJ.
The tADJ achieves the highest AUC because $P(Y=1 \mid X=x)$ is proportional to $P( \{ x \text{ is from the positive dataset} \} \mid X=x)$.
This empirically shows that WMMD has a comparable discriminant ability to the other algorithms for a wide range of class-priors.

\begin{figure}[t]
\subfigure[Accuracy comparison on various $n_{\text{u}}$.]{
\includegraphics[width=0.48\textwidth, height=1.75in]{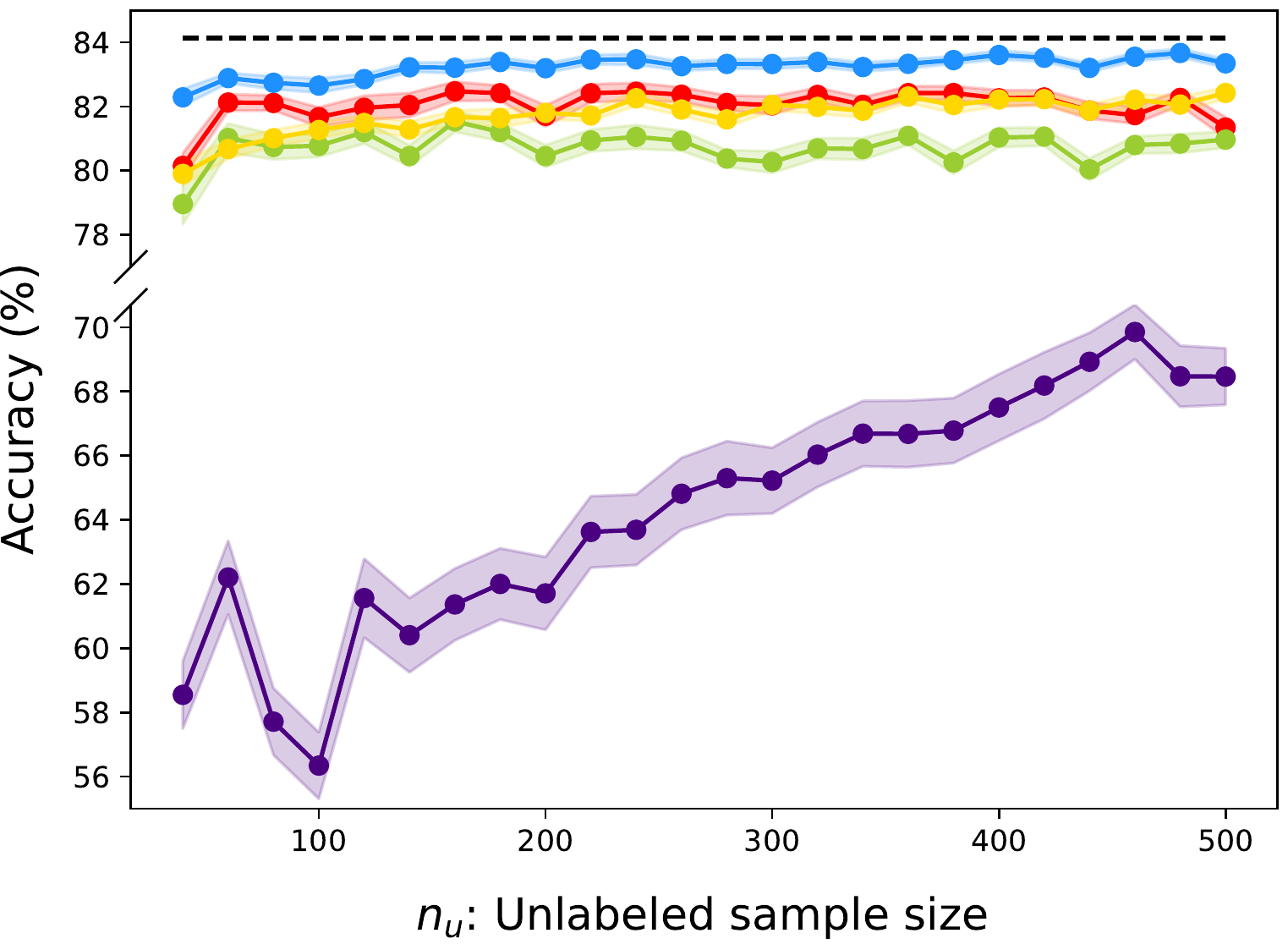}
\label{fig:synthetic_data_accuracy_n_u}
}
~~
\subfigure[Accuracy comparison on various $\pi_{+}$.]{
\includegraphics[width=0.48\textwidth, height=1.75in]{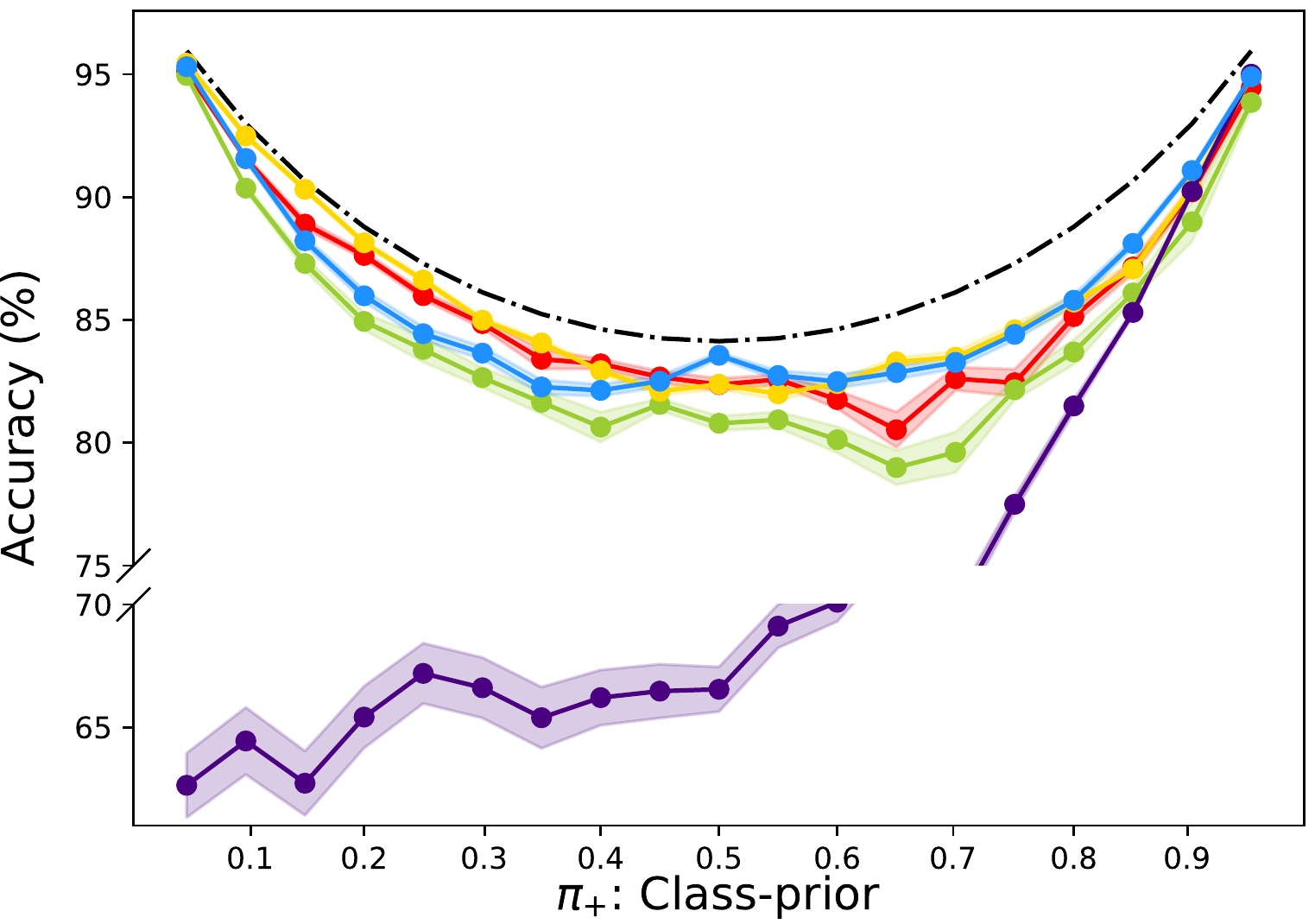}
\label{fig:synthetic_data_accuracy_pi_plus}
}
~~
\subfigure[AUC comparison on various $n_{\text{u}}$.]{
\includegraphics[width=0.48\textwidth, height=1.75in]{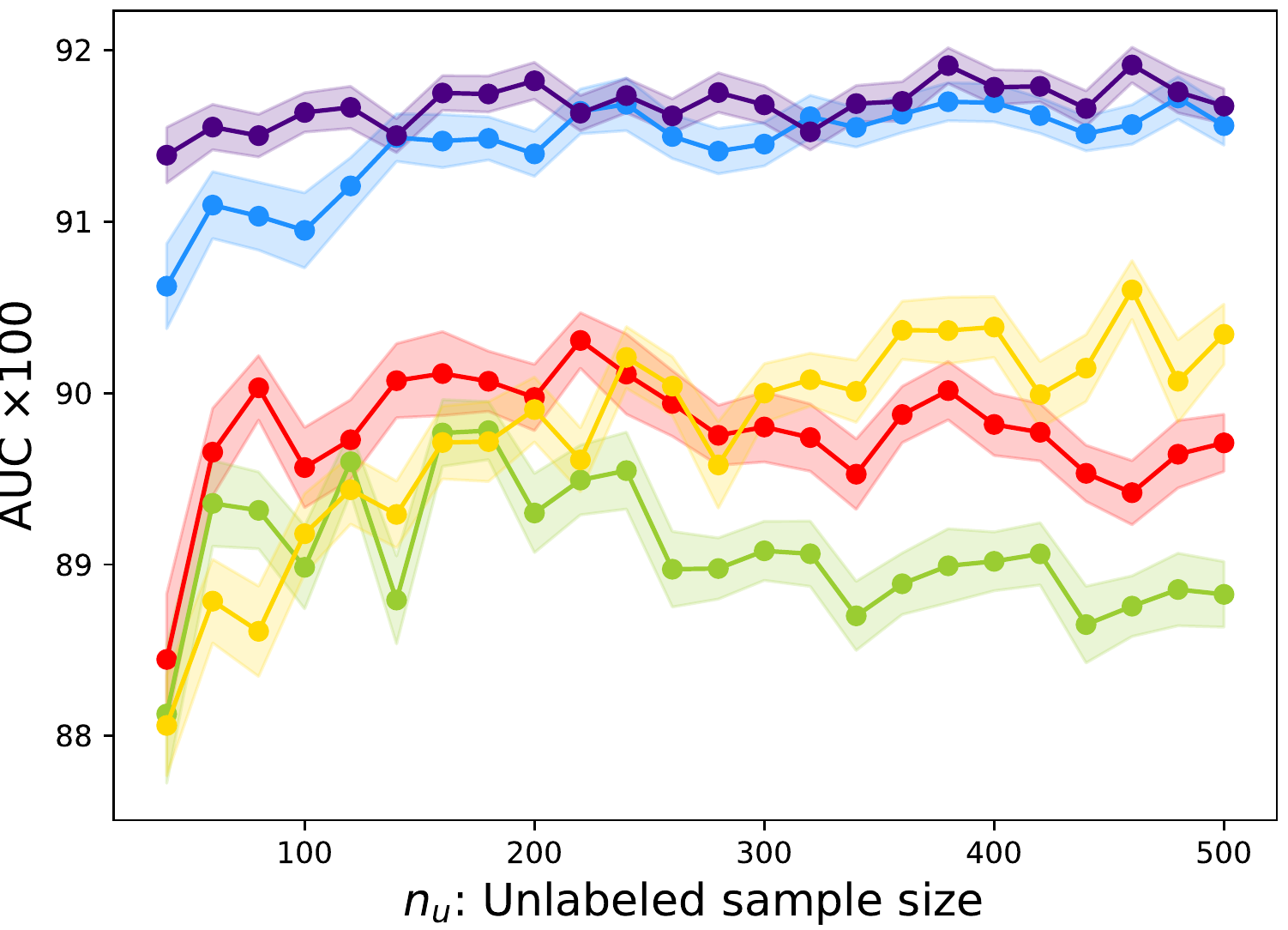}
\label{fig:synthetic_data_AUC_n_u}
}
~~
\subfigure[AUC comparison on various $\pi_{+}$.]{
\includegraphics[width=0.48\textwidth, height=1.75in]{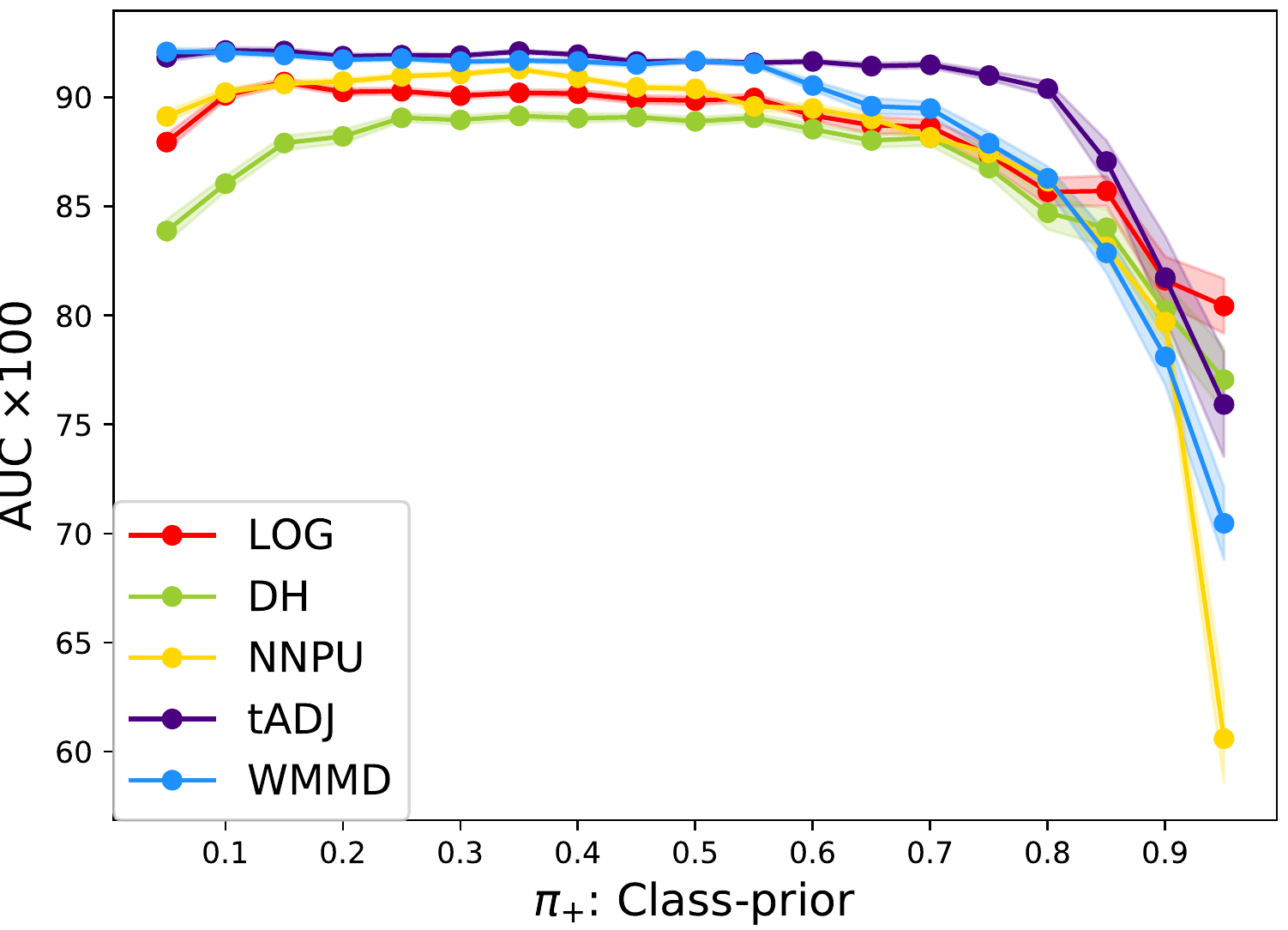}
\label{fig:synthetic_data_AUC_pi_plus}
}
\caption{The comparison of the accuracy and AUC of the five PU learning algorithms when each of $n_{\text{u}}$ and $\pi_{+}$ changes.
The dashed curve represents the $1-$Bayes risk.
The curve and the shaded region represent the average and the standard error, respectively, based on 100 replications.}
\end{figure}

\begin{figure}[t]
\subfigure[Accuracy comparison on various $n_{\text{u}}$ when the class-prior is unknown.]{
\includegraphics[width=0.48\textwidth, height=1.75in]{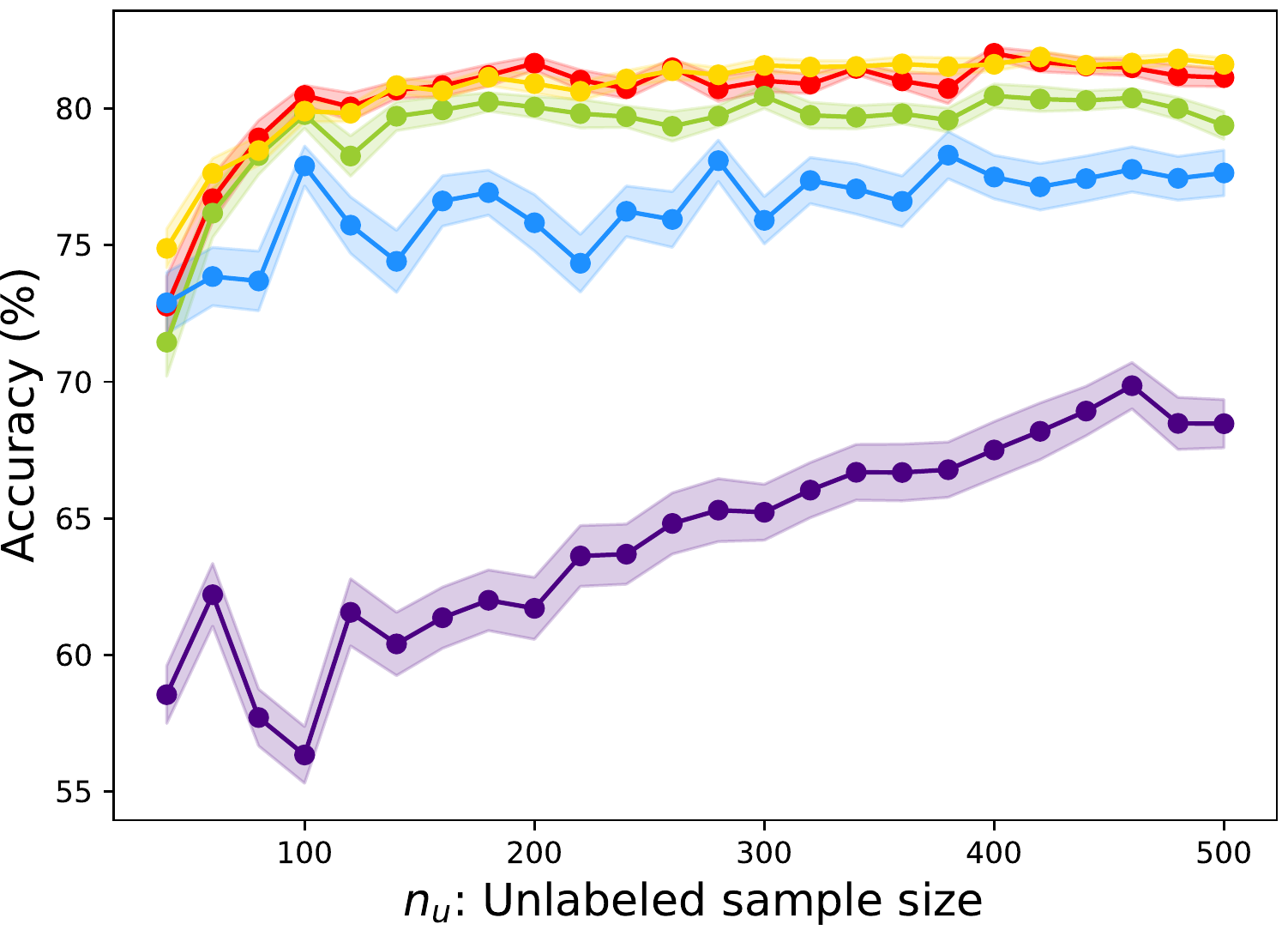}
\label{fig:synthetic_data_accuracy_n_u_unknown}
}
~~
\subfigure[Accuracy comparison on various $\pi_{+}$ when the class-prior is unknown.]{
\includegraphics[width=0.48\textwidth, height=1.75in]{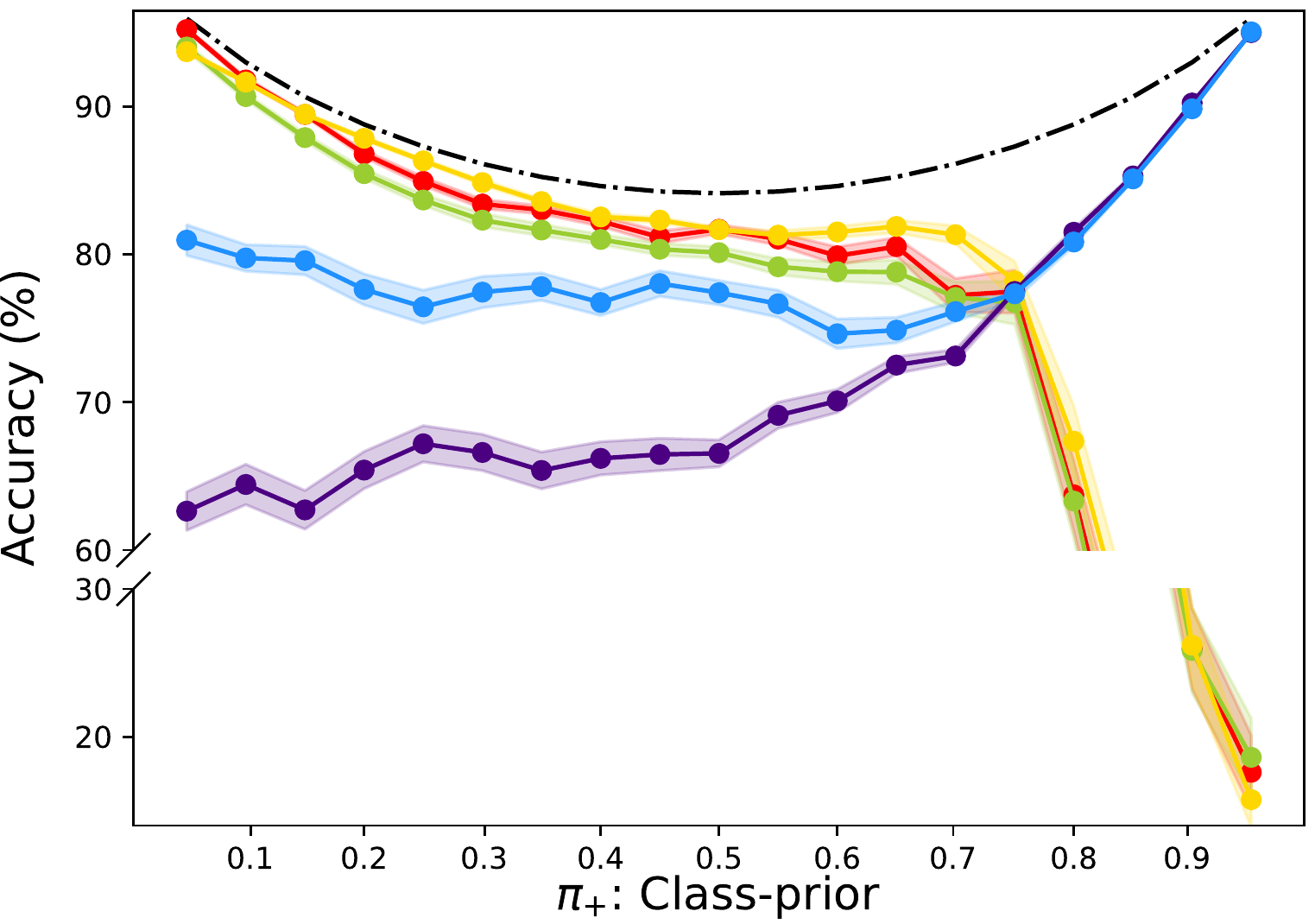}
\label{fig:synthetic_data_accuracy_pi_plus_unknown}
}
~~
\subfigure[AUC comparison on various $n_{\text{u}}$ when the class-prior is unknown.]{
\includegraphics[width=0.48\textwidth, height=1.75in]{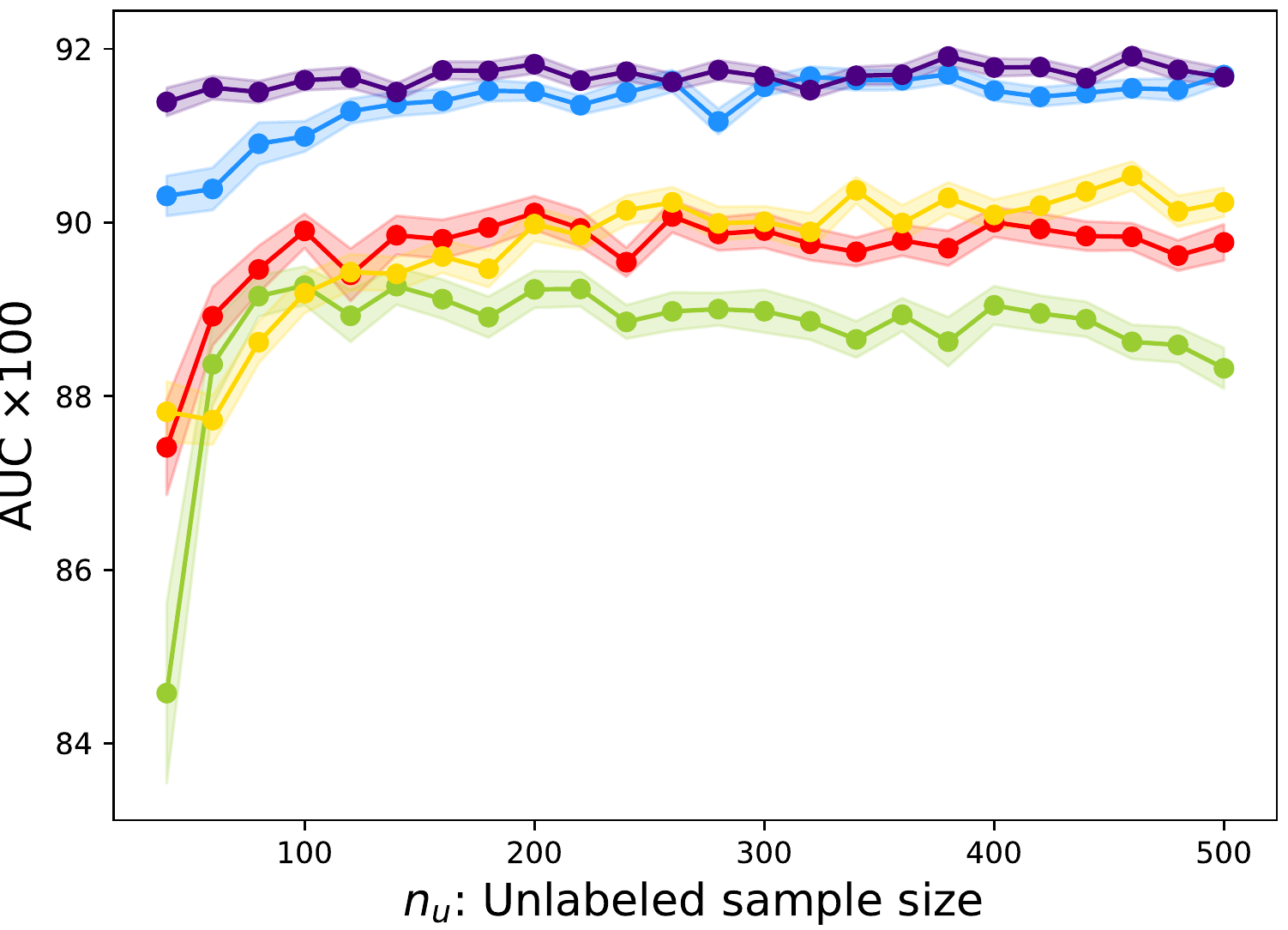}
\label{fig:synthetic_data_AUC_n_u_unknown}
}
~~
\subfigure[AUC comparison on various $\pi_{+}$ when the class-prior is unknown.]{
\includegraphics[width=0.48\textwidth, height=1.75in]{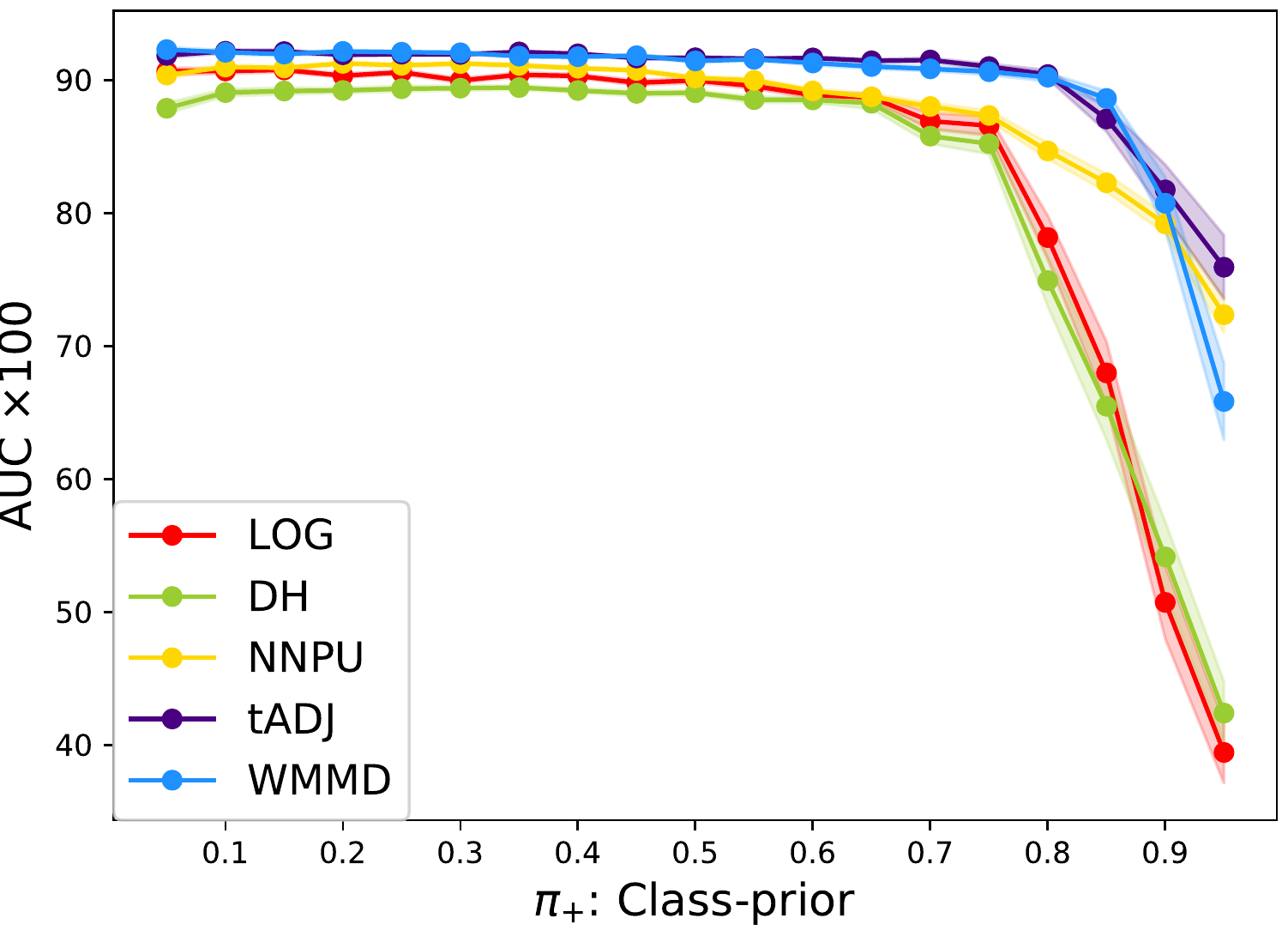}
\label{fig:synthetic_data_AUC_pi_plus_unknown}
}
\caption{The comparison of the accuracy and AUC of the five PU learning algorithms when each of $n_{\text{u}}$ and $\pi_{+}$ changes under the situation where ${\pi}_{+}$ is unknown.
The dashed curve represents the $1-$Bayes risk.
The curve and the shaded region represent the average and the standard error, respectively, based on 100 replications.
LOG, DH, and NNPU use the estimate of the class-prior from the \lq{}KM1\rq{} method.}
\end{figure}

\textbf{Experiment 3:}
The main goal of this subsection is to show the robustness of the proposed classifier in the case of unknown class-prior ${\pi}_{+}$.
In PU learning literature, $\pi_{+}$ has been frequently assumed to be known \citep{du2015, niu2016, kiryo2017, kato2019}.
However, this assumption can be considered to be strong in real-world applications, and to correctly execute existing PU learning algorithms, an accurate estimate of $\pi_{+}$ is necessary.
In this experiment, we compare the accuracy and AUC when the class-prior ${\pi}_{+}$ is unknown.
For the WMMD classifier, we used a density-based method for the class-prior estimation described in Appendix \ref{app:wmmd_imple_details}, which can be obtained as a byproduct of the proposed algorithm.
The results of LOG, DH, and NNPU are given for completeness sake using the \lq{}KM1\rq{} method\footnote{While the \lq{}KM2\rq{} method by \citet{ramaswamy2016} is often considered to be a state-of-the-art method for estimating $\pi_{+}$, in our experiments, estimates based on the \lq{}KM2\rq{} method have a larger estimation error than that of the \lq{}KM1\rq{} method and thus we omitted it.} by \citet{ramaswamy2016}.
We take these estimates as true values and repeat the same comparative numerical experiments in Experiment 2. 

Since the objective functions of the LOG, DH, and NNPU algorithms depend on the estimate $\hat{\pi}_{+}$, we anticipate that both the accuracy and AUC rely on the quality of the estimation.
On the other hand, the tADJ algorithm does not depend on the class-prior, so the performance is not affected.
Also, as the proposed score function does not depend on the class-prior $\pi_{+}$, and since $\pi_{+}$ is used only to determine a cutoff, the AUC of the proposed algorithm is less affected by the estimation of $\pi_{+}$.

Figures \ref{fig:synthetic_data_accuracy_n_u_unknown} and \ref{fig:synthetic_data_AUC_n_u_unknown} compare the accuracy and AUC as a function of $n_{\rm u}$.
WMMD performs worse than LOG, DH, and NNPU, while AUC is higher.
Though tADJ shows poor accuracy in a wide range, it achieves high AUC comparable to WMMD.
As we anticipated, WMMD is more robust than LOG, DH, and NNPU in AUC.
This is possibly because our score function $\hat{\lambda}_{n_{\rm p},n_{\rm u} }$ does not depend on $\pi_{+}$.
A similar trend can be found in Figures \ref{fig:synthetic_data_accuracy_pi_plus_unknown} and \ref{fig:synthetic_data_AUC_pi_plus_unknown}.
We note that the \lq{}KM1\rq{} method is not scalable and thus may not be used for large-scale datasets.

\textbf{Experiment 4:}
In this experiment, we compare the elapsed training time, including hyperparameter optimization, of the five PU learning algorithms.
The data are generated from the distributions described in Equation \eqref{eq:comparison_data_distribution}, and we set $n_{\rm p}=100, n_{\rm u}=400$, and $\pi_{+} = 0.5$.
The elapsed time is measured with 20 Intel\raisebox{0.5ex}{\small{\textregistered}} Xeon\raisebox{0.5ex}{\small{\textregistered}} E5-2630 v4@2.20GHz CPU processors.

Table \ref{t:synthetic_elapsed_time} compares the elapsed training time and its ratio relative to that of WMMD.
WMMD takes the shortest time among the five baseline methods.
In particular, the training time for WMMD is at least about 300 times shorter than that of the LOG and DH methods.
This is because the WMMD classifier has an analytic form while the LOG and DH methods require solving a non-linear programming problem. 

\subsection{Real data analysis}
\label{s:real_data}

We demonstrate the practical utility of the proposed algorithm using the eight real binary classification datasets from the LIBSVM\footnote{\url{https://www.csie.ntu.edu.tw/~cjlin/libsvmtools/datasets/}} \citep{chang2011}.
Since some observations from the raw datasets are not completely recorded, we removed such observations and construct the dataset with fully recorded data.
Next, to investigate the effect of varying $\pi_{+}$, we artificially reconstructed $\mathcal{X}_{\rm p}$ and $\mathcal{X}_{\rm u}$ through a random sampling from the fully recorded datasets.
For the three datasets \texttt{australian\_scale}, \texttt{breast-cancer\_scale} and \texttt{skin\_nonskin},
we reconstructed the data so that the resulting class-prior $\pi_{+}$ ranges from $0.15$ to $0.79$. 
We add the suffix \texttt{2} for those datasets.
We randomly resampled data $100$ times for the seven small datasets and $10$ times for the four big datasets: \texttt{skin\_nonskin}, \texttt{skin\_nonskin2}, \texttt{epsilon\_normalized}, and \texttt{HIGGS}.
Table \ref{t:summary_real_data} summarizes statistics for the eleven real datasets.
We conduct two comparative numerical experiments when $\pi_{+}$ is known and unknown.

\begin{table}[t]
\centering
\caption{A summary of elapsed training time and its ratio for the five PU learning algorithms based on 100 replications. We set $n_{\rm p}=100, n_{\rm u}=400$, and $\pi_{+}=0.5$. Average and standard error are denoted by \lq{}average$\pm$standard error\rq{}.}
\label{t:synthetic_elapsed_time}
\begin{tabular}{lccccc}
\toprule
  & LOG & DH  & NNPU & tADJ & WMMD \\
\midrule
in seconds $\times 10$ & $90.0\pm4.7$ & $96.1\pm6.1$ &  $6.0\pm0.1$ & $0.4\pm0.0$ & $0.2\pm0.0$ \\
in ratio & $ 347.9 \pm 23.4$ & $371.0 \pm 28.5$  &$23.2\pm1.1$ & $1.8 \pm 0.0$ &$1.0\pm0.0$\\
\bottomrule
\end{tabular}
\end{table}

\begin{table}[!t]
\centering
\caption{A summary of the eleven binary classification datasets. 
`\# of samples' denotes the number of total samples after removing incomplete observations.
We denote the number of positive, unlabeled, and test samples, by $n_{\rm p}$, $n_{\rm u}$, and $n_{\rm te}$ after the random sampling, respectively.
We categorize the eleven datasets into two groups: the first seven datasets as small-scale and the last four datasets as large-scale.}
\label{t:summary_real_data}
\resizebox{\textwidth}{!}{
\begin{tabular}{lcccccccccccc}
\toprule
Dataset & $d$ & \# of samples & $n_{\rm p}$ & $n_{\rm u}$ & $n_{\rm te}$ &  $\pi_{+}$ & Scale  \\ 
\midrule
{heart\_scale} & 12 & 122 & 10 & 60 & 60  & 0.62 & Small \\
{sonar\_scale} & 60 & 207 & 10 & 100 & 100 & 0.47 & Small \\
{australian\_scale} & 12 & 449 & 20 & 220 & 220 & 0.51 & Small \\
{australian\_scale2} & 12 & 449 & 10 & 130 & 130 & 0.15 & Small \\
{breast-cancer\_scale} & 10 & 683 & 20 & 340 & 340 & 0.35 & Small \\
{breast-cancer\_scale2} & 10 & 683 & 40 & 340 & 340 & 0.65 & Small \\
{diabetes\_scale}  & 8 & 759 & 50 & 380 & 370 & 0.65 & Small  \\
{skin\_nonskin}  & 3 & 245,057 & $10^3$ & $10^5$ & $10^5$ & 0.79 & Large \\
{skin\_nonskin2}  & 3 & 245,057 & $10^3$ & $10^5$ & $10^5$ & 0.21 & Large \\
{epsilon\_normalized} & 2,000 & 500,000 & $10^3$ & $4 \times 10^5$ & $10^5$ &  0.50 & Large \\
{HIGGS}  & 26 & 8,786,441 & $10^3$ & $10^6$ & $10^5$ & 0.50 & Large \\
\bottomrule
\end{tabular}
}

\end{table}

Table \ref{t:result_real_data_known} shows the average and the standard error of the accuracy and AUC when the class-prior $\pi_{+}$ is known.  
LOG and DH fail to compute the $(n_{\rm p}+n_{\rm u})$ $\times$ $(n_{\rm p}+n_{\rm u})$ Gram matrix due to out of memory in the 12 GB GPU memory limit.
WMMD achieves comparable to or better accuracy and AUC than LOG, DH, and tADJ on most datasets.
Compared to NNPU, WMMD performs comparably on the small datasets.
However, NNPU achieves higher accuracy on \texttt{skin\_nonskin}, \texttt{epsilon\_normalized}, and \texttt{HIGGS}.
The neural network used in NNPU fits well to the complicated and high-dimensional structure of data and shows high accuracy.

\begin{table*}[!t]
\centering
\caption{Accuracy and AUC comparison using the real datasets when the class-prior $\pi_{+}$ is known.
We denote the memory error results for LOG and DH by the hyphen.
Average and standard error are denoted by \lq{}average$\pm$standard error\rq{}.
Boldface numbers denote the best and equivalent algorithms with respect to a t-test with a significance level of 5\%.}
\label{t:result_real_data_known}
{\fontsize{6}{6}
\begin{tabular}{rccccccccc}
\toprule
Dataset & LOG & DH & NNPU & tADJ & WMMD \\ 
\midrule
\multicolumn{5}{l}{Accuracy (in \%)}\\
{heart\_scale}  & $\boldsymbol{70.5\pm0.8}$ & $68.4\pm0.9$ & $\boldsymbol{71.0\pm0.8}$ & $65.1\pm0.8$  & $\boldsymbol{71.6\pm0.8}$ \\ 
{sonar\_scale} & $55.8\pm0.6$ & $52.9\pm0.6$ & $\boldsymbol{63.2\pm0.6}$ & $60.7\pm0.6$  & $\boldsymbol{62.4\pm0.6}$ \\ 
{australian\_scale} & $\boldsymbol{85.4\pm0.4}$ & $\boldsymbol{84.9\pm0.6}$ & $79.2\pm0.5$ & $80.0\pm0.7$ & $\boldsymbol{84.2\pm0.6}$  \\ 
australian\_scale2 & $85.7\pm0.2$ & $85.7\pm0.2$ & $\boldsymbol{86.7\pm0.3}$ & $75.4\pm2.0$ & $\boldsymbol{86.2\pm0.2}$ \\ 
{breast-cancer\_scale} & $\boldsymbol{95.8\pm0.1}$ & $\boldsymbol{96.0\pm0.3}$ & $90.1\pm0.3$ & $91.0\pm0.4$  & $89.3\pm0.5$  \\
{breast-cancer\_scale2} & $\boldsymbol{95.6\pm0.3}$ & $94.4\pm0.8$ & $\boldsymbol{95.9\pm0.1}$ & $92.5\pm0.2$ & $94.2\pm0.3$  \\
{diabetes\_scale} & $66.7\pm0.7$ & $65.5\pm0.9$ & $\boldsymbol{69.4\pm0.4}$ & $67.9\pm0.3$  & $66.4\pm0.2$  \\ 
{skin\_nonskin}  & - & -  & $\boldsymbol{98.2\pm0.1}$ & $78.0\pm0.4$ & $85.3\pm0.7$ \\ 
{skin\_nonskin2}  & - & - & $\boldsymbol{98.6\pm0.0}$ & $93.9\pm0.1$ & $\boldsymbol{98.1\pm0.2}$   \\ 
{epsilon\_normalized} & - & - &  $\boldsymbol{64.5\pm0.3}$ & $63.1\pm0.1$ & $56.3\pm1.3$ \\ 
{HIGGS} & -  & -  &  $\boldsymbol{56.3\pm0.2}$ & $52.6\pm0.1$ & $54.0\pm0.2$ 
 \\ 
\midrule
\multicolumn{5}{l}{AUC $\times 100$}\\
{heart\_scale}  & $\boldsymbol{78.4\pm1.0}$ & $\boldsymbol{78.3\pm1.1}$ & $73.8\pm0.9$ & $72.5\pm1.1$  & $\boldsymbol{79.0\pm0.9}$ \\ 
{sonar\_scale} & $61.2\pm0.8$ & $60.6\pm0.9$ & $\boldsymbol{67.4\pm0.7}$ & $66.2\pm0.7$  & $\boldsymbol{68.9\pm0.8}$ \\ 
{australian\_scale} & $\boldsymbol{91.1\pm0.2}$ & $\boldsymbol{91.3\pm0.3}$ & $87.8\pm0.4$ & $87.8\pm0.5$  & $\boldsymbol{90.4\pm0.4}$ \\ 
australian\_scale2 & $\boldsymbol{89.2\pm0.4}$ & $87.3\pm0.4$ & $84.3\pm0.6$ & $85.9\pm0.7$ & $\boldsymbol{88.6\pm0.6}$ \\
{breast-cancer\_scale}  & $99.4\pm0.0$ & $99.3\pm0.0$ & $97.8\pm0.1$& $95.6\pm0.4$  & $\boldsymbol{99.5\pm0.0}$  \\
{breast-cancer\_scale2} & $\boldsymbol{99.3\pm0.0}$ & $\boldsymbol{99.2\pm0.1}$ & $\boldsymbol{99.3\pm0.0}$ & $97.2\pm0.2$ & $98.7\pm0.2$ \\
{diabetes\_scale}  & $\boldsymbol{74.0\pm0.6}$ & $71.5\pm1.1$ & $\boldsymbol{73.5\pm0.6}$ & $\boldsymbol{74.7\pm0.5}$  & $\boldsymbol{74.5\pm0.7}$  \\ 
{skin\_nonskin}  & - & - & $\boldsymbol{99.5\pm0.1}$ & $94.8\pm0.1$ & $\boldsymbol{99.4\pm0.1}$  \\ 
{skin\_nonskin2}  & - & - & $99.7\pm0.0$ & $94.6\pm0.0$ & $\boldsymbol{99.8\pm0.0}$ 
   \\ 
{epsilon\_normalized} & - & - & $\boldsymbol{70.0\pm0.4}$ & ${69.3\pm0.1}$ & $62.2\pm2.3$ \\ 
{HIGGS} & - & - & $59.6\pm0.2$ & $\boldsymbol{65.3\pm0.1}$ & $55.7\pm0.3$ \\ 
\bottomrule
\end{tabular}
\selectfont}
\end{table*}

Table \ref{t:result_real_data_unknown} compares the average and the standard error of the accuracy and AUC when the class-prior $\pi_{+}$ is unknown. 
As in Experiment 3 in Section \ref{s:simul_synthetic}, we estimate $\pi_{+}$ using the \lq{}KM1\rq{} method for LOG, DH, and NNPU, and using the density-based method for WMMD.
The LOG, DH, and NNPU algorithms are implemented on the seven small-scale datasets alone because the method by \citet{ramaswamy2016} is not feasible with the large-scale datasets \citep{bekker2018}. 
Overall, WMMD shows comparable to or better performances than other PU learning algorithms on most datasets.
Compared to Table \ref{t:result_real_data_known}, WMMD and tADJ show robustness to unknown $\pi_{+}$ in terms of AUC.
This is because WMMD and tADJ do not require estimation of $\pi_{+}$ to construct score functions.
In contrast, the other methods require an estimate $\hat{\pi}_{+}$, and we observe a substantial drop in accuracy and AUC when the \lq{}KM1\rq{} method estimate is used.

\begin{table*}[!t]
\centering
\caption{Accuracy and AUC comparison using the real datasets when the class-prior $\pi_{+}$ is unknown.
The \lq{}KM1\rq{} method by \protect\citet{ramaswamy2016} is used for LOG, DH, and NNPU, and the density-based method is used for WMMD. 
We denote the infeasible cases due to \lq{}KM1\rq{} method by the hyphen.
Other details are given in Table \ref{t:result_real_data_known}.}
\label{t:result_real_data_unknown}
{\fontsize{6}{6}
\begin{tabular}{rccccccccc}
\toprule
Dataset & LOG & DH & NNPU & tADJ & WMMD \\ 
\midrule
\multicolumn{5}{l}{Accuracy (in \%)}\\
{heart\_scale}  & $39.5\pm0.8$ & $39.1\pm0.7$ & $42.4\pm0.9$ & $65.1\pm0.8$  & $\boldsymbol{70.5\pm0.7}$ \\ 
{sonar\_scale} & $53.5\pm0.6$ & $52.1\pm0.5$ & $\boldsymbol{59.9\pm0.7}$ & $\boldsymbol{60.7\pm0.6}$  & $54.3\pm0.8$ \\ 
{australian\_scale} & $50.0\pm0.2$ & $50.0\pm0.2$ & $50.0\pm0.2$ & $\boldsymbol{80.0\pm0.7}$  & $\boldsymbol{79.4\pm1.0}$ \\
australian\_scale2 & $84.9\pm0.2$ & $84.9\pm0.2$ & $\boldsymbol{85.6\pm0.3}$ & $\boldsymbol{85.5\pm1.0}$  & $80.2\pm1.0$  \\
{breast-cancer\_scale} & $65.0\pm0.2$ & $65.2\pm0.2$ & $65.0\pm0.2$ & $\boldsymbol{91.0\pm0.4}$  & $\boldsymbol{93.0\pm1.0}$  \\
{breast-cancer\_scale2} & $35.0\pm0.2$ & $35.0\pm0.2$ & $35.0\pm0.2$ & $\boldsymbol{92.5\pm0.2}$ & $\boldsymbol{92.6\pm0.4}$  \\
{diabetes\_scale}  & $36.0\pm0.3$ & $41.1\pm1.1$ & $37.9\pm0.5$ & $\boldsymbol{67.9\pm0.3}$  & $65.1\pm0.2$  \\
{skin\_nonskin}  & - & - & - & $78.0\pm0.4$ & $\boldsymbol{82.2\pm0.9}$  \\ 
{skin\_nonskin2}  & - & - & - & $93.9\pm0.1$ & $\boldsymbol{95.7\pm0.4}$   \\ 
{epsilon\_normalized} & - & - & -  & $\boldsymbol{63.1\pm0.1}$  & $49.9\pm0.1$ \\ 
{HIGGS} & - & - & -  & $\boldsymbol{52.6\pm0.1}$  & ${50.9\pm0.0}$ \\ 
\midrule
\multicolumn{5}{l}{AUC $\times 100$}\\
{heart\_scale}  &  $67.2\pm1.8$ & $67.2\pm1.4$ & $71.0\pm0.9$ & $72.5\pm1.1$  & $\boldsymbol{77.1\pm0.9}$ \\
{sonar\_scale} & $60.5\pm0.9$ & $62.2\pm0.9$ & ${66.6\pm0.8}$ & $66.2\pm0.7$  & $\boldsymbol{69.7\pm0.8}$ \\
{australian\_scale} & $78.4\pm1.1$ & $72.5\pm1.4$ & ${80.3\pm0.6}$ & $87.8\pm0.5$ & $\boldsymbol{90.3\pm0.3}$ \\
australian\_scale2 & $\boldsymbol{92.9\pm0.2}$ & $89.6\pm0.6$  & $85.9\pm0.7$ & $92.4\pm0.3$ & $\boldsymbol{93.3\pm0.2}$ \\
{breast-cancer\_scale} & ${98.9\pm0.1}$ & $93.5\pm1.8$ & $54.8\pm1.1$ & $95.6\pm0.4$  & $\boldsymbol{99.5\pm0.0}$ \\
{breast-cancer\_scale2} & $14.4\pm2.0$ & $19.4\pm3.8$ & $91.5\pm0.3$ & $97.2\pm0.2$ & $\boldsymbol{99.0\pm0.1}$  \\
{diabetes\_scale}  & $64.0\pm1.3$ & $63.8\pm1.4$ & $72.6\pm0.5$ & $\boldsymbol{74.7\pm0.5}$  & $\boldsymbol{75.9\pm0.5}$ \\
{skin\_nonskin}  & - & - & - & $94.8\pm0.1$ & $\boldsymbol{99.5\pm0.1}$  \\ 
{skin\_nonskin2}  & - & - & - & $94.6\pm0.0$ & $\boldsymbol{99.8\pm0.0}$ \\ 
{epsilon\_normalized} & - & - & -  & $\boldsymbol{69.3\pm0.1}$  & $59.7\pm1.8$ \\ 
{HIGGS} & - & - & -  & $\boldsymbol{65.3\pm0.1}$  & ${55.4\pm0.2}$ \\ 
\bottomrule
\end{tabular}
\selectfont}

\end{table*}

\section{Concluding remarks}
Existing methods use different objective functions and hypothesis spaces, and as a consequence, different optimization algorithms.
Hence, there is no reason that one method outperforms uniformly for all scenarios.
It is possible that one particular method may outperform in one scenario, for example, NNPU proposed by \citet{kiryo2017} would perform better in complicated data settings because of the expressive power of neural networks. 
However, the proposed method has a clear computational advantage due to the closed-form as well as theoretical strength in terms of the explicit excess risk bound.
Further, the proposed method works reasonably well in both cases in which $\pi_{+}$ is known or unknown.
In this regard, we believe the proposed method can be used as a principled and easy-to-compute baseline algorithm in PU learning.

\section{Acknowledgement}
YK, WK, and MCP were supported by the National Research Foundation of Korea under grant NRF-2017R1A2B4008956.
MS was supported by JST CREST JPMJCR1403.

\appendix

\section{Proof of Theorem \ref{thm:relationship}}
\label{app:relationship}

\begin{proof}[Proof of Theorem \ref{thm:relationship}]
Since a function $f \in \mathcal{F}$ is bounded by 1, we have $\ell_{\mathrm{h}}(yf(x)) = \max(0, 1-yf(x)) = 1-yf(x)$.
Then, from Equation \eqref{eq:pu_risk}, 
\begin{align*}
R_{\ell_{\mathrm{h}}}(f) &= \pi_{+} \int_{\mathcal{X}}  \ell_{\mathrm{h}}(f(x))- \ell_{\mathrm{h}} (-f(x) ) dP_{X \mid Y=1}(x) + \int_{\mathcal{X}}  \ell_{\mathrm{h}} (-f(x)) dP_X(x) \\
&= 1 + \int_{\mathcal{X}}  f(x) dP_X(x)  -2 \pi_{+} \int_{\mathcal{X}} f(x) dP_{X \mid Y=1}(x).
\end{align*}
Thus, we have
\begin{align*}
\inf_{f \in \mathcal{F}} R_{\ell_{\mathrm{h}}}(f) &= 1 + \inf_{f \in \mathcal{F}} \left\{ \int_{\mathcal{X}}  f(x) dP_X(x)  -2 \pi_{+} \int_{\mathcal{X}} f(x) dP_{X \mid Y=1}(x) \right\} \\
&= 1-  \sup_{f \in \mathcal{F}} \left\{ - \int_{\mathcal{X}}  f(x) dP_X(x)  + 2 \pi_{+} \int_{\mathcal{X}} f(x) dP_{X \mid Y=1}(x) \right\} \\
&\overset{(\ast)}{=} 1 - \sup_{f \in \mathcal{F}} \left\{ \int_{\mathcal{X}} f(x) dP_{X}(x) - 2\pi_{+} \int_{\mathcal{X}} f(x) dP_{X \mid Y=1}(x) \right\} \\
&= 1 - {\rm WIPM} (P_{X}, P_{X \mid Y = 1} ; 2\pi_{+}, \mathcal{F}).
\end{align*}
Equation ($\ast$) holds because $\mathcal{F}$ is symmetric.

For the second result, note that a WIPM optimizer $g_{\mathcal{F}}$ satisfies ${\rm WIPM} (P_{X}, P_{X \mid Y = 1} ; 2\pi_{+}, \mathcal{F}) $ $ =  \int_{\mathcal{X}}  g_{\mathcal{F}}(x) dP_X(x)$ $-2\pi_{+} \int_{\mathcal{X}} g_{\mathcal{F}}(x) dP_{X \mid Y=1}(x)$.
Thus, we have
\begin{align*}
\inf_{f \in \mathcal{F}} R_{\ell_{\mathrm{h}}}(f) &= 1 - {\rm WIPM} (P_{X}, P_{X \mid Y = 1} ; 2\pi_{+}, \mathcal{F}) \\
&= 1 - \left\{ \int_{\mathcal{X}}  g_{\mathcal{F}}(x) dP_X(x) -2\pi_{+} \int_{\mathcal{X}} g_{\mathcal{F}}(x) dP_{X \mid Y=1}(x) \right\} \\
&= 1 - \left\{ R_{\ell_{\mathrm{h}}}(g_{\mathcal{F}}) - 1 \right\} \\
&= 2 - R_{\ell_{\mathrm{h}}}(g_{\mathcal{F}}) = R_{\ell_{\mathrm{h}}}(-g_{\mathcal{F}}).
\end{align*}
The last equality is from $R_{\ell_{\mathrm{h}}}(g_{\mathcal{F}}) + R_{\ell_{\mathrm{h}}}(-g_{\mathcal{F}}) = 2$ due to $g_{\mathcal{F}} \in \mathcal{F} \subseteq \mathcal{M}$.
\end{proof}


\section{Proofs for Section \ref{s:theory}: Theoretical properties of empirical WIPM optimizer}
In this section, we present a proof of Theorem \ref{thm:estimation_error} in Appendix \ref{app:estimation_error}.
We also provide a proof for Proposition \ref{prop:sharper} in Appendix \ref{app:sharper}.
Before presenting the proof for Theorem \ref{thm:estimation_error}, we begin with necessary technical proposition and lemma in Appendix \ref{app:preliminary}.

\subsection{Preliminaries for Theorem \ref{thm:estimation_error}}
\label{app:preliminary}

In supervised binary classification settings, \citet{sriperumbudur2012} introduced an empirical estimator for IPM and developed its consistency result. 
In Proposition \ref{prop:consistency_wipm}, we recreate theoretical results for PU learning settings, giving a consistency result of empirical WIPM estimator.

\begin{proposition}[Consistency result of WIPM estimator]
Let $\mathcal{F}$ be the symmetric function space such that  $\norm{f}_{\infty} \leq \nu$, $\mathrm{Var}_{{P}_{X \mid Y=1}} (f) \leq \sigma_{X \mid Y=1} ^2$, and $\mathrm{Var}_{{P}_{X}} (f) \leq \sigma_{X} ^2$.
Denote $\rho^2 = \sigma_{X \mid Y=1} ^2 \vee \sigma_{X} ^2$.
Then for all $w, \alpha, \tau >0$, the following holds with probability at least $1-e^{-\tau}$ over the choice of $\mathcal{X}_{\rm pu} := \{x_1 ^{\rm p}, \dots, x_{n_{\rm p}} ^{\rm p}, x_1 ^{\rm u}, \dots, x_{n_{\rm u}} ^{\rm u}\} \sim P_{\rm pu} := P_{X \mid Y=1} ^{n_{\rm p}} \times P_{X} ^{n_{\rm u}}$,
\begin{align}
& | {\rm WIPM}({P}_{X, {n_{\rm u}}}, {P}_{X \mid Y=1, {n_{\rm p}}}; w, \mathcal{F}) - {\rm WIPM}({P}_{X}, {P}_{X \mid Y=1}; w, \mathcal{F}) | \notag \\
&\leq 2(1+\alpha) [ \mathbb{E}_{P_{X}^{n_{\rm u}}} \{ \mathfrak{R}_{\mathcal{X}_{\rm u}}(\mathcal{F}) \} + w \mathbb{E}_{P_{X \mid Y=1}^{n_{\rm p}}} \{ \mathfrak{R}_{\mathcal{X}_{\rm p}}(\mathcal{F})\} ] \notag \\
&+  \chi_{n_{\rm p}, n_{\rm u}} ^{(1)}(w) \sqrt{ 2 \tau \rho^2 }  + \tau \chi_{n_{\rm p}, n_{\rm u}} ^{(2)}(w) \nu \left( \frac{2}{3}   + \frac{1}{\alpha} \right). 
\label{eq:wipm_consistency}
\end{align}
\label{prop:consistency_wipm}
\end{proposition}

\begin{proof}[Proof of Proposition \ref{prop:consistency_wipm}]
The following proof is a slight modification of the proof of Theorem 3.3 in \citet{sriperumbudur2012}. Without loss of generality, by changing an order, we define a set of observations and a set of weights as follows,
$$(x_1, \dots, x_{n_{\rm p}}, x_{n_{\rm p}+1}, \dots, x_{n_{\rm p}+n_{\rm u}} ) := \mathcal{X}_{\rm pu},$$
and
$$(\tilde{y}_1, \dots, \tilde{y}_{n_{\rm p}}, \tilde{y}_{n_{\rm p}+1}, \dots, \tilde{y}_{n_{\rm p}+n_{\rm u}}) = (w/n_{\rm p}, \dots, w/n_{\rm p}, -1/n_{\rm u}, \dots, -1/n_{\rm u}),$$ respectively. 
For independent Rademacher random variables $\{\sigma_i\}_{i=1} ^{n_{\rm p}+n_{\rm u}}$, we define the empirical Rademacher complexity-like term given by
$$\tilde{\mathfrak{R}}_{\mathcal{X}_{\rm pu}}(\mathcal{F}) := \mathbb{E}_{\sigma} \left\{ \sup_{f \in \mathcal{F}} \left|  \sum_{i=1} ^{n_{\rm p}+n_{\rm u}} \sigma_i \tilde{y}_i f(X_i) \right| : (X_1, \dots, X_{n_{\rm p}+n_{\rm u}}) = \mathcal{X}_{\rm pu} \right\}.$$
Note that $\tilde{\mathfrak{R}}_{\mathcal{X}_{\rm pu}}(\mathcal{F}) \leq \mathfrak{R}_{\mathcal{X}_{\rm u}}(\mathcal{F}) + w \mathfrak{R}_{\mathcal{X}_{\rm p}}(\mathcal{F})$.
Define $\mu_i = {P}_{X \mid Y =1}$ for $i \in \{1, \dots, n_{\rm p}\}$ and $\mu_i = {P}_{X}$ for $i \in \{n_{\rm p}+1, \dots, n_{\rm p}+n_{\rm u}\}$, respectively. 
That is, $P_{\rm pu} = \times_{i=1} ^{n_{\rm p}+ n_{\rm u}} \mu_i$. Let $(X_1, \dots, X_{n_{\rm p}+n_{\rm u}}) \sim P_{\rm pu}$ and define random variables $\theta_i (f, X_i) = w\{f(X_i) - {P}_{X \mid Y =1}(f) \}/n_{\rm p}$ for $i \in \{1, \dots, n_{\rm p}\}$ and $\theta_i (f, X_i) = \{f(X_i) - {P}_{X}(f) \}/n_{\rm u}$ for $i \in \{n_{\rm p}+1, \dots, n_{\rm p}+n_{\rm u}\}$, respectively.

Then, using the fact that $\left| \sup |C| - \sup |D| \right| \leq \sup |C-D|$, we have
\begin{align}
\label{sup_bound} 
& | {\rm WIPM}({P}_{X, n_{\rm u}}, {P}_{X \mid Y=1, n_{\rm p}}; w, \mathcal{F}) - {\rm WIPM}({P}_{X}, {P}_{X \mid Y=1}; w, \mathcal{F}) | \notag \\
&\leq  \sup_{f \in \mathcal{F}} \Big| \{ \int f d({P}_{X, n_{\rm u}}- w {P}_{X \mid Y=1, n_{\rm p}}) \} -  \{ \int f d({P}_{X}- w {P}_{X \mid Y=1}) \} \Big| \\
&=  \sup_{f \in \mathcal{F}} \Big| \sum_{i=1} ^{n_{\rm p}+n_{\rm u}} \theta_i (f, X_i) \Big| =:  h(X_1, \cdots, X_{n_{\rm p}+n_{\rm u}}). \notag 
\end{align}
Further, it is easy to verify that (i) $\int_{\mathcal{X}} \theta_i (f, z) d\mu_i(z) = 0$ for all $i$ and $f \in \mathcal{F}$ and (ii) $\int_{\mathcal{X}} \theta_i ^2 (f, z) d\mu_i(z)$ is bounded by $w^2{\sigma^2 _{X \mid Y =1}}/{n_{\rm p} ^2}$ for $i \in \{1, \dots, n_{\rm p}\}$ and ${\sigma^2 _{X}}/{n_{\rm u}^2}$ for $i \in \{n_{\rm p}+1, \dots, n_{\rm p}+n_{\rm u}\}$, respectively.
Finally, (iii) $\norm{\theta_i (f, \cdot) }_{\infty} \leq 2 (w/n_{\rm p} + 1/n_{\rm u}) \nu=  \chi_{n_{\rm p}, n_{\rm u}} ^{(2)}(w) \nu $ for all $i \in \{1, \dots, n_{\rm p}+n_{\rm u} \}$. 

Then, for all $\alpha >0$, the following holds with probability at least $1-e^{-\tau}$,
\begin{align*}
& h(X_1, \cdots, X_{n_{\rm p}+n_{\rm u}}) \notag \\
&\leq (1+\alpha) \mathbb{E}_{P_{\rm pu}} (h) + \sqrt{2 \tau \frac{ (w^2 n_{\rm u}+n_{\rm p})}{n_{\rm p} n_{\rm u}} (\sigma_{X \mid Y=1}^2 \vee \sigma_{X} ^2) } + \tau \chi_{n_{\rm p}, n_{\rm u}} ^{(2)}(w) \nu  \left( \frac{2}{3} + \frac{1}{\alpha} \right) \notag \\
&\leq 2(1+\alpha) \mathbb{E}_{P_{\rm pu}} \{ \tilde{\mathfrak{R}}_{\mathcal{X}_{\rm pu}}(\mathcal{F}) \} +  \chi_{n_{\rm p}, n_{\rm u}} ^{(1)}(w) \sqrt{ 2 \tau \rho^2 }  + \tau \chi_{n_{\rm p}, n_{\rm u}} ^{(2)}(w) \nu  \left( \frac{2}{3}   + \frac{1}{\alpha} \right).
\end{align*}
The first inequality is derived by the second inequality of Lemma \ref{lem:talagrand} in Appendix \ref{app:preliminary}, a variant of the Talagrand\rq{}s inequality. The second inequality is from using a symmetrization lemma: with corresponding independent ghost empirical distributions $\tilde{P}_{X, n_{\rm u}}$ and $\tilde{P}_{X \mid Y=1, n_{\rm p}}$, 

\begin{align*}
& \mathbb{E}_{P_{\rm pu}} (h) \\
&= \mathbb{E}_{P_{\rm pu}} \sup_{f \in \mathcal{F}} \Big| \{ \int f d({P}_{X, n_{\rm u}}- w {P}_{X \mid Y=1, n_{\rm p}}) \} -  \{ \int f d({P}_{X}- w {P}_{X \mid Y=1}) \} \Big| \\
&= \mathbb{E}_{P_{\rm pu}} \sup_{f \in \mathcal{F}} \Big| \{ \int f d({P}_{X, n_{\rm u}}- w {P}_{X \mid Y=1, n_{\rm p}}) \} -   \mathbb{E}_{P_{\rm pu}} \{ \int f d(\tilde{P}_{X, n_{\rm u}}- w \tilde{P}_{X \mid Y=1, n_{\rm p}}) \} \Big| \\
&\leq \mathbb{E}_{P_{\rm pu}} \sup_{f \in \mathcal{F}} \Big| \{ \int f d({P}_{X, n_{\rm u}}- w {P}_{X \mid Y=1, n_{\rm p}}) \} -   \{ \int f d(\tilde{P}_{X, n_{\rm u}}- w \tilde{P}_{X \mid Y=1, n_{\rm p}}) \} \Big| \\
&\leq 2\mathbb{E}_{P_{pu}} \sup_{f \in \mathcal{F}} \Big|  \int f d({P}_{X, n_{\rm u}}- w {P}_{X \mid Y=1, n_{\rm p}}) \Big| \\
&\leq 2 \mathbb{E}_{P_{\rm pu}} \mathbb{E}_{\sigma} \left\{ \sup_{f \in \mathcal{F}} \left|  \sum_{i=1} ^{n_{\rm p}+n_{\rm u}} \sigma_i \tilde{y}_i f(X_i) \right| : (X_1, \dots, X_{n_{\rm p}+n_{\rm u}}) = \mathcal{X}_{\rm pu} \right\}\\
&\leq 2 \mathbb{E}_{P_{\rm pu}} \{ \tilde{\mathfrak{R}}_{\mathcal{X}_{\rm pu}}(\mathcal{F})  \}
\end{align*}

Next, simply using the fact $\tilde{\mathfrak{R}}_{\mathcal{X}_{\rm pu}}(\mathcal{F}) \leq \mathfrak{R}_{\mathcal{X}_{\rm u}}(\mathcal{F}) + w \mathfrak{R}_{\mathcal{X}_{\rm p}}(\mathcal{F})$, we have for all $w, \alpha, \tau >0$, the following holds with probability at least $1-e^{-\tau}$,
\begin{align*}
& | {\rm WIPM}({P}_{X, n_{\rm u}}, {P}_{X \mid Y=1, n_{\rm p}}; w, \mathcal{F}) - {\rm WIPM}({P}_{X}, {P}_{X \mid Y=1}; w, \mathcal{F}) | \\
&\leq 2(1+\alpha) [ \mathbb{E}_{P_{X}^{n_{\rm u}}} \{ \mathfrak{R}_{\mathcal{X}_{\rm u}}(\mathcal{F}) \} + w \mathbb{E}_{P_{X \mid Y=1}^{n_{\rm p}}} \{ \mathfrak{R}_{\mathcal{X}_{\rm p}}(\mathcal{F})\} ] \\
&+  \chi_{n_{\rm p}, n_{\rm u}} ^{(1)}(w) \sqrt{ 2 \tau \rho^2 }  + \tau \chi_{n_{\rm p}, n_{\rm u}} ^{(2)}(w) \nu  \left( \frac{2}{3}   + \frac{1}{\alpha} \right).
\end{align*}
It concludes the proof.
\end{proof}

For Lemmas \ref{lem:talagrand}, we quote the Proposition B.1 of \citet{sriperumbudur2012} without proofs.

\begin{lemma}[Proposition B.1 of \citet{sriperumbudur2012}: A variant of Talagrand\rq{}s inequality]
\label{lem:talagrand}
Let $B \geq 0, n \geq 1, (\Omega_i, \mathcal{A}_i, \mu_i), i = 1, \dots, n$ be a probability space and $\theta_i : \mathcal{F} \times \Omega_i \to \mathbb{R}$ be bounded measurable functions, where $\mathcal{F}$ is the space of real-valued $\mathcal{A}_i$-measurable functions for all $i$. Suppose
\begin{itemize}
\item[(a)] $\int_{\Omega_i} \theta_i (f, \omega) d\mu_i (\omega) =0$ for all $i$ and $f \in \mathcal{F}$
\item[(b)] $\int_{\Omega_i} \theta_i ^2 (f, \omega) d\mu_i (\omega) \leq \rho_i ^2$ for all $i$ and $f \in \mathcal{F}$
\item[(c)] $\norm{\theta_i (f, \cdot)}_{\infty}  \leq B$ for all $i$ and $f \in \mathcal{F}$.
\end{itemize}
Define $Z := \times_{i=1} ^n \Omega_i$ and $P := \times_{i=1} ^n \mu_i$. Furthermore, define $g: Z \to \mathbb{R}$ by
$$
g(z) :=  \sup_{f \in \mathcal{F}} \left| \sum_{i=1} ^n \theta_i (f, \omega_i) \right|, z=(\omega_1, \dots, \omega_n) \in Z.
$$
Then, for all $\tau >0$, we have
$$
P\left( \left\{ z \in Z: g(z) \geq \mathbb{E}_P g + \sqrt{2\tau \left( \sum_{i=1} ^n \rho_i ^2 + 2B \mathbb{E}_P g \right) } + \frac{2 \tau B}{3} \right\} \right) \leq e^{-\tau}.
$$
In addition, for all $\tau >0$ and $\alpha >0$,
$$
P\left( \left\{ z \in Z: g(z) \geq (1+\alpha) \mathbb{E}_P g + \sqrt{2\tau  \sum_{i=1} ^n \rho_i ^2  } + \tau B \left( \frac{2 }{3} + \frac{1}{\alpha} \right) \right\} \right) \leq e^{-\tau}.
$$
\end{lemma}

\subsection{Proof of Theorem \ref{thm:estimation_error}: estimation error bound of WIPM optimizer}
\label{app:estimation_error}

\begin{proof}[Proof of Theorem \ref{thm:estimation_error}]
We first define the empirical risk estimator $\hat{R}_{\ell_{\mathrm{h}}}(f)$ by replacing data distributions in Equation \eqref{eq:pu_risk} with the empirical distributions. To be more specific, we define
\begin{align*}
\hat{R}_{\ell_{\mathrm{h}}}(f) &:= \pi_{+} \int_{\mathcal{X}}  [\ell_{\mathrm{h}}\{f(x)\}- \ell_{\mathrm{h}}\{-f(x)\}] dP_{X \mid Y=1, n_{\rm p}}(x) + \int_{\mathcal{X}}  \ell_{\mathrm{h}}\{-f(x)\} dP_{X, n_{\rm u}}(x)\\
&= \frac{ \pi_{+}}{n_{\rm p}} \sum_{i=1} ^{n_{\rm p}} [\ell_{\mathrm{h}}\{f( x_i ^{\rm p} )\}- \ell_{\mathrm{h}}\{-f( x_i ^{\rm p} )\}]  + \frac{1}{n_{\rm u}} \sum_{i=1} ^{n_{\rm u}}  \ell_{\mathrm{h}}\{-f( x_i ^{\rm u} )\} .
\end{align*}
Since $\norm{f}_{\infty} \leq 1$, using the similar derivations in Appendix \ref{app:relationship}, we have
$$
\hat{R}_{\ell_{\mathrm{h}}}(f) = 1+ \frac{1}{n_{\rm u}} \sum_{i=1} ^{n_{\rm u}} f( x_i ^{\rm u}) - \frac{2 \pi_{+}}{n_{\rm p}} \sum_{i=1} ^{n_{\rm p}} f( x_i ^{\rm p}).    
$$
By the result of Theorem \ref{thm:relationship}, the negative of an WIPM optimizer ${f}_{\mathcal{F}}$ is minimizer of $R_{\ell_{\mathrm{h}}} (f)$, {\it i.e.}, $ R_{\ell_{\mathrm{h}}}({f}_{\mathcal{F}}) = \inf_{f \in \mathcal{F}} R_{\ell_{\mathrm{h}}} (f)$. Thus, we have
\allowdisplaybreaks
\begin{align*}
&R_{\ell_{\mathrm{h}}}(\hat{f}_{\mathcal{F}}) - \inf_{f \in \mathcal{F}} R_{\ell_{\mathrm{h}}} (f)  \\ &= R_{\ell_{\mathrm{h}}}(\hat{f}_{\mathcal{F}}) - R_{\ell_{\mathrm{h}}}({f}_{\mathcal{F}})  \\
&= \left\{ R_{\ell_{\mathrm{h}}}(\hat{f}_{\mathcal{F}}) - \hat{R}_{\ell_{\mathrm{h}}}(\hat{f}_{\mathcal{F}}) \right\} + \left\{ \hat{R}_{\ell_{\mathrm{h}}}(\hat{f}_{\mathcal{F}}) - \hat{R}_{\ell_{\mathrm{h}}}(f_{\mathcal{F}}) \right\} + \left\{ \hat{R}_{\ell_{\mathrm{h}}}(f_{\mathcal{F}}) - R_{\ell_{\mathrm{h}}}(f_{\mathcal{F}}) \right\}  \\
&\leq \sup_{f \in \mathcal{F} } | \hat{R}_{\ell_{\mathrm{h}}}({f}) - R_{\ell_{\mathrm{h}}}({f})|+ 0 + \sup_{f \in \mathcal{F} } | \hat{R}_{\ell_{\mathrm{h}}}({f}) - R_{\ell_{\mathrm{h}}}({f})|  \\
&= 2 \sup_{f \in \mathcal{F} } | \hat{R}_{\ell_{\mathrm{h}}}({f}) - R_{\ell_{\mathrm{h}}}({f})|. 
\end{align*}
The first inequality holds since $\hat{R}_{\ell_{\mathrm{h}}}(\hat{f}_{\mathcal{F}}) = 1 - {\rm WIPM}({P}_{X, n_{\rm u}}, {P}_{X \mid Y=1, n_{\rm p}}; 2\pi_{+}, \mathcal{F}) \leq \hat{R}_{\ell_{\mathrm{h}}}(f)$ for any $f \in \mathcal{F}$.
Thus it is enough to bound $\sup_{f \in \mathcal{F} } | \hat{R}_{\ell_{\mathrm{h}}}({f}) - R_{\ell_{\mathrm{h}}}({f})|$, and
\begin{align*}
&\sup_{f \in \mathcal{F} } | \hat{R}_{\ell_{\mathrm{h}}}({f}) - R_{\ell_{\mathrm{h}}}({f})| \\
&\leq \sup_{f \in \mathcal{F} } \Big| \{ \frac{1}{n_{\rm u}} \sum_{i=1} ^{n_{\rm u}} f( x_i ^{\rm u}) - \frac{2 \pi_{+}}{n_{\rm p}} \sum_{i=1} ^{n_{\rm p}} f( x_i ^{\rm p}) \} - \{ \int f d({P}_{X}- 2\pi_{+} {P}_{X \mid Y=1}) \} \Big|  \\
&= \sup_{f \in \mathcal{F}} \Big| \{ \int f d({P}_{X, n_{\rm u}}- 2\pi_{+} {P}_{X \mid Y=1, n_{\rm p}}) \} -  \{ \int f d({P}_{X}- 2\pi_{+} {P}_{X \mid Y=1}) \} \Big|.
\end{align*}
Note that this is a special case of Equation \eqref{sup_bound}. 
Therefore, applying Equation \eqref{eq:wipm_consistency} in Proposition \ref{prop:consistency_wipm} with $w=2\pi_{+}$, we have for all $\alpha, \tau >0$, the following holds with probability at least $1-e^{-\tau}$,
\allowdisplaybreaks
\begin{align*}
&R_{\ell_{\mathrm{h}}}(\hat{f}_{\mathcal{F}}) - \inf_{f \in \mathcal{F} } R_{\ell_{\mathrm{h}}} (f) \\
&\leq 4(1+\alpha) [ \mathbb{E}_{P_{X} ^{n_{\rm u}} } \{ \mathfrak{R}_{\mathcal{X}_{\rm u}}(\mathcal{F}) \} + 2\pi_{+} \mathbb{E}_{P_{X \mid Y=1} ^{n_{\rm p}} } \{ \mathfrak{R}_{\mathcal{X}_{\rm p}}(\mathcal{F})\} ] +  2\chi_{n_{\rm p}, n_{\rm u}} ^{(1)}(2\pi_{+}) \sqrt{ 2 \tau \rho^2 }  \\
&+ 2\tau \chi_{n_{\rm p}, n_{\rm u}} ^{(2)}(2\pi_{+}) \nu \left( \frac{2}{3}   + \frac{1}{\alpha} \right).
\end{align*}
This concludes a proof.
\end{proof}

\subsection{Proof of Proposition \ref{prop:sharper}}
\label{app:sharper}

\begin{proof}[Proof of Proposition \ref{prop:sharper}]
After omitting the positive term $\Delta$ from the upper bound \eqref{eq:kiryo_ineq} and plugging $\alpha=1$ in the upper bound \eqref{eq:convergence_optim}, it is enough to show that, under the condition \eqref{eq:condition}, the following is satisfied.
$$
8R+ \chi_{n_{\rm p}, n_{\rm u}} ^{(1)} (2\pi_{+}) (1 + \nu) \sqrt{2\tau}  \geq 8R + 2\chi_{n_{\rm p}, n_{\rm u}} ^{(1)} (2\pi_{+}) \sqrt{ 2 \tau \rho^2 }  + \frac{10}{3}\tau \chi_{n_{\rm p}, n_{\rm u}} ^{(2)}(2\pi_{+}) \nu ,
$$
where $ R = \mathbb{E}_{P_{X} ^{n_{\rm u}} } \{ \mathfrak{R}_{\mathcal{X}_{\rm u}}(\mathcal{F}) \} + 2\pi_{+} \mathbb{E}_{P_{X \mid Y=1} ^{n_{\rm p}} } \{ \mathfrak{R}_{\mathcal{X}_{\rm p}}(\mathcal{F})\}$.

Using simple algebras, we have
\allowdisplaybreaks
\begin{align*}
& \frac{1+\nu}{2} - \frac{ 5 \sqrt{2 \tau }\chi_{n_{\rm p}, n_{\rm u}} ^{(2)}(2\pi_{+}) \nu}{ 6 \chi_{n_{\rm p}, n_{\rm u}} ^{(1)} (2\pi_{+})}  \geq   \rho  \\
&\Longrightarrow \chi_{n_{\rm p}, n_{\rm u}} ^{(1)} (2\pi_{+})  (1+\nu) \sqrt{2 \tau }   - 2\tau \chi_{n_{\rm p}, n_{\rm u}} ^{(2)}(2\pi_{+}) \nu  \frac{5}{3}  \geq  2\chi_{n_{\rm p}, n_{\rm u}} ^{(1)}(2\pi_{+}) \sqrt{ 2 \tau \rho^2 }  \\
&\Longrightarrow 8R+ \chi_{n_{\rm p}, n_{\rm u}} ^{(1)} (2\pi_{+}) (1 + \nu) \sqrt{2\tau}  \geq 8R + 2\chi_{n_{\rm p}, n_{\rm u}} ^{(1)} (2\pi_{+}) \sqrt{ 2 \tau \rho^2 }  + \frac{10}{3}\tau \chi_{n_{\rm p}, n_{\rm u}} ^{(2)}(2\pi_{+}) \nu.
\end{align*}

Thus, the proposed upper bound is sharper than that of \citet{kiryo2017} if the condition \eqref{eq:condition} holds.
\end{proof}

\section{Proofs for Section \ref{s:RKHS}: The empirical WMMD optimizer and the WMMD classifier}

In this section, we first show that the empirical WMMD optimizer has a closed-form expression in Appendix \ref{app:MMD}.
We provide proofs of Lemmas \ref{lem:estimation_error_bound_RKHS} and \ref{lem:approximation_error_bound_RKHS} in Appendix \ref{app:rade_rkhs} and Theorem \ref{thm:excess_super_fast} in Appendix \ref{app:excess_super_fast}.

\subsection{Proof of Proposition \ref{prop:MMD_emp}: WMMD optimizer has a closed-form expression}
\label{app:MMD}
We first state and prove the following Proposition \ref{prop:MMD}, which is an extended version of Proposition \ref{prop:MMD_emp}.
Please note that Proposition \ref{prop:MMD_emp} can be directly obtained by plugging the two empirical distributions ${P}_{X \mid Y=1, n_{\rm p}}$ and  ${P}_{X, n_{\rm u}}$ and $w = 2\pi_{+}$.

\begin{proposition}[Weighted maximum mean discrepancy]
Let ${P}$ and ${Q}$ be two probability measures defined on $\mathcal{X}$ and let $k: \mathcal{X} \times \mathcal{X} \to \mathbb{R}$ a bounded reproducing kernel.

(a) WMMD between two probability measures ${P}$ and ${Q}$ with a weight $w$ and a closed ball $\mathcal{H}_{k,r}$ with the radius $r$ can be represented in a closed form, 
\begin{align*}
{\rm WMMD}_k ({P}, {Q}; w, r) &= r \left\{ \mathbb{E}_{{P}^2} [k(x,x^{\prime})] + w^2 \mathbb{E}_{{Q}^2} [k(y, y^{\prime})] - 2w \mathbb{E}_{{P}\times {Q}} [k(x, y)] \right\}^{1/2},
\end{align*}
where $x$ and $x^{\prime}$ independently follow ${P}$ and $y$, and $y^{\prime}$ independently follow ${Q}$.

(b) we also have a closed-form expression for the unique optimizer $g_{\mathcal{H}_{k,r}} \in \mathcal{H}_{k,r}$ given by
\begin{equation*}
g_{\mathcal{H}_{k,r}} (z) = r \times T \left( \int k(z, x) d{P}(x) - w \int k(z, x) d{Q}(x) \right), 
\end{equation*}
where $T$ is the normalizing operator defined by $T(g) = g/ {\norm{g}_{\mathcal{H}_k}}$.

(c) The associated classifier is given by
$${\rm sign}\{ f_{\mathcal{H}_{k,r}}(z)\} = \begin{cases} 
      +1 & {\rm if} \quad  w^{-1}  < {\lambda}_{{Q}, {P}} (z)    \\
      -1 & otherwise
   \end{cases},$$
where
$${\lambda}_{{Q}, {P}}(z) = \frac{ \int k(z, x) d{Q}(x) }{\int k(z, x) d{P}(x)}.$$
\label{prop:MMD}
\end{proposition}

\begin{proof}[Proof of Proposition \ref{prop:MMD}]
The main idea of this proof is to use $f(x) = {\langle f, k(\cdot, x) \rangle}_{\mathcal{H}_k}$ as \citet{gretton2012} showed. 
From the definition of the {\rm WMMD}, we have
\begin{align*}
{\rm WMMD}_k ({P}, {Q}; w, r) &:=  \sup_{f \in \mathcal{H}_{k,r}} \left| \int_{\mathcal{X}} f(x) d{P} (x) - w \int_{\mathcal{X}} f(x) d{Q}(x) \right| \\ 
&=  \sup_{f \in \mathcal{H}_{k,r}} \left| \int_{\mathcal{X}} {\langle f, k(\cdot, x) \rangle}_{\mathcal{H}_k} d{P} (x) - w \int_{\mathcal{X}} {\langle f, k(\cdot, x) \rangle}_{\mathcal{H}_k} d{Q}(x) \right| \\ 
&=  \sup_{f \in \mathcal{H}_{k,r}} \left| {\langle f, \int_{\mathcal{X}} k(\cdot, x) \{d{P} (x) -w d{Q}(x) \}  \rangle}_{\mathcal{H}_k} \right| \\ 
&= r \norm{ \int_{\mathcal{X}} k(\cdot, x) d{P}(x) - w \int_{\mathcal{X}} k(\cdot, x) d{Q}(x)} _{\mathcal{H}_k}.
\end{align*}
The last equation is obtained by Cauchy-Schwarz inequality with the WMMD optimizer $g_{\mathcal{H}_{k,r}} \in \mathcal{H}_{k,r}$ given by
$$g_{\mathcal{H}_{k,r}} (z) = r \times T \left( \int_{\mathcal{X}} k(z, x) d{P}(x) - w \int_{\mathcal{X}} k(z, x) d{Q}(x) \right).$$
It concludes a proof for (b).

Furthermore, 
\begin{align*}
&r^{-1} \times {\rm WMMD}_k ({P}, {Q}; w, r)\\
=& \norm{ \int_{\mathcal{X}} k(\cdot, x) d{P}(x) - w \int_{\mathcal{X}} k(\cdot, x) d{Q}(x)} _{\mathcal{H}_k} \\
=& {\Big\langle \int_{\mathcal{X}} k(\cdot, x) d{P}(x) - w \int_{\mathcal{X}} k(\cdot, x) d{Q}(x),  \int_{\mathcal{X}} k(\cdot, x) d{P}(x) - w \int k(\cdot, x) d{Q}(x) \Big\rangle}_{\mathcal{H}_k} ^{1/2} \\ 
=& \Big\{ \int_{\mathcal{X}} \int_{\mathcal{X}} k(x, x^{\prime}) d{P}(x) d{P}(x^{\prime}) + w^2 \int_{\mathcal{X}} \int_{\mathcal{X}} k(y, y^{\prime}) d{Q}(y) d{Q}(y^{\prime}) \\
&- 2w \int_{\mathcal{X}} \int_{\mathcal{X}} k(x, y) d{P}(x) d{Q}(y)  \Big\}^{1/2} \\ 
=& \left\{ \mathbb{E}_{{P}^2} [k(x,x^{\prime})] + w^2 \mathbb{E}_{{Q}^2} [k(y, y^{\prime})] - 2w \mathbb{E}_{{P}\times {Q}} [k(x, y)] \right\}^{1/2}.
\end{align*}
It concludes a proof for (a). 

Lastly, we prove the statement (c). 
Note that the associated classifier is determined by the sign of $f_{\mathcal{H}_{k,r}} = -g_{\mathcal{H}_{k,r}}$. 
From the statement (b), we have
$$
{\rm sign} ( g_{\mathcal{H}_{k,r}}(z)) = \begin{cases} 
      +1 & {\rm if} \quad \frac{\int k(z, x) d{P}(x)}{ \int k(z, x) d{Q}(x) } > w \\
      -1 & otherwise
   \end{cases}.
$$
Hence, we have
$${\rm sign} ( f_{\mathcal{H}_{k,r}}(z) ) = -{\rm sign} ( g_{\mathcal{H}_{k,r}}(z) ) = \begin{cases} 
      +1 & {\rm if} \quad w^{-1}  < {\lambda}_{{Q}, {P}} (z)   \\
      -1 & otherwise
   \end{cases}.$$
\end{proof}

\subsection{Lemmas for Section \ref{s:RKHS} }
\label{app:rade_rkhs}

In this subsection, we provide the proof for the two useful Lemmas: (i) explicitly showing the estimation error bound in Lemma \ref{lem:estimation_error_bound_RKHS} and (ii) deriving an approximation error bound in Lemma \ref{lem:approximation_error_bound_RKHS}.

\begin{proof}[Proof of Lemma \ref{lem:estimation_error_bound_RKHS}]
We first prove $\mathcal{H}_{k,r_1} \subseteq \mathcal{M}$.
By the reproducing property of $\mathcal{H}_{k}$ and Cauchy-Schwarz inequality, for any $f \in \mathcal{H}_{k,r_1}, x\in\mathcal{X}$ 
$$
|f(x)|=|\langle k(\cdot, x), f \rangle_{\mathcal{H}_{k}}| \le \norm{k(\cdot,x)}_{\mathcal{H}_{k}} \norm{f}_{\mathcal{H}_{k}} = \sqrt{k(x,x)} \norm{f}_{\mathcal{H}_{k}} \le r_1^{-1} r_1 = 1.
$$
Thus, $\norm{f}_{\infty} \leq 1$ for all $f \in \mathcal{H}_{k,r_1}$.
This proves $\mathcal{H}_{k,r_1} \subseteq \mathcal{M}$.

Now, we prove the inequality.
First, we apply Theorem \ref{thm:estimation_error} with $\mathcal{H}_{k,r_1}$.

[Step 1] 
From the result of Theorem \ref{thm:estimation_error}, for all $\alpha, \tau >0$, the estimation error term is bounded above with probability at least $1-e^{-\tau}$,
\begin{align}
R_{\ell_{\mathrm{h}}}(\hat{f}_{\mathcal{H}_{k,r_1}}) - \inf_{f \in \mathcal{H}_{k,r_1} } R_{\ell_{\mathrm{h}}} (f) \leq & C_{\alpha} (\mathbb{E}_{P_{X} ^{n_{\rm u}} } ( \mathfrak{R}_{\mathcal{X}_{\rm u}}(\mathcal{H}_{k,r_1}) )  + 2 \pi_{+} \mathbb{E}_{P_{X \mid Y =1 } ^{n_{\rm p}} }( \mathfrak{R}_{\mathcal{X}_{\rm p}}(\mathcal{H}_{k,r_1}) ) ) \notag \\ 
&+  C_{\tau, \rho^2}^{(1)} \chi_{n_{\rm p}, n_{\rm u}} ^{(1)} (2\pi_{+})  +  C_{\tau, \nu, \alpha}^{(2)} \chi_{n_{\rm p}, n_{\rm u}} ^{(2)} (2\pi_{+}). \label{eq:hilbert_est_1}
\end{align}

Using the notations $(\tilde{y}_1, \dots, \tilde{y}_{n_{\rm p}}, \tilde{y}_{n_{\rm p}+1}, \dots, \tilde{y}_{n_{\rm p}+n_{\rm u}}) $ $ = (2\pi_{+}/n_{\rm p}, $ $ \dots,  $ $ 2\pi_{+}/n_{\rm p}, $ $ -1/n_{\rm u}, \dots, -1/n_{\rm u})$, we obtain upper bound of the empirical Rademacher complexity of $\mathcal{H}_{k,r}$ given the positive samples $\mathfrak{R}_{\mathcal{X}_{\rm p}}(\mathcal{H}_{k,r})$ as follows.
\allowdisplaybreaks

\begin{align*}
&2\pi_{+} \mathfrak{R}_{\mathcal{X}_{\rm p}}(\mathcal{H}_{k,r_1}) \\
&= \mathbb{E}_{\sigma} \left\{ \sup_{f \in \mathcal{H}_{k,r_1}} \left|  \sum_{i=1} ^{n_{\rm p}} \sigma_i \tilde{y}_i f(X_i) \right| \right\} = \mathbb{E}_{\sigma} \left\{ \sup_{f \in \mathcal{H}_{k,r_1}}   \left| \sum_{i=1} ^{n_{\rm p}} \sigma_i \tilde{y}_i {\langle k(\cdot, X_i), f \rangle}_{\mathcal{H}_k} \right| \right\} \\
&= \mathbb{E}_{\sigma} \left\{ \sup_{f \in \mathcal{H}_{k,r_1}}  \left| {\left\langle \sum_{i=1} ^{n_{\rm p}} \sigma_i \tilde{y}_i k(\cdot, X_i), f \right\rangle}_{\mathcal{H}_k} \right| \right\} \leq r_1  \mathbb{E}_{\sigma} \left\{ \norm{ \sum_{i=1} ^{n_{\rm p}} \sigma_i \tilde{y}_i k(\cdot, X_i) }_{\mathcal{H}_k} \right\} \\
&\leq r_1  \sqrt{ \mathbb{E}_{\sigma} \left\{  \sum_{i=1} ^{n_{\rm p}}  \tilde{y}_i ^2 k(X_i, X_i)  \right\} } + r_1  \sqrt{ \mathbb{E}_{\sigma} \left\{  \sum_{i \neq j} ^{n_{\rm p}} \sigma_i \sigma_j \tilde{y}_i \tilde{y}_j k(X_i, X_j)  \right\} }\\
&= r_1 \sqrt{ \mathbb{E}_{\sigma} \left\{ \sum_{i=1} ^{n_{\rm p}}  \tilde{y}_i ^2 k(X_i, X_i)  \right\} } \leq  \frac{2\pi_{+}}{\sqrt{n_{\rm p}}}.
\end{align*}
We continue the similar method to the unlabeled dataset, and applying expectation operator gives
\begin{align}
&(\mathbb{E}_{P_{X} ^{n_{\rm u}} } ( \mathfrak{R}_{\mathcal{X}_{\rm u}}(\mathcal{H}_{k,r_1}) )  + 2 \pi_{+} \mathbb{E}_{P_{X \mid Y =1 } ^{n_{\rm p}} }( \mathfrak{R}_{\mathcal{X}_{\rm p}}(\mathcal{H}_{k,r_1}) ) ) \leq  \frac{1}{\sqrt{n_{\rm u}}} +  \frac{2\pi_{+}}{\sqrt{n_{\rm p}}} =  \chi_{n_{\rm p}, n_{\rm u}} ^{(1)} (2\pi_{+}).
\label{eq:hilbert_est_2}
\end{align}

[Step 2] Using Equations \eqref{eq:hilbert_est_1} and \eqref{eq:hilbert_est_2}, we then conclude that for all $\alpha, \tau >0$, the following holds with probability at least $1-e^{-\tau}$,
\begin{align*}
R_{\ell_{\mathrm{h}}}(\hat{f}_{\mathcal{H}_{k,r_1}}) - \inf_{f \in \mathcal{H}_{k,r_1} }  R_{\ell_{\mathrm{h}}}(f) \leq ( C_{\alpha} + C_{\tau, \rho^2}^{(1)} )  \chi_{n_{\rm p}, n_{\rm u}} ^{(1)} (2\pi_{+})  +  C_{\tau, \nu, \alpha}^{(2)}  \chi_{n_{\rm p}, n_{\rm u}} ^{(2)} (2\pi_{+}).
\end{align*}
These equations conclude the proof.
\end{proof}

\begin{proof}[Proof of Lemma \ref{lem:approximation_error_bound_RKHS}]

[Step 1]
In this step, we first claim that $\inf_{f \in \beta \mathcal{M} } R_{\ell_{\mathrm{h}}}(f) = R_{\ell_{\mathrm{h}}}(\beta f^*_1)$.
By \citet[Lemma 3.1]{lin2002}, the $f^*_1$ satisfies that $\inf_{f \in \mathcal{U} } R_{\ell_{\mathrm{h}}}(f) = \inf_{f \in \mathcal{M} } R_{\ell_{\mathrm{h}}}(f) $ $= R_{\ell_{\mathrm{h}}}( f^*_1)$.
Note that $\ell_{\mathrm{h}} ( yf(x) ) = \max (0, 1-yf(x)) = 1-yf(x)$ for all $\norm{f}_{\infty} \leq 1$.
It is obvious that $\beta f_1^* \in \beta \mathcal{M}$. 
By definition of the infimum, we have $R_{\ell_{\mathrm{h}}}(\beta f^*_1) \ge \inf_{f \in \beta \mathcal{M} } R_{\ell_{\mathrm{h}}}(f)$.
Suppose $R_{\ell_{\mathrm{h}}}(\beta f^*_1) > \inf_{f \in \beta \mathcal{M} } R_{\ell_{\mathrm{h}}}(f)$.
Let $f_{\beta}^*$ be a function in $\beta \mathcal{M} $ such that $R_{\ell_{\mathrm{h}}}(f^*_{\beta})
= \inf_{f \in \beta \mathcal{M} } R_{\ell_{\mathrm{h}}}(f)$
Then,
\begin{align}
R_{\ell_{\mathrm{h}}}(\beta f^*_1) > R_{\ell_{\mathrm{h}}}(f^*_{\beta})  \notag 
&\Longleftrightarrow  \int 1 - y \beta f^*_1(x) dP_{X,Y}(x,y)> \int 1 - y f^*_{\beta}(x) dP_{X,Y}(x,y)\\ \notag 
&\Longleftrightarrow \beta \int  y f^*_1(x) dP_{X,Y}(x,y) < \int  y f^*_{\beta}(x) dP_{X,Y}(x,y). \notag
\end{align}
Since $\beta^{-1} f_{\beta}^* \in \mathcal{M}$,
\begin{align*}
&R_{\ell_{\mathrm{h}}}(\beta^{-1} f^*_{\beta}) 
\ge \inf_{f \in \mathcal{M} } R_{\ell_{\mathrm{h}}}(f) \\
&\Longleftrightarrow \int 1 - y \beta^{-1} f^*_{\beta}(x) dP_{X,Y}(x,y) \ge \int 1 - y f^*_1(x) dP_{X,Y}(x,y)  \\
&\Longleftrightarrow \int y f^*_{\beta}(x) dP_{X,Y}(x,y)  \le \beta \int y  f^*_1(x) dP_{X,Y}(x,y). 
\end{align*}
Note that $\inf_{f \in \mathcal{M} } R_{\ell_{\mathrm{h}}}(f) = R_{\ell_{\mathrm{h}}}(f^*_1)$
This contradicts with the assumption $R_{\ell_{\mathrm{h}}}(\beta f^*_1) > \inf_{f \in \beta \mathcal{M} } R_{\ell_{\mathrm{h}}}(f)$, and we have $\inf_{f \in \beta \mathcal{M} } R_{\ell_{\mathrm{h}}}(f) = R_{\ell_{\mathrm{h}}}(\beta f^*_1)$.

[Step 2] By Lemma \ref{lem:estimation_error_bound_RKHS}, we have $\mathcal{H}_{k, r_1} \subseteq \mathcal{M}$.
Thus, for all $g \in \mathcal{H}_{k, r_1} \subseteq \mathcal{M}$ and $0<\beta \leq 1$, we have
\begin{align*}
&\inf_{f \in \mathcal{H}_{k, r_1} } R_{\ell_{\mathrm{h}}} (f) -\inf_{f \in \beta \mathcal{M} } R_{\ell_{\mathrm{h}}}(f) \\
&\leq R_{\ell_{\mathrm{h}}}(g) - R_{\ell_{\mathrm{h}}}( \beta f_1^*) \\ \notag
&= \int \{ (1-yg(x)) - ( 1-y \beta f_1^*(x))\} dP_{X,Y}(x,y)\\ \notag
&\leq \int |yg(x) - y \beta f_1^*(x)| dP_{X,Y}(x,y)\\ \notag
&\leq \sqrt{\int |y|^2 dP_{X,Y}(x,y)} \sqrt{\int |g(x)- \beta f_1^*(x)|^2 dP_{X,Y}(x,y)} \\ \notag
&= \norm{g- \beta f_1^*}_{L_2(P_X)}. \notag
\end{align*}
The first equality holds because $\ell_{\mathrm{h}} ( yg(x) ) = \max ( 0, 1-yg(x) ) = 1-yg(x)$ for all $g \in \mathcal{H}_{k, r_1} \subseteq \mathcal{M}$.
Hence, 
\begin{align*}
\inf_{f \in \mathcal{H}_{k,r_1} } R_{\ell_{\mathrm{h}}} (f) -\inf_{f \in \beta \mathcal{M} } R_{\ell_{\mathrm{h}}}(f)
&\leq \inf_{g \in \mathcal{H}_{k,{r_1}}} \norm{g-\beta f_1^*}_{L_2(P_X)} \\
&= \beta \inf_{g \in \mathcal{H}_{k,{r_1}}} \norm{ g /\beta  - f_1^*}_{L_2(P_X)} \\
&= \beta \inf_{g \in \mathcal{H}_{k,{r_1 / \beta }}} \norm{ g- f_1^*}_{L_2(P_X)}.
\end{align*}
Therefore, by [Step 1] and [Step 2], 
$$ \inf_{f \in \mathcal{H}_{k,r_1} } {R_{\ell_{\mathrm{h}}} ( f ) - \inf_{f \in \beta \mathcal{M} } R_{\ell_{\mathrm{h}}} (f)}  \leq \beta \inf_{g \in \mathcal{H}_{k,{r_1 / \beta }}} \norm{ g- f_1^*}_{L_2(P_X)},$$
for any $0 < \beta \leq 1$.
\end{proof}

\subsection{Proof of Theorem \ref{thm:excess_super_fast}}
\label{app:excess_super_fast}
We first state the following assumptions for Theorem \ref{thm:excess_super_fast}.
\begin{itemize}
\item[(A1)] The distribution functions $P_X$ and $P_{X \mid Y=1} $ have probability density functions $p_X(x)$ and $p_{X \mid Y=1}(x)$, respectively.
\item[(A2)] For $\alpha_{\text{H}} \in (0,1)$, the density functions $p_X(x)$ and $p_{X \mid Y=1}(x)$ are $\alpha_{\text{H}}$-H{\"o}lder continuous, \textit{i.e.}, $| g(x) - g(x^{\prime}) | \leq C_{\text{H}} |x - x^{\prime} |^{\alpha_{\text{H}}}$ for some constant $C_{\text{H}}$ and for all $x, x^{\prime} \in \mathbb{R}^{d}$ and $g \in \{ p_X, p_{X \mid Y=1} \}$.
\item[(A3)] The marginal density function is bounded away from zero, \textit{i.e.}, $p_X(x) \geq p_{\rm min} >0$ for all $x$ on its support.
\item[(A4)] The marginal distribution $P_X$ has Tsybakov's noise exponent $q \in [0, \infty)$, \textit{i.e.}, there exists a constant $C_{\text{noise}}>0$ such that for all sufficiently small $t>0$, we have
$$
P_X ( \{ x \in \mathcal{X} \mid |2 \eta(x) -1 | \leq t\} ) \leq C_{\text{noise}} t^{q},
$$
where $\eta(x) := P(Y=1 \mid X=x)$.
\end{itemize}

To begin,  we quote a useful Lemma from \citet[Theorem 2]{jiang2017} into our contexts.
We let $\tilde{f}_{\mathcal{H}_{k}}(z) = \frac{2\pi_{+}}{\sqrt{2\pi} n_{\rm p} h^d} \sum_{i=1} ^{n_{\rm p}} k(z, x_i ^{{\rm p}}) - \frac{1}{\sqrt{2\pi} n_{\rm u} h^d} $ $\sum_{i=1} ^{n_{\rm u}} k(z, x_i ^{{\rm u}})$ and $\tilde{f}_{{\rm Bayes}}(z) = 2\pi_{+} p_{X \mid Y=1}(z) -p_X (z)$.
\begin{lemma}[Theorem 2 of \citet{jiang2017}]
Suppose $p_X (x)$ and $p_{X \mid Y=1} (x)$ are $\alpha_{\text{H}}$-H{\"o}lder continuous. Then there exist a constant $\tilde{C}_{(\rm p, \rm u)}$ such that the following holds with probability at least $1- 1/n_{\rm p}  - 1/ n_{\rm u}$.
\begin{align}
\norm{ \tilde{f}_{\mathcal{H}_{k}} - \tilde{f}_{{\rm Bayes}} }_{\infty} \leq 4\pi_{+} \tilde{C}_{(\rm p, \rm u)} (n_{\rm p} \wedge n_{\rm u})^{-\frac{\alpha_{\text{H}}}{2\alpha_{\text{H}}+d}},
\label{eq:event}
\end{align}
where the bandwidth $h = (n_{\rm p} \wedge n_{\rm u} )^{-\frac{1}{2\alpha_{\text{H}}+d}}$.
\label{lem:uniform_kernel_density}
\end{lemma}

\begin{proof}[Proof of Theorem \ref{thm:excess_super_fast}]
Due to the $\ell_{01} (z)= \ell_{01}(cz)$ for any $c>0$ and $z\in \mathbb{R}$, we have
\begin{align*}
R_{\ell_{01}}(\hat{f}_{\mathcal{H}_{k, r_1}}) - \inf_{f \in \mathcal{U} }  R_{\ell_{01}}(f) &=  R_{\ell_{01}}(\tilde{f}_{\mathcal{H}_{k}}) -  R_{\ell_{01}}( \tilde{f}_{{\rm Bayes}} )\\ 
&= \mathbb{E}_{X} \left( | 2\eta(X)-1 |  \mathbbm{1}_{ \{ \tilde{f}_{\mathcal{H}_{k}} (X) \tilde{f}_{{\rm Bayes}}(X) < 0\}} \right),
\end{align*}
where $\mathbbm{1}_{\{ \cdot \}}$ denotes the indicator function. 
The second equality is due to
\begin{align*}
R_{\ell_{01}}(\tilde{f}_{\mathcal{H}_{k}}) -  R_{\ell_{01}}( \tilde{f}_{{\rm Bayes}} ) &= \int  \mathbbm{1}_{ \{Y \tilde{f}_{\mathcal{H}_{k}}(X) < 0\} } - \mathbbm{1}_{ \{Y \tilde{f}_{{\rm Bayes}}(X) < 0\} } dP_{X,Y}(x,y) \\
&= \int \eta(x)  (\mathbbm{1}_{ \{ \tilde{f}_{\mathcal{H}_{k}}(X) < 0 \} } - \mathbbm{1}_{ \{ \tilde{f}_{{\rm Bayes}}(X) < 0 \} } ) \\
&+ (1-\eta(x))  ( \mathbbm{1}_{ \{\tilde{f}_{\mathcal{H}_{k}}(X) > 0 \} } - \mathbbm{1}_{ \{ \tilde{f}_{{\rm Bayes}}(X) > 0 \} } )  dP_{X }(x) \\
&= \int (2\eta(x)-1) \mathbbm{1}_{ \{ \tilde{f}_{\mathcal{H}_{k}} (X) \tilde{f}_{{\rm Bayes}}(X) < 0\}} \mathbbm{1}_{ \{ \tilde{f}_{{\rm Bayes}}(X) < 0 \} }  dP_{X }(x) \\
&+ \int (1-2\eta(x)) \mathbbm{1}_{ \{ \tilde{f}_{\mathcal{H}_{k}} (X) \tilde{f}_{{\rm Bayes}}(X) < 0\}} \mathbbm{1}_{ \{ \tilde{f}_{{\rm Bayes}}(X) > 0 \} } dP_{X }(x) \\
&= \mathbb{E}_{X} \left( | 2\eta(X)-1 |  \mathbbm{1}_{ \{ \tilde{f}_{\mathcal{H}_{k}} (X) \tilde{f}_{{\rm Bayes}}(X) < 0\}} \right).
\end{align*}


Denote $\varepsilon(n_{\rm p}, n_{\rm u}) := 4\pi_{+} \tilde{C}_{(\rm p, \rm u)} (n_{\rm p} \wedge n_{\rm u})^{-\frac{\alpha_{\text{H}}}{2\alpha_{\text{H}}+d}}$ and let $\mathcal{E}$ be the event that Equation \eqref{eq:event} holds.
By Lemma \ref{lem:uniform_kernel_density}, under the event $\mathcal{E}$, $\mathbbm{1}_{ \{ \tilde{f}_{\mathcal{H}_{k}} (X) \tilde{f}_{{\rm Bayes}}(X) < 0\} } \leq \mathbbm{1}_{ \{ | {\tilde{f}_{{\rm Bayes}}(X) } | \leq  \varepsilon(n_{\rm p}, n_{\rm u}) \} }$.
Thus, we have
\begin{align*}
\mathbb{E}_{X} \left( | 2\eta(X)-1 |  \mathbbm{1}_{ \{ \tilde{f}_{\mathcal{H}_{k}} (X) \tilde{f}_{{\rm Bayes}}(X) < 0\} } \right) & \leq \mathbb{E}_{X} \left( | 2\eta(X)-1 |  \mathbbm{1}_{ \{ | {\tilde{f}_{{\rm Bayes}}(X) } | \leq  \varepsilon(n_{\rm p}, n_{\rm u}) \} } \right) \\
& \leq \mathbb{E}_{X} \left( | 2\eta(X)-1 |  \mathbbm{1}_{ \{  | 2\eta(X)-1 | \leq \frac{\varepsilon(n_{\rm p}, n_{\rm u})}{p_{\rm min}} \} } \right)\\
& \leq \frac{\varepsilon(n_{\rm p}, n_{\rm u})}{p_{\rm min}} P_X \left( | 2\eta(X)-1 | \leq \frac{\varepsilon(n_{\rm p}, n_{\rm u})}{p_{\rm min}} \right) \\
&\leq C_{(\rm p, \rm u)} (n_{\rm p} \wedge n_{\rm u})^{-\frac{\alpha_{\text{H}}  (1+q)}{2\alpha_{\text{H}}+d}}, 
\end{align*}
for some constant $C_{(\rm p, \rm u)}>0$. Thus, under the assumptions (A1)-(A4), we have the following with probability at least $1- 1/n_{\rm p}  - 1/ n_{\rm u}$:
\begin{align*}
R_{\ell_{01}}(\hat{f}_{\mathcal{H}_{k, r_1}}) - \inf_{f \in \mathcal{U} }  R_{\ell_{01}}(f) \leq C_{(\rm p, \rm u)} (n_{\rm p} \wedge n_{\rm u})^{-\frac{\alpha_{\text{H}} (1+q)}{2\alpha_{\text{H}}+d}}.
\end{align*}
\end{proof}


\section{Implementation details}
\label{app:imple_details}

In this section, we provide implementation details for WMMD and other baseline PU learning algorithms in Appendix \ref{app:wmmd_imple_details} and \ref{app:existing_imple_details}, respectively.

\subsection{Proposed PU learning algorithm: WMMD algorithm}
\label{app:wmmd_imple_details}
We divided the original training data into training and validation sets, with an 80-20 random split. 
Let $\tilde{x}_j ^{\rm p}$ and $\tilde{x}_j ^{\rm u}$ be the positive and the unlabeled samples in validation set, respectively. 
Similarly, $\tilde{n}_{\rm p}$ and $\tilde{n}_{\rm u}$ be the number of samples in the positive and the unlabeled validation set, respectively.
With the validation set, we conducted a grid search method for the hyperparameter selection with a grid $\gamma \in \{1, 0.4, 0.2, 0.1, 0.05 \}$ for all the numerical experiments.
With the grids, we selected the optimal hyperparameters $\gamma^*$ which minimized 
\begin{align}
\hat{L} (\gamma) := -\pi_{+} + \frac{2\pi_{+}}{ \tilde{n}_{\rm p}} \sum_{j=1} ^{ \tilde{n}_{\rm p}} \mathbbm{1} \left[ {\rm sign} ( \hat{f}_{\mathcal{H}_{k,r}} ^{\gamma} (\tilde{x}^{\rm p} _j) ) \neq 1 \right] + \frac{1}{ \tilde{n}_{\rm u}} \sum_{j=1} ^{ \tilde{n}_{\rm u}} \mathbbm{1} \left[ {\rm sign} ( \hat{f}_{\mathcal{H}_{k,r}} ^{\gamma} (\tilde{x}^{\rm u} _j) ) \neq -1 \right], \label{eq:hyperparameter_risk}
\end{align}
where
\begin{equation*}
    {\rm sign}\{\hat{f}_{\mathcal{H}_{k,r}} ^{\gamma} (z)\} = \begin{cases} 
      +1 & \text{if } (2\pi_{+})^{-1} <  \quad \hat{\lambda}_{n_{\rm p},n_{\rm u} } ^{\gamma}(z)     \\
      -1 & \text{otherwise}
   \end{cases},
\end{equation*}
$$
\hat{\lambda}_{n_{\rm p},n_{\rm u} } ^{\gamma}(z) = \frac{ n_{\rm p}^{-1} \sum_{i=1 } ^{n_{\rm p}} k_{\gamma} (z, x_i ^{\rm p})}{ n_{\rm u}^{-1} \sum_{i=1 } ^{n_{\rm u}} k_{\gamma} (z, x_i ^{\rm u})},
$$
and $k_{\gamma} (z_1, z_2) = \exp( -\gamma \norm{z_1 - z_2}_2 ^2 )$.
Note that Equation \eqref{eq:hyperparameter_risk} is an empirical estimation of the misclassification error since
\begin{align*}
&P_{X,Y} ( f(X)Y <0 ) \\
=& \pi_{+} P_{X \mid Y=1} ( {\rm sign}( f(X)) = -1) + (1-\pi_{+}) P_{X \mid Y=-1} ( {\rm sign} ( f(X)) = 1 ) \\
=& \pi_{+} P_{X \mid Y=1} ( {\rm sign} ( f(X) ) = -1 )  \\
&+  (P_{X} ( {\rm sign} ( f(X) ) = 1 ) - \pi_{+} P_{X \mid Y=1} ( {\rm sign} ( f(X) ) = 1 ) ) \\
=& -\pi_{+} + 2\pi_{+} P_{X \mid Y=1} ( {\rm sign} ( f(X) ) = -1 ) + P_{X} ( {\rm sign} ( f(X) ) = 1 ).    
\end{align*}
Final classification for a test datum $z$ was determined by ${\rm sign} ( \hat{f}_{\mathcal{H}_{k,r}} ^{\gamma^*} (z) )$ and AUC was computed by using $\hat{\lambda}_{n_{\rm p},n_{\rm u} } ^{\gamma^*}(z)$.

When the class-prior $\pi_{+}$ is unknown, we suggest a simple $\pi_{+}$ estimation method, called the density-based method, to find $\hat{\pi}_{+} ^{\rm WMMD}$ such that for some $\eta \in (0,1)$,
\begin{align*}
 \hat{\pi}_{+} ^{\rm WMMD} :=  \sup \left\{  t \in (0,1) \Big\vert \frac{ \#\{ j \in \{1, \dots, \tilde{n}_{\rm p} \} \mid \{ \hat{\lambda}_{n_{\rm p},n_{\rm u} } ^{\gamma}( \tilde{x}_j ^{\rm p} ) \}^{-1} \leq t  \} }{\tilde{n}_{\rm p}} \leq \eta \right\}.
\end{align*}
This estimator is sensible because $\pi_{+} \leq p_{X} (x)/p_{X \mid Y=1} (x)$ and $\{ \hat{\lambda}_{n_{\rm p},n_{\rm u} } ^{\gamma}( x ) \}^{-1}$ can be considered as a kernel density estimation of $p_{X} (x)/p_{X \mid Y=1}(x)$ for $x \in supp(P_{X \mid Y=1})$.
Here, we denote the density functions by $p_{X}$ and $p_{X \mid Y=1}$.
In our experiments, we fix $\eta=0.1$ and using $\hat{\pi}_{+} ^{\rm WMMD}$ leaded better performance than using \lq{}KM1\rq{} method in terms of accuracy.

\subsection{The baseline PU learning algorithms}
\label{app:existing_imple_details}

We compared the following 4 PU learning algorithms: (i) the logistic loss $\ell_{\mathrm{log}}$, denoted by LOG, (ii) the double hinge loss $\ell_{\mathrm{dh}}$, denoted by DH, both proposed by \citet{du2015}, (iii) the non-negative risk estimator method, denoted by NNPU, proposed by \citet{kiryo2017}, and (iv) the threshold adjustment method, denoted by tADJ, proposed by \citet{elkan2008}.

\textbf{General: } Similar to the procedure in Appendix \ref{app:wmmd_imple_details}, we set training and validation sets with the 80-20 random split of the original training dataset and we conducted a grid search method for hyperparameter selection.

\textbf{LOG and DH: } As \citet{du2015} proposed, we followed a binary discriminant function as
$$
g_{\alpha, b}(x) = \sum_{i=1} ^{N} \alpha_i \varphi_i (x) + b = \alpha ^T \varphi(x) + b,
$$
where $N = n_{\rm p}+n_{\rm u}$ and $\varphi_i (x) = \exp( -\gamma \norm{x -c_i}_2 ^2 )$ for $ \{c_1, \dots, c_N \} = \{x_1 ^{\rm p}, \dots, x_{n_{\rm p}} ^{\rm p},$ $ x_1 ^{\rm u}, \dots, x_{n_{\rm u}} ^{\rm u} \}$.
For a loss function $\ell \in \{\ell_{\mathrm{log}}, \ell_{\mathrm{dh}} \} $, the empirical risk function is given by
\begin{align*}
\hat{J}_{\lambda, \gamma}(\alpha, b) = -\frac{\pi_{+}}{n_{\rm p}} \sum_{i=1} ^{n_{\rm p}} \alpha^T \varphi(x_i ^{\rm p}) - \pi_{+}b + \frac{1}{n_{\rm u}} \sum_{i=1 } ^{n_{\rm u}} \ell\left( -\alpha^T \varphi(x_i ^{\rm u}) - b\right) + \frac{\lambda}{2} \alpha^T \alpha.
\end{align*}

Here, the hyperparameter grids are $\lambda \in \{1, 0.4, 0.2, 0.1, 0.05 \}$ and $\gamma \in \{1, 0.4, 0.2, $ $0.1, 0.05 \}$ for all the numerical experiments.
With the grids, we selected the optimal hyperparameter $(\lambda^*, \gamma^*)$ which minimized the empirical risk on the validation set $\tilde{J}_{\lambda, \gamma}(\alpha, b)$ defined by
$$
\tilde{J}_{\lambda, \gamma}(\alpha, b) = -\frac{\pi_{+}}{ \tilde{n}_{\rm p}} \sum_{i=1} ^{ \tilde{n}_{\rm p}} \alpha^T \varphi( \tilde{x}_i ^{\rm p}) - \pi_{+}b + \frac{1}{ \tilde{n}_{\rm u}} \sum_{i=1 } ^{ \tilde{n}_{\rm u}} \ell\left( -\alpha^T \varphi( \tilde{x}_i ^{\rm u}) - b\right) + \frac{\lambda}{2} \alpha^T \alpha.
$$

After selecting the optimal hyperparameter $(\lambda^*, \gamma^*)$, we minimized $\hat{J}_{(\lambda^*, \gamma^*)}(\alpha, b)$ with the gradient descent algorithm.
Learning rate was fixed by $0.1$ and the number of epochs was $100$.
During the training, we applied the early stopping rule: we stopped training if the validation error is not minimized in 10 successive epochs.
After the training phase, with the trained $\hat{\alpha}$ and $\hat{b}$, we classified a test datum $z$ as ${\rm sign} ( g_{\hat{\alpha}, \hat{b}}(z) )$.
AUC was computed by using $g_{\hat{\alpha}, \hat{b}}(z)$.

\textbf{NNPU: } We followed the method by \citet{kiryo2017}.
The model for NNPU was a 5-layer multilayer perceptron with ReLU nonlinearity ($d$-300-300-300-1). 
We applied the batch normalization before each ReLU nonlinearity.
Please note that this network architecture is quite similar to the model in \citet{kiryo2017}.
We used a stochastic gradient descent algorithm with a learning rate $0.01$. 
Loss function was the sigmoid function.
The number of epochs was $100$ and the optimal weights were selected at the best validation error during the training.

\textbf{tADJ: } We followed the method by \citet{elkan2008}.
We used \textsf{\lq{}LogisticRegressionCV\rq{}} function in the Python module \textsf{\lq{}sklearn.linear\_model\rq{}} \citep{scikit-learn} to estimate $P( \{ x \text{ is from the positive dataset} \} \mid X=x)$.
The hyperparameter grid for $L_2$-regularizer was $\{0.01,0.1,1,10,100\}$ and the optimal hyperparameter was chosen based on 5-fold cross validation on the split training dataset, {\it i.e.}, 80\% of the original training dataset.
Then, $P( \{ x \text{ is from the positive dataset} \} \mid Y=1)$ was estimated by the split validation set, {\it i.e.}, 20\% of the original training dataset.  

\section{Comparison between Gaussian and inverse kernels}
We compared LOG, DH, and WMMD using two kernels: (i) the Gaussian kernel $k(x,y)=\exp(-\gamma\norm{x-y}_2 ^2) $ and (ii) the inverse kernel $k(x,y)=\frac{\gamma}{\gamma + \norm{x-y}_2 ^2}$ for $\gamma > 0$.
Figures \ref{fig:kernels_accuracy_n_u} and \ref{fig:kernels_AUC_n_u} show the accuracy and AUC of LOG, DH, and WMMD on various $n_{\rm u}$. 
The training sample size for the positive data is $n_{\rm p} = 100$ and the class prior is $\pi_{+}= 0.5$.
The unlabeled sample size changes from 40 to 500 by 20. 
We repeat a random generation of training and test data 100 times. For comparison purposes, we add the $1-$Bayes risk for each unlabeled sample size.
In every algorithm, using the Gaussian kernel achieves higher accuracy and AUC than using the inverse kernel in every $n_{\rm u}$.

Figures \ref{fig:kernels_accuracy_pi_plus} and \ref{fig:kernels_AUC_pi_plus} show comparison of the accuracy and AUC as $\pi_{+}$ changes.
The training sample size for the positive and the unlabeled data are $n_{\rm p} = 100$ and $n_{\rm u} = 400$, respectively. 
The class-prior $\pi_{+}$ changes from 0.05 to 0.95 by 0.05.
The test sample size is $10^3$. 
We repeat a random generation of training and test data 100 times.
Both kernels perform comparably for LOG and WMMD algorithm.
The DH algorithm with the inverse kernel achieves higher accuracy when the class-prior is close to 0.5.
But in terms of AUC, both kernels perform comparably in every $\pi_{+}$.

\begin{figure}[t]
\subfigure[Accuracy comparison on various $n_{\text{u}}$.]{
\includegraphics[width=0.48\textwidth, height=1.65in]{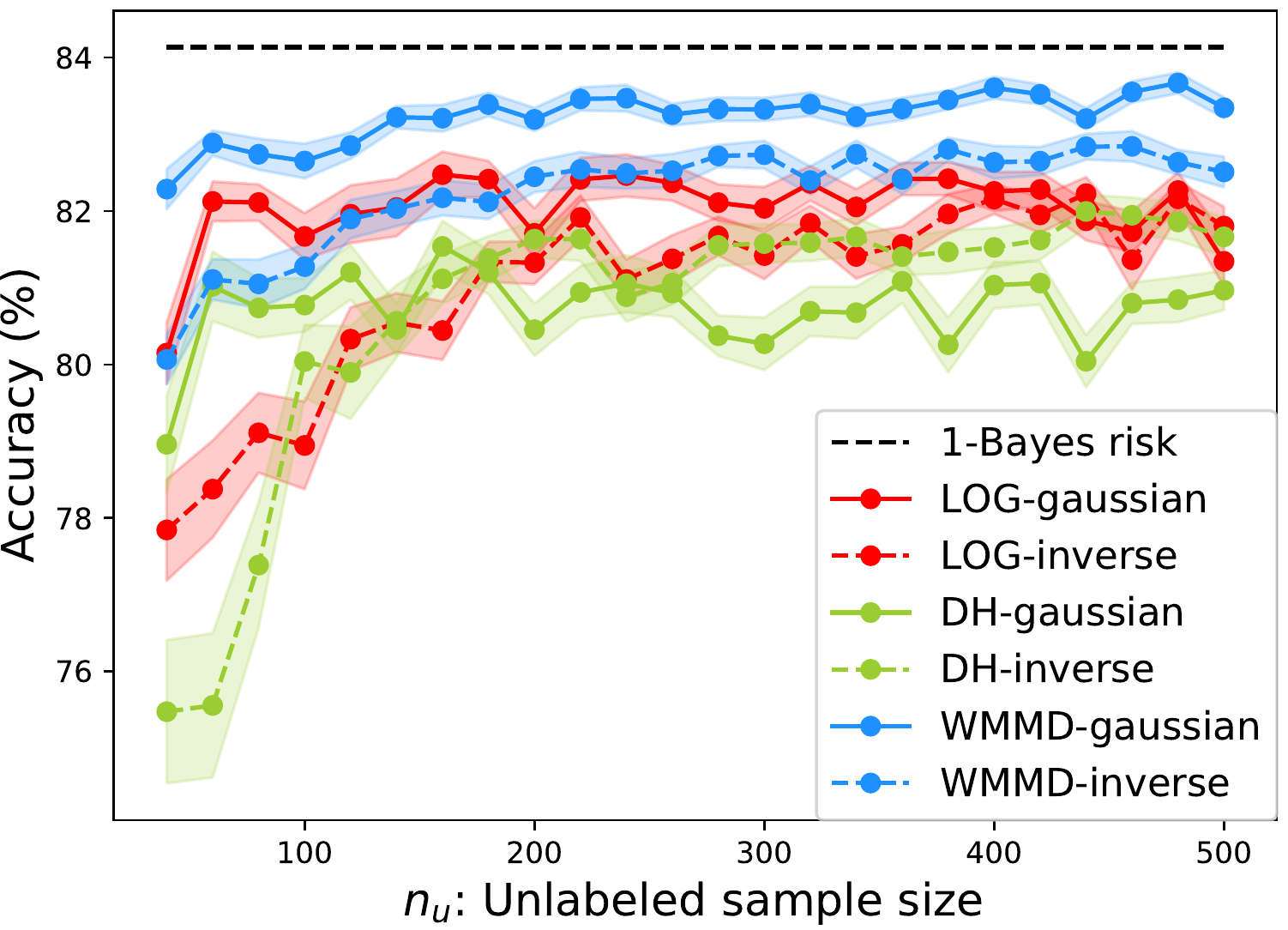}
\label{fig:kernels_accuracy_n_u}
}
~~
\subfigure[Accuracy comparison on various $\pi_{+}$.]{
\includegraphics[width=0.48\textwidth, height=1.65in]{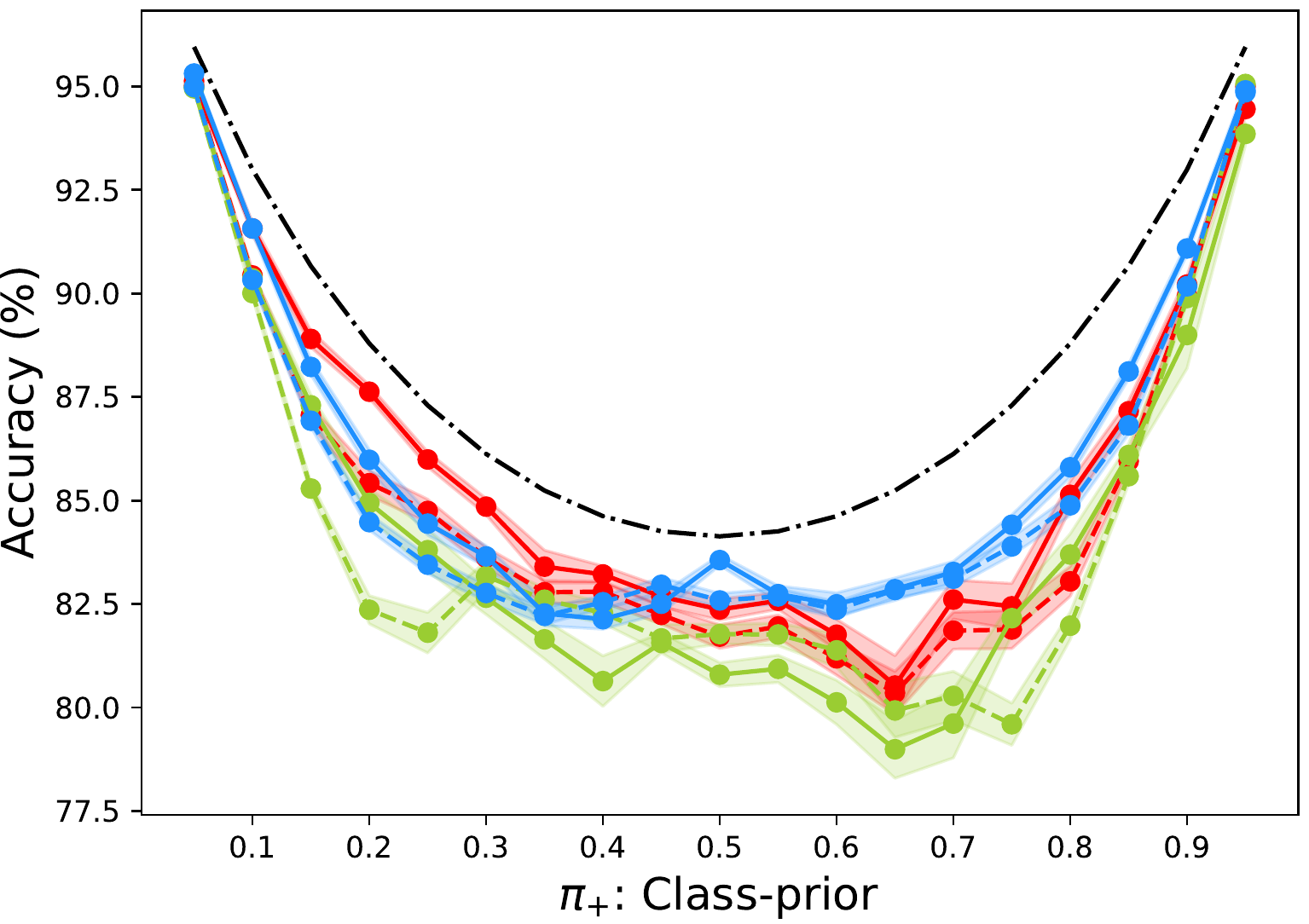}
\label{fig:kernels_accuracy_pi_plus}
}
~~
\subfigure[AUC comparison on various $n_{\text{u}}$.]{
\includegraphics[width=0.48\textwidth, height=1.65in]{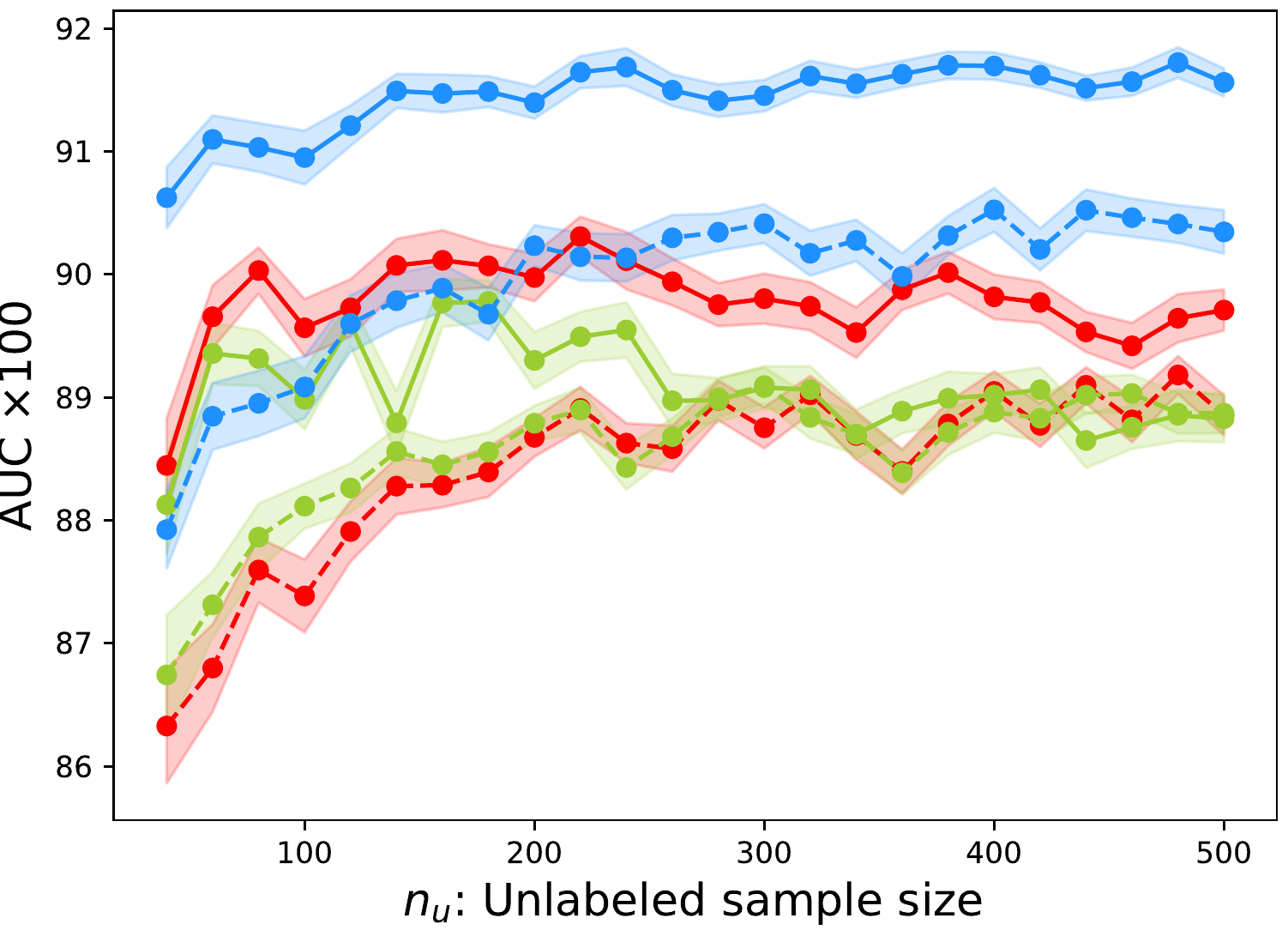}
\label{fig:kernels_AUC_n_u}
}
~~
\subfigure[AUC comparison on various $\pi_{+}$.]{
\includegraphics[width=0.48\textwidth, height=1.65in]{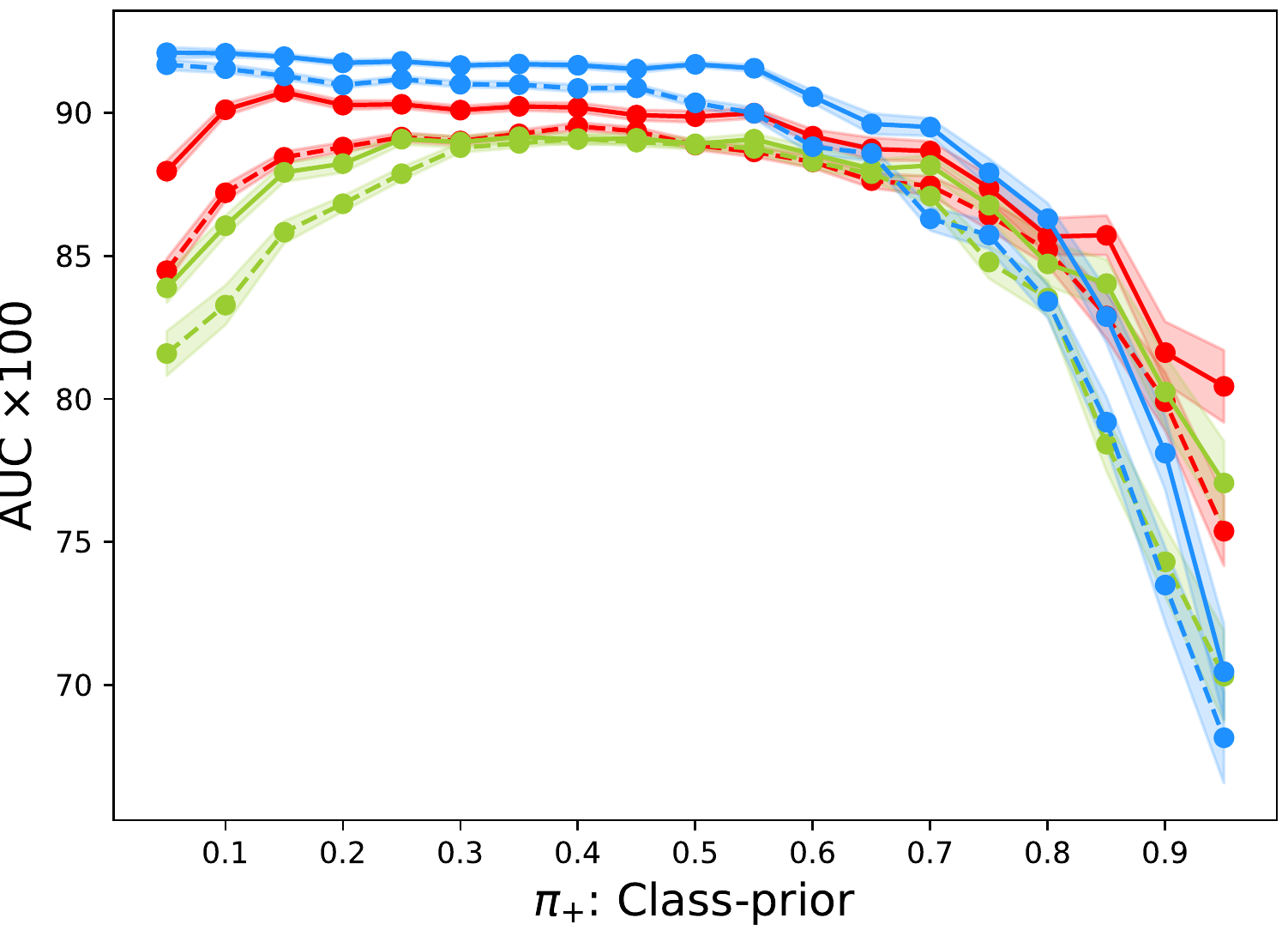}
\label{fig:kernels_AUC_pi_plus}
}
\caption{The comparison of the accuracy and AUC of the LOG, DH, and WMMD algorithms with the Gaussian and the inverse kernels when each of $n_{\rm {u}}$ and $\pi_{+}$ changes.
The black dashed-dotted curve represents the $1-$Bayes risk.
The algorithms with the Gaussian and inverse kernel are displayed with dashed and solid lines, respectively.
The curve and the shaded region represent the average and the standard error, respectively, based on 100 replications.}

\end{figure}

\bibliographystyle{unsrtnat}
\small
\bibliography{ref}

\begin{thebibliography}{49}
\providecommand{\natexlab}[1]{#1}
\providecommand{\url}[1]{\texttt{#1}}
\expandafter\ifx\csname urlstyle\endcsname\relax
  \providecommand{\doi}[1]{doi: #1}\else
  \providecommand{\doi}{doi: \begingroup \urlstyle{rm}\Url}\fi

\bibitem[Chapelle et~al.(2006)Chapelle, Sch{\"o}lkopf, and Zien]{chapelle2006}
O~Chapelle, B~Sch{\"o}lkopf, and A~Zien.
\newblock \emph{Semi-Supervised Learning}.
\newblock MIT Press, 2006.

\bibitem[Denis et~al.(2005)Denis, Gilleron, and Letouzey]{denis2005}
Fran{\c{c}}ois Denis, R{\'e}mi Gilleron, and Fabien Letouzey.
\newblock Learning from positive and unlabeled examples.
\newblock \emph{Theoretical Computer Science}, 348\penalty0 (1):\penalty0
  70--83, 2005.

\bibitem[Li and Liu(2005)]{li2005}
Xiao-Li Li and Bing Liu.
\newblock Learning from positive and unlabeled examples with different data
  distributions.
\newblock In \emph{European Conference on Machine Learning}, pages 218--229.
  Springer, 2005.

\bibitem[Elkan and Noto(2008)]{elkan2008}
Charles Elkan and Keith Noto.
\newblock Learning classifiers from only positive and unlabeled data.
\newblock In \emph{Proceedings of the 14th ACM SIGKDD International Conference
  on Knowledge Discovery and Data Mining}, pages 213--220. ACM, 2008.

\bibitem[Xiao et~al.(2011)Xiao, Liu, Yin, Cao, Zhang, and Hao]{xiao2011}
Yanshan Xiao, Bo~Liu, Jie Yin, Longbing Cao, Chengqi Zhang, and Zhifeng Hao.
\newblock Similarity-based approach for positive and unlabelled learning.
\newblock In \emph{Twenty-Second International Joint Conference on Artificial
  Intelligence}, 2011.

\bibitem[Zuluaga et~al.(2011)Zuluaga, Hush, Leyton, Hoyos, and
  Orkisz]{zuluaga2011}
Maria~A Zuluaga, Don Hush, Edgar JF~Delgado Leyton, Marcela~Hern{\'a}ndez
  Hoyos, and Maciej Orkisz.
\newblock Learning from only positive and unlabeled data to detect lesions in
  vascular ct images.
\newblock In \emph{International Conference on Medical Image Computing and
  Computer-Assisted Intervention}, pages 9--16. Springer, 2011.

\bibitem[Gong et~al.(2018)Gong, Wang, Ye, Xu, and Lin]{gong2018}
Tieliang Gong, Guangtao Wang, Jieping Ye, Zongben Xu, and Ming Lin.
\newblock Margin based pu learning.
\newblock In \emph{AAAI Conference on Artificial Intelligence}, 2018.

\bibitem[Yang et~al.(2012)Yang, Li, Mei, Kwoh, and Ng]{yang2012}
Peng Yang, Xiao-Li Li, Jian-Ping Mei, Chee-Keong Kwoh, and See-Kiong Ng.
\newblock Positive-unlabeled learning for disease gene identification.
\newblock \emph{Bioinformatics}, 28\penalty0 (20):\penalty0 2640--2647, 2012.

\bibitem[Yang et~al.(2014)Yang, Li, Chua, Kwoh, and Ng]{yang2014}
Peng Yang, Xiaoli Li, Hon-Nian Chua, Chee-Keong Kwoh, and See-Kiong Ng.
\newblock Ensemble positive unlabeled learning for disease gene identification.
\newblock \emph{PloS One}, 9\penalty0 (5):\penalty0 e97079, 2014.

\bibitem[Blanchard et~al.(2010)Blanchard, Lee, and Scott]{blanchard2010}
Gilles Blanchard, Gyemin Lee, and Clayton Scott.
\newblock Semi-supervised novelty detection.
\newblock \emph{Journal of Machine Learning Research}, 11\penalty0
  (Nov):\penalty0 2973--3009, 2010.

\bibitem[Zhang et~al.(2017)Zhang, Wang, Yuan, and Tan]{zhang2017}
Jiaqi Zhang, Zhenzhen Wang, Junsong Yuan, and Yap-Peng Tan.
\newblock Positive and unlabeled learning for anomaly detection with
  multi-features.
\newblock In \emph{Proceedings of the 2017 ACM on Multimedia Conference}, pages
  854--862. ACM, 2017.

\bibitem[Liu et~al.(2002)Liu, Lee, Yu, and Li]{liu2002}
Bing Liu, Wee~Sun Lee, Philip~S Yu, and Xiaoli Li.
\newblock Partially supervised classification of text documents.
\newblock In \emph{International Conference on Machine Learning}, volume~2,
  pages 387--394. Citeseer, 2002.

\bibitem[Li and Liu(2003)]{li2003}
Xiaoli Li and Bing Liu.
\newblock Learning to classify texts using positive and unlabeled data.
\newblock In \emph{Proceedings of the 18th International Joint Conference on
  Artificial Intelligence}, pages 587--592. Morgan Kaufmann Publishers Inc.,
  2003.

\bibitem[Liu et~al.(2003)Liu, Dai, Li, Lee, and Yu]{liu2003}
Bing Liu, Yang Dai, Xiaoli Li, Wee~Sun Lee, and Philip~S Yu.
\newblock Building text classifiers using positive and unlabeled examples.
\newblock In \emph{Data Mining, 2003. ICDM 2003. Third IEEE International
  Conference on}, pages 179--186. IEEE, 2003.

\bibitem[Scott and Blanchard(2009)]{scott2009}
Clayton Scott and Gilles Blanchard.
\newblock Novelty detection: Unlabeled data definitely help.
\newblock In \emph{Artificial Intelligence and Statistics}, pages 464--471,
  2009.

\bibitem[du~Plessis et~al.(2014)du~Plessis, Niu, and Sugiyama]{du2014}
Marthinus~C du~Plessis, Gang Niu, and Masashi Sugiyama.
\newblock Analysis of learning from positive and unlabeled data.
\newblock In \emph{Advances in Neural Information Processing Systems}, pages
  703--711, 2014.

\bibitem[du~Plessis et~al.(2015)du~Plessis, Niu, and Sugiyama]{du2015}
Marthinus~C du~Plessis, Gang Niu, and Masashi Sugiyama.
\newblock Convex formulation for learning from positive and unlabeled data.
\newblock In \emph{International Conference on Machine Learning}, pages
  1386--1394, 2015.

\bibitem[Kiryo et~al.(2017)Kiryo, Niu, du~Plessis, and Sugiyama]{kiryo2017}
Ryuichi Kiryo, Gang Niu, Marthinus~C du~Plessis, and Masashi Sugiyama.
\newblock Positive-unlabeled learning with non-negative risk estimator.
\newblock In \emph{Advances in Neural Information Processing Systems}, pages
  1675--1685, 2017.

\bibitem[Oh et~al.(2018)Oh, Gavves, and Welling]{oh2018}
ChangYong Oh, Efstratios Gavves, and Max Welling.
\newblock Bock: Bayesian optimization with cylindrical kernels.
\newblock \emph{arXiv preprint arXiv:1806.01619}, 2018.

\bibitem[Sriperumbudur et~al.(2012)Sriperumbudur, Fukumizu, Gretton,
  Sch{\"o}lkopf, and Lanckriet]{sriperumbudur2012}
Bharath~K Sriperumbudur, Kenji Fukumizu, Arthur Gretton, Bernhard
  Sch{\"o}lkopf, and Gert~RG Lanckriet.
\newblock On the empirical estimation of integral probability metrics.
\newblock \emph{Electronic Journal of Statistics}, 6:\penalty0 1550--1599,
  2012.

\bibitem[Ward et~al.(2009)Ward, Hastie, Barry, Elith, and Leathwick]{ward2009}
Gill Ward, Trevor Hastie, Simon Barry, Jane Elith, and John~R Leathwick.
\newblock Presence-only data and the em algorithm.
\newblock \emph{Biometrics}, 65\penalty0 (2):\penalty0 554--563, 2009.

\bibitem[Niu et~al.(2016)Niu, du~Plessis, Sakai, Ma, and Sugiyama]{niu2016}
Gang Niu, Marthinus~Christoffel du~Plessis, Tomoya Sakai, Yao Ma, and Masashi
  Sugiyama.
\newblock Theoretical comparisons of positive-unlabeled learning against
  positive-negative learning.
\newblock In \emph{Advances in Neural Information Processing Systems}, pages
  1199--1207, 2016.

\bibitem[Kato et~al.(2019)Kato, Teshima, and Honda]{kato2019}
Masahiro Kato, Takeshi Teshima, and Junya Honda.
\newblock Learning from positive and unlabeled data with a selection bias.
\newblock In \emph{International Conference on Learning Representations}, 2019.
\newblock URL \url{https://openreview.net/forum?id=rJzLciCqKm}.

\bibitem[Steinwart and Christmann(2008)]{steinwart2008}
Ingo Steinwart and Andreas Christmann.
\newblock \emph{Support vector machines}.
\newblock Springer Science \& Business Media, 2008.

\bibitem[Sakai et~al.(2017)Sakai, Plessis, Niu, and Sugiyama]{sakai2017}
Tomoya Sakai, Marthinus~Christoffel Plessis, Gang Niu, and Masashi Sugiyama.
\newblock Semi-supervised classification based on classification from positive
  and unlabeled data.
\newblock In \emph{International Conference on Machine Learning}, pages
  2998--3006, 2017.

\bibitem[Collobert et~al.(2006)Collobert, Sinz, Weston, and
  Bottou]{collobert2006}
Ronan Collobert, Fabian Sinz, Jason Weston, and L{\'e}on Bottou.
\newblock Trading convexity for scalability.
\newblock In \emph{Proceedings of the 23rd International Conference on Machine
  Learning}, pages 201--208. ACM, 2006.

\bibitem[Sansone et~al.(2018)Sansone, De~Natale, and Zhou]{sansone2018}
Emanuele Sansone, Francesco~GB De~Natale, and Zhi-Hua Zhou.
\newblock Efficient training for positive unlabeled learning.
\newblock \emph{IEEE Transactions on Pattern Analysis and Machine
  Intelligence}, 2018.

\bibitem[M{\"u}ller(1997)]{muller1997}
Alfred M{\"u}ller.
\newblock Integral probability metrics and their generating classes of
  functions.
\newblock \emph{Advances in Applied Probability}, 29\penalty0 (2):\penalty0
  429--443, 1997.

\bibitem[Sriperumbudur et~al.(2010{\natexlab{a}})Sriperumbudur, Fukumizu, and
  Lanckriet]{sriperumbudur2010b}
Bharath~K Sriperumbudur, Kenji Fukumizu, and Gert Lanckriet.
\newblock On the relation between universality, characteristic kernels and rkhs
  embedding of measures.
\newblock In \emph{Proceedings of the Thirteenth International Conference on
  Artificial Intelligence and Statistics}, pages 773--780, 2010{\natexlab{a}}.

\bibitem[Arjovsky et~al.(2017)Arjovsky, Chintala, and Bottou]{arjovsky2017}
Martin Arjovsky, Soumith Chintala, and L{\'e}on Bottou.
\newblock Wasserstein generative adversarial networks.
\newblock In \emph{International Conference on Machine Learning}, pages
  214--223, 2017.

\bibitem[Tolstikhin et~al.(2018)Tolstikhin, Bousquet, Gelly, and
  Schoelkopf]{tolstikhin2017}
Ilya Tolstikhin, Olivier Bousquet, Sylvain Gelly, and Bernhard Schoelkopf.
\newblock Wasserstein auto-encoders.
\newblock In \emph{International Conference on Learning Representations}, 2018.

\bibitem[Gretton et~al.(2012)Gretton, Borgwardt, Rasch, Sch{\"o}lkopf, and
  Smola]{gretton2012}
Arthur Gretton, Karsten~M Borgwardt, Malte~J Rasch, Bernhard Sch{\"o}lkopf, and
  Alexander Smola.
\newblock A kernel two-sample test.
\newblock \emph{Journal of Machine Learning Research}, 13\penalty0
  (Mar):\penalty0 723--773, 2012.

\bibitem[Huang et~al.(2007)Huang, Gretton, Borgwardt, Sch{\"o}lkopf, and
  Smola]{huang2007}
Jiayuan Huang, Arthur Gretton, Karsten~M Borgwardt, Bernhard Sch{\"o}lkopf, and
  Alex~J Smola.
\newblock Correcting sample selection bias by unlabeled data.
\newblock In \emph{Advances in neural information processing systems}, pages
  601--608, 2007.

\bibitem[Gretton et~al.(2009)Gretton, Smola, Huang, Schmittfull, Borgwardt,
  Sch{\"o}lkopf, Candela, Sugiyama, Schwaighofer, Lawrence,
  et~al.]{gretton2009}
A~Gretton, AJ~Smola, J~Huang, M~Schmittfull, KM~Borgwardt, B~Sch{\"o}lkopf,
  Qui{\~n}onero Candela, M~Sugiyama, A~Schwaighofer, ND~Lawrence, et~al.
\newblock Covariate shift by kernel mean matching.
\newblock In \emph{Dataset Shift in Machine Learning}, pages 131--160. MIT
  Press, 2009.

\bibitem[Yan et~al.(2017)Yan, Ding, Li, Wang, Xu, and Zuo]{yan2017}
Hongliang Yan, Yukang Ding, Peihua Li, Qilong Wang, Yong Xu, and Wangmeng Zuo.
\newblock Mind the class weight bias: Weighted maximum mean discrepancy for
  unsupervised domain adaptation.
\newblock In \emph{Computer Vision and Pattern Recognition (CVPR), 2017 IEEE
  Conference on}, pages 945--954. IEEE, 2017.

\bibitem[Bartlett and Mendelson(2002)]{bartlett2002}
Peter~L Bartlett and Shahar Mendelson.
\newblock Rademacher and gaussian complexities: Risk bounds and structural
  results.
\newblock \emph{Journal of Machine Learning Research}, 3\penalty0
  (Nov):\penalty0 463--482, 2002.

\bibitem[Bartlett et~al.(2006)Bartlett, Jordan, and McAuliffe]{bartlett2006}
Peter~L Bartlett, Michael~I Jordan, and Jon~D McAuliffe.
\newblock Convexity, classification, and risk bounds.
\newblock \emph{Journal of the American Statistical Association}, 101\penalty0
  (473):\penalty0 138--156, 2006.

\bibitem[Zhang(2004)]{zhang2004}
Tong Zhang.
\newblock Statistical behavior and consistency of classification methods based
  on convex risk minimization.
\newblock \emph{Annals of Statistics}, pages 56--85, 2004.

\bibitem[Sriperumbudur et~al.(2010{\natexlab{b}})Sriperumbudur, Gretton,
  Fukumizu, Sch{\"o}lkopf, and Lanckriet]{sriperumbudur2010a}
Bharath~K Sriperumbudur, Arthur Gretton, Kenji Fukumizu, Bernhard
  Sch{\"o}lkopf, and Gert~RG Lanckriet.
\newblock Hilbert space embeddings and metrics on probability measures.
\newblock \emph{Journal of Machine Learning Research}, 11\penalty0
  (Apr):\penalty0 1517--1561, 2010{\natexlab{b}}.

\bibitem[Lin(2002)]{lin2002}
Yi~Lin.
\newblock Support vector machines and the bayes rule in classification.
\newblock \emph{Data Mining and Knowledge Discovery}, 6\penalty0 (3):\penalty0
  259--275, 2002.

\bibitem[Audibert et~al.(2007)Audibert, Tsybakov, et~al.]{audibert2007}
Jean-Yves Audibert, Alexandre~B Tsybakov, et~al.
\newblock Fast learning rates for plug-in classifiers.
\newblock \emph{The Annals of statistics}, 35\penalty0 (2):\penalty0 608--633,
  2007.

\bibitem[Natarajan et~al.(2013)Natarajan, Dhillon, Ravikumar, and
  Tewari]{natarajan2013}
Nagarajan Natarajan, Inderjit~S Dhillon, Pradeep~K Ravikumar, and Ambuj Tewari.
\newblock Learning with noisy labels.
\newblock In \emph{Advances in neural information processing systems}, pages
  1196--1204, 2013.

\bibitem[Patrini et~al.(2016)Patrini, Nielsen, Nock, and Carioni]{patrini2016}
Giorgio Patrini, Frank Nielsen, Richard Nock, and Marcello Carioni.
\newblock Loss factorization, weakly supervised learning and label noise
  robustness.
\newblock In \emph{International conference on machine learning}, pages
  708--717, 2016.

\bibitem[Blanchard et~al.(2016)Blanchard, Flaska, Handy, Pozzi, and
  Scott]{blanchard2016}
Gilles Blanchard, Marek Flaska, Gregory Handy, Sara Pozzi, and Clayton Scott.
\newblock Classification with asymmetric label noise: Consistency and maximal
  denoising.
\newblock \emph{Electronic Journal of Statistics}, 10\penalty0 (2), 2016.

\bibitem[Pedregosa et~al.(2011)Pedregosa, Varoquaux, Gramfort, Michel, Thirion,
  Grisel, Blondel, Prettenhofer, Weiss, Dubourg, Vanderplas, Passos,
  Cournapeau, Brucher, Perrot, and Duchesnay]{scikit-learn}
F.~Pedregosa, G.~Varoquaux, A.~Gramfort, V.~Michel, B.~Thirion, O.~Grisel,
  M.~Blondel, P.~Prettenhofer, R.~Weiss, V.~Dubourg, J.~Vanderplas, A.~Passos,
  D.~Cournapeau, M.~Brucher, M.~Perrot, and E.~Duchesnay.
\newblock Scikit-learn: Machine learning in {P}ython.
\newblock \emph{Journal of Machine Learning Research}, 12:\penalty0 2825--2830,
  2011.

\bibitem[Ramaswamy et~al.(2016)Ramaswamy, Scott, and Tewari]{ramaswamy2016}
Harish Ramaswamy, Clayton Scott, and Ambuj Tewari.
\newblock Mixture proportion estimation via kernel embeddings of distributions.
\newblock In \emph{International Conference on Machine Learning}, pages
  2052--2060, 2016.

\bibitem[Chang and Lin(2011)]{chang2011}
Chih-Chung Chang and Chih-Jen Lin.
\newblock Libsvm: a library for support vector machines.
\newblock \emph{ACM transactions on intelligent systems and technology (TIST)},
  2\penalty0 (3):\penalty0 27, 2011.

\bibitem[Bekker and Davis(2018)]{bekker2018}
Jessa Bekker and Jesse Davis.
\newblock Estimating the class prior in positive and unlabeled data through
  decision tree induction.
\newblock In \emph{Proceedings of the 32th AAAI Conference on Artificial
  Intelligence}, 2018.

\bibitem[Jiang(2017)]{jiang2017}
Heinrich Jiang.
\newblock Uniform convergence rates for kernel density estimation.
\newblock In \emph{Proceedings of the 34th International Conference on Machine
  Learning-Volume 70}, pages 1694--1703, 2017.

\end{thebibliography}

\end{document}